\documentclass[accepted]{uai2024} % after acceptance, for a revised version; 
\usepackage[american]{babel}
\usepackage{natbib} % has a nice set of citation styles and commands
\usepackage{mathtools} % amsmath with fixes and additions
\usepackage{booktabs} % commands to create good-looking tables
\usepackage{tikz} % nice language for creating drawings and diagrams

\usepackage{amsfonts, amssymb, amsmath, amsthm, bbm, graphicx, color, tikz, booktabs, multirow, bm, cases, mathtools,  mathrsfs, subcaption, pgfplots, csvsimple,algorithm,algorithmic,wrapfig}
\usepackage{enumitem}
\usetikzlibrary{intersections, pgfplots.fillbetween, shapes, arrows, positioning, decorations.markings}
\definecolor {processblue}{cmyk}{0.96,0,0,0}
\tikzstyle{int}=[draw, fill=blue!20, minimum size=2em]
\tikzstyle{init} = [pin edge={to-,thin,black}]
\usepackage{pgfplotstable}
\usepackage{pgfplots}
\pgfplotsset{compat=1.14}
\hypersetup{colorlinks=true,linkcolor=blue,filecolor=gray, urlcolor=blue, citecolor=blue}
\usepackage{natbib}
\newtheorem{defn}{Definition}[section]
\newtheorem{lem}{Lemma}[section]
\newtheorem{rem}{Remark}[section]
\newtheorem{assum}{Assumption}

\newtheorem{thm}{Theorem}[section]

\newtheorem{cor}{Corollary}[section]

\newcommand{\ex}[2]{{\ifx&#1& \mathbb{E} \else {\mathbb{E}_{#1}} \fi \left[#2\right]}}

\newcommand{\cex}[3]{{\ifx&#1& \mathbb{E} \else {\mathbb{E}_{#1}^{#2}} \fi \left[#3\right]}}

\newcommand{\pr}[1]{\left(#1\right)}
\newcommand{\br}[1]{\left[#1\right]}
\newcommand{\abs}[1]{\left|#1\right|}
\newcommand{\eucd}[1]{\left|\left|#1\right|\right|}
\newcommand{\tr}[1]{tr\left\{#1\right\}}

\usepackage[toc,page,header]{appendix}
\title{Two Facets of SDE Under an Information-Theoretic Lens: Generalization of SGD via Training Trajectories and via Terminal States}
\author[1]{\href{mailto:<zwang286@uottawa.ca>?Subject=Your UAI 2024 paper}{Ziqiao~Wang}}
\author[1]{\href{mailto:<ymao@uottawa.ca>?Subject=Your UAI 2024 paper}{Yongyi~Mao}}
\affil[1]{%
School of Electrical Engineering and Computer Science\\

University of Ottawa\\

    Ottawa, Ontario, Canada
}
\begin{document}
\maketitle

\begin{abstract}
  Stochastic differential equations (SDEs) have been shown recently to  characterize well the dynamics of training machine learning models with SGD. When the generalization error of the SDE approximation closely aligns with that of SGD in expectation, it 
  % this
  provides two opportunities for understanding better the generalization behaviour of SGD through its SDE approximation. 
Firstly, viewing SGD as full-batch gradient descent with Gaussian gradient noise allows us to obtain trajectory-based generalization bound using the information-theoretic bound from \citet{xu2017information}. Secondly, assuming mild conditions, we estimate the steady-state weight distribution of SDE and use information-theoretic bounds from \citet{xu2017information} and \citet{negrea2019information} to establish terminal-state-based generalization bounds. 
Our proposed bounds have some advantages, notably the trajectory-based bound outperforms results in \cite{wang2022generalization}, and the terminal-state-based bound exhibits a fast decay rate comparable to stability-based bounds. 
\end{abstract}

\section{Introduction}
Modern deep neural networks trained with SGD and its variants have achieved surprising successes: the overparametrized networks often contain more parameters than the size of training dataset,  and yet are capable of generalizing well on the testing set; this contrasts the traditional wisdom in statistical learning that suggests such high-capacity models will overfit the training data and fail on the unseen data \citep{zhang2017understanding}.  Intense recent efforts have been spent to explain this peculiar phenomenon via investigating the properties of SGD \citep{arpit2017closer,bartlett2017spectrally,neyshabur2017exploring,arora2019fine}, and the current understanding is still far from being complete. For example, neural tangent kernel (NTK)-based generalization bounds of SGD normally require the width of network to be sufficiently large (or even go to infinite) \citep{arora2019fine}, and the stability-based bounds of SGD have a poorly dependence on an intractable Lipschitz constant \citep{hardt2016train,bassily2020stability}.
% \textcolor{red}{[NEED A highlight/summary of what is known and why it is not good enough. Or point the reader to the related works and discuss there.]} \looseness=-1
% In particular, the implicit regularization effect of SGD is often invoked in the literature. It then becomes well-known that such implicit regularization of SGD is crucial for both optimization and generalization \citep{barrett2020implicit,smith2020origin}. 
% \vspace{-3mm}
\begin{figure*}[!ht]
    \centering
    \begin{subfigure}[b]{0.245\textwidth}
\includegraphics[scale=0.28]{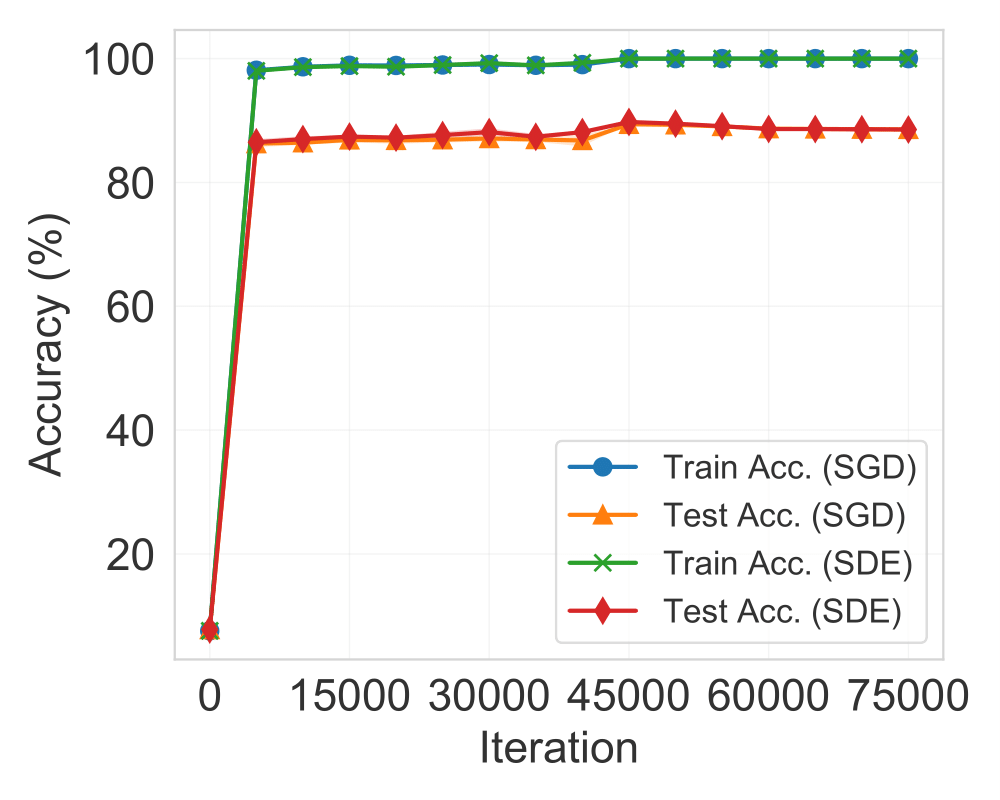}    
\caption{VGG on (small) SVHN}            \label{fig:vgg-svhn-acc}
    \end{subfigure}
\begin{subfigure}[b]{0.245\textwidth}
\includegraphics[scale=0.28]{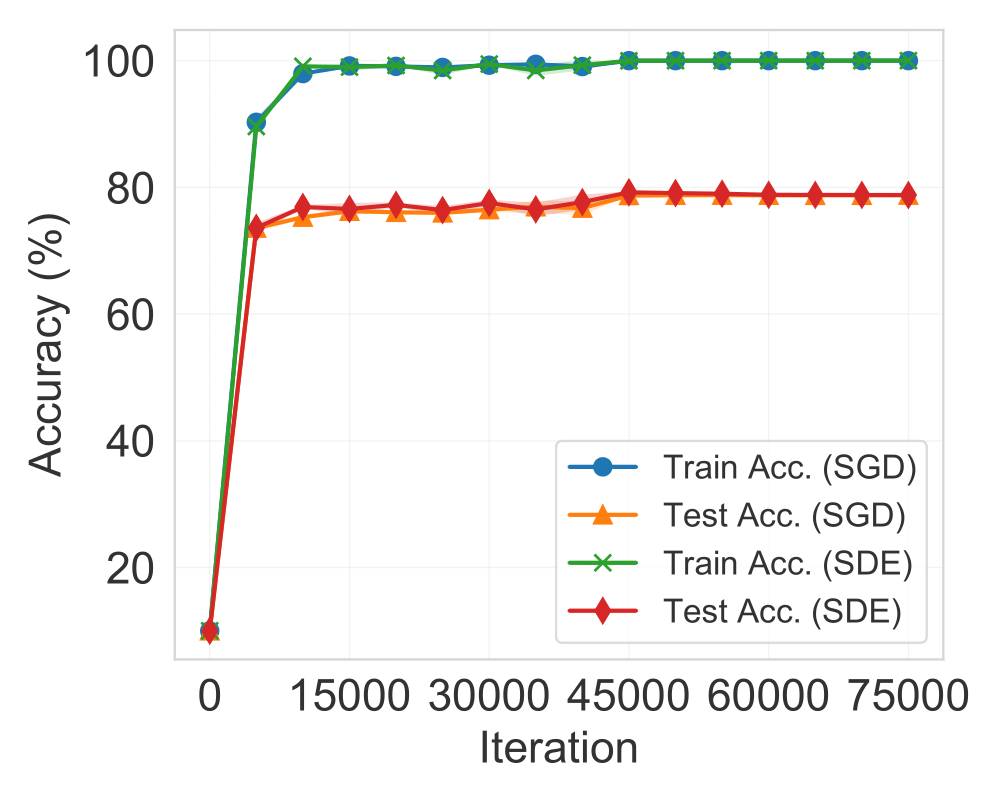}
\caption{VGG on CIFAR10}
    \label{fig:vgg-cifa10-acc}
\end{subfigure}
 \begin{subfigure}[b]{0.245\textwidth}
\includegraphics[scale=0.28]{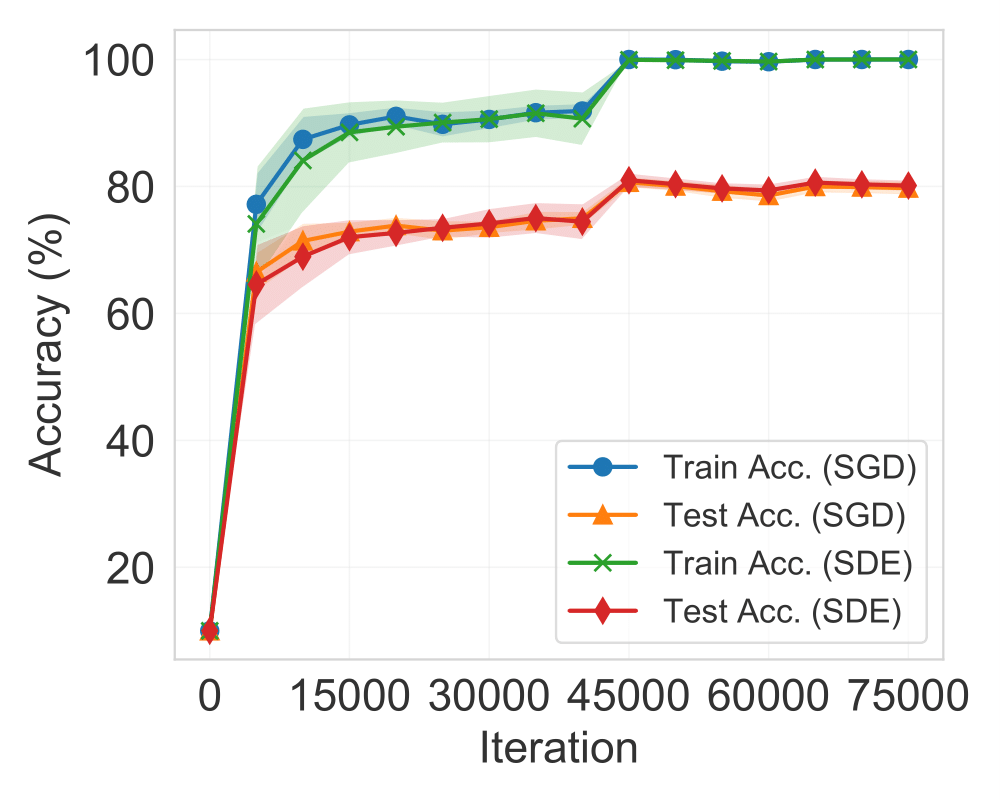}
\caption{ResNet on CIFAR10}
\label{fig:resnet-cifa10-acc}
    \end{subfigure}
\begin{subfigure}[b]{0.245\textwidth}
\includegraphics[scale=0.28]{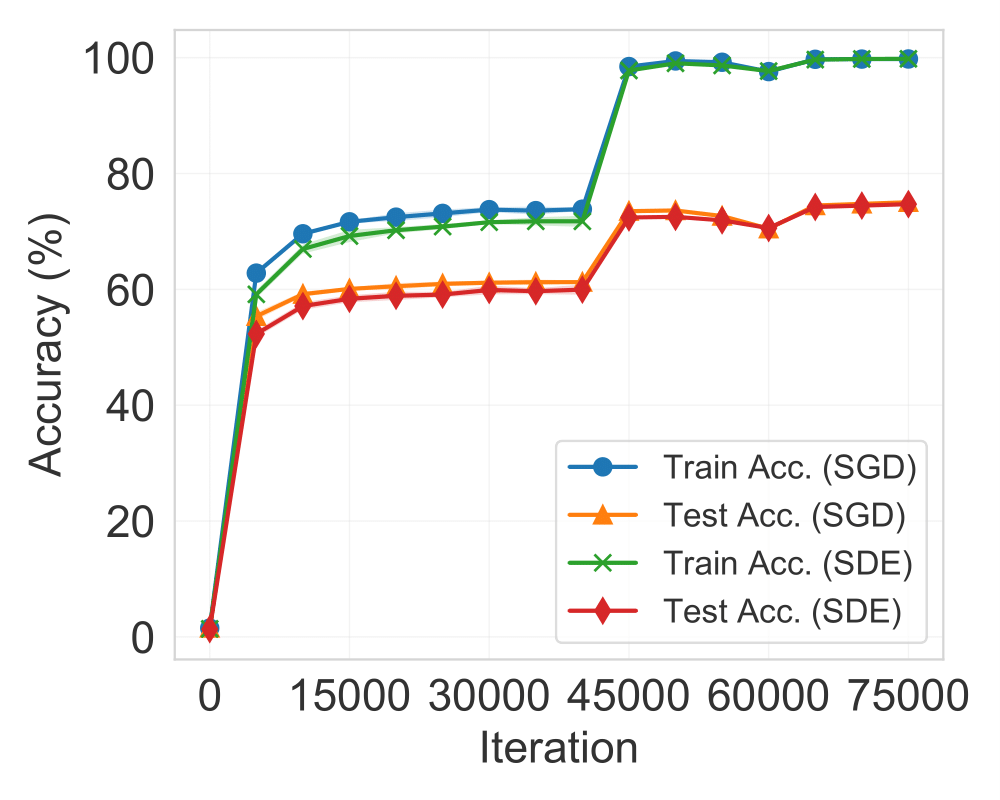}
\caption{ResNet on CIFAR100}
\label{fig:resnet-cifa100-acc}
\end{subfigure}
\caption{Performance of VGG-11 and ResNet-18 trained with SGD and SDE.
% Standard data augmentation techniques are only used in (d).
}\label{fig:Acc-Dynamics}
% \vspace{-4mm}
\end{figure*}
% \vspace{-0.1in}

% \textcolor{red}{Gap between SGD and SGLD, SDE is necessary}
Recently, information-theoretic generalization bounds have been developed to analyze the expected generalization error of a learning algorithm. The main advantage of such bounds is that they are not only distribution-dependent, but also algorithm-dependent, making them an ideal tool for studying the generalization behaviour of models trained with a specific algorithm, such as SGD.
%which seems necessary to obtain the non-vacuous bound in practice. 
The concept of mutual information (MI)-based bounds can be traced back to \cite[Section~1.3.1]{catoni2007pac}, where it is discussed how the prior distribution that minimizes the expected KL complexity term in a PAC-Bayesian (PAC-Bayes) bound can transform the KL term into the MI term. This idea has recently been popularized by \citep{russo2016controlling,russo2019much,xu2017information}, and further strengthened by additional techniques \citep{asadi2018chaining,negrea2019information,bu2019tightening,steinke2020reasoning,haghifam2020sharpened,wang2021analyzing}.
% Mutual information-based bounds are first proposed
% they are recently popularized by \citep{russo2016controlling,russo2019much,xu2017information}. They are then strengthened by additional techniques \citep{asadi2018chaining,negrea2019information,bu2019tightening,steinke2020reasoning,haghifam2020sharpened,wang2021analyzing}. 
Particularly, \citet{negrea2019information} derive MI-based bounds by developing a PAC-Bayes-like bounding technique, which upper-bounds the generalization error in terms of the KL divergence between the posterior distribution of learned model parameter given by a learning algorithm with respect to any data-dependent prior distribution. It is remarkable that the application of these information-theoretic techniques usually requires the learning algorithm to be an iterative noisy algorithm, such as stochastic gradient Langevin dynamics (SGLD) \citep{raginsky2017non,pensia2018generalization}, so as to avoid the MI bounds becoming infinity, and can not be directly applied to SGD. While it is feasible to adopt similar techniques from the PAC-Bayes bounds in \cite{lotfi2022pacbayes} to analyze SGD, in order to study the generalization of SGD using methods akin to those for SGLD, \citet{neu2021information} and \citet{wang2022generalization} develop generalization bounds for SGD via constructing an auxiliary iterative noisy process. However,  identifying an optimal auxiliary process is difficult, and arbitrary choices may not provide meaningful insights into the generalization of SGD, see Appendix~\ref{sec:IT-SGD} for more discussions.

Recent research has suggested that the SGD dynamics can be well approximated by using stochastic differential equations (SDEs), where the gradient signal in SGD is regarded as the full-batch gradient perturbed with an additive Gaussian noise. 
%There have been some recent attempts formulating SGD dynamics using stochastic differential equations (SDE), in which the key component is the modelling of the gradient noise. 
Specifically, \cite{mandt2017stochastic} and \cite{jastrzkebski2017three} model this gradient noise drawn from a Gaussian distribution with a fixed covariance matrix,
% a locally fixed Gaussian, 
thereby viewing SGD as performing variational inference. \cite{zhu2019anisotropic,wu2020noisy,xie2020diffusion}, and \cite{xie2021positive} further model the gradient noise as dependent of the current weight parameter and the training data. 
% Modelling SGD in this way provide explanations as to when SGD finds flat minima \citep{zhu2019anisotropic,xie2020diffusion} and sharp minima \citep{ziyin2022sgd}, and inspire some new training techniques \citep{wu2020noisy,xie2021positive}. 
Moreover, \citet{li2017stochastic,li2019stochastic} and \citet{wu2020noisy}
prove that when the learning rate is sufficiently small, the SDE trajectories are theoretically close to those of SGD (cf. Lemma~\ref{lem:sde-weak}). More recently, \cite{li2021validity} has demonstrated that the SDE approximation 
well characterizes the optimization and generalization behavior of SGD without requiring small learning rates.

In this work, we also empirically verify the consistency between the dynamics of SGD and its associated discrete SDE (cf. Eq.~(\ref{eq:sgd-update-gaussian})). As illustrated in Figure~\ref{fig:Acc-Dynamics}, the strong agreement in their performance suggests that, despite the potential presence of non-Gaussian components in the SGD gradient noise, analyzing its SDE through a Gaussian approximation suffices for exploring SGD's generalization behavior.
% \looseness=-1
% The results establishing SDE as a good approximation for the SGD dynamics motivate us to study the generalization behavior of SGD under such approximations. 
Furthermore, under the SDE formalism of SGD, SGD becomes an iterative noisy algorithm, on which the aforementioned information-theoretic bounding techniques can directly apply. In particular, we summarize our contributions below.
\begin{itemize}[leftmargin=*]
    \item We obtain a generalization bound (cf. Theorem \ref{thm:isotropic-prior-bound}) in the form of a summation over training steps of a quantity that involves both the sensitivity of the full-batch gradient to the variation of the training set and the covariance of the gradient noise (which makes the SGD gradient deviate from the full-batch gradient). We also give a tighter bound in Theorem \ref{thm:anisotropic-prior-bound},
    where 
    % the population gradient covariance and also the covariance of the gradient noise are involved, and 
    the generalization performance of SGD depends on the alignment of the population gradient covariance and the batch  gradient covariance.
    % these two matrices.
    These bounds highlight the significance of (the trace of) the gradient noise covariance in the generalization ability of SGD.  
    % Additionally, these results also provide justifications for the effectiveness of regularization by controlling the gradient norm, a technique suggested by previous works. 

    \item In addition to the time-dependent trajectory-based bounds, we also provide time-independent (or asymptotic) bounds by some mild assumptions. Specifically, based on  previous information-theoretic bounds, we obtain generalization bounds in terms of the KL divergence between the steady-state weight distribution of SGD with respect to a distribution-dependent prior distribution (by Lemma~\ref{lem:xu's-bound}) or data-dependent prior distribution (by Lemma~\ref{lem:data-dependent-prior}). The former gives us a bound based on the alignment between the weight covariance matrix for each individual local minimum and the weight covariance matrix for the average of local minima (cf. Theorem~\ref{thm:opt-state-inde-bound}). Under mild assumptions, we can estimate the steady-state weight distribution of SDE (cf. Lemma~\ref{lem:stationary-real}), leading to a variant of Theorem~\ref{thm:opt-state-inde-bound} (cf. Corollary~\ref{cor:pacbayes-anisotropic-prior}) and a norm-based bound (cf. Corollary~\ref{cor:pacbayes-isotropic-prior}). Additionally, we obtain a stability-based like bound by Lemma~\ref{lem:data-dependent-prior} (cf. Theorem \ref{thm:pacbayes-data-dependent-prior}), with the notable omission of the Lipschitz constant in other stability-based bounds. Since stability-based bounds often achieve fast decay rates, e.g., $\mathcal{O}(1/n)$, Theorem \ref{thm:pacbayes-data-dependent-prior} provides theoretical advantages compared with other information-theoretic bounds, as it can attain the same rate of decay as the stability-based bound.
    Comparing to the first family of bounds (i.e., trajectory-based bounds), the second family of bounds directly upper-bound the generalization error via the terminal
state, which avoids summing over training steps; these bounds can be tighter when the steady-state estimates are accurate. On the other hand, not relying on the steady-state estimates and the approximating
assumptions they base upon is arguably an advantage of the first family.

    \item  
    % In addition to comparing the generalization performance of SGD and SDE, we 
    We empirically analyze key components within the derived bounds for both algorithms. Our empirical findings reveal that these components for SGD and SDE align remarkably well, further validating the effectiveness of our bounds for assessing the generalization of SGD. Moreover, we provide numerical validation of the presented bounds and demonstrate that our trajectory-based bound is tighter than the result in \cite{wang2022generalization}. Additionally, compared with norm-based bounds, we show that the terminal-state-based bound that integrates the geometric properties of local minima can better characterize generalization.
    % Trivially choosing the prior distribution to be a Gaussian independent of training data (which reduces to MI bound of \citet{xu2017information}), we obtain a bound
% (Theorem~\ref{thm:pacbayes-isotropic-prior}) in terms of the distance of the weight output from SGD  to the prior estimate (e.g. initialization of weights). It is more interesting  choosing the prior as the steady-state weight distribution obtained by SGD on the same training set but with one example held out. In this case, the bound we obtain (Theorem \ref{thm:pacbayes-data-dependent-prior}) can be elegantly expressed using the influence function \citep{koh2017understanding}, which suggests that the generalization of the SGD is related to the stability of SGD. Comparing to the first family of bounds (i.e., those based on the MI between the training set and training trajectories), the second family of bounds directly bound the generalization error via the terminal state, which avoids summing over training steps; these bounds are in general tighter when the steady-state estimates are accurate. On the other hand, not relying on the steady-state estimates and the approximating assumptions they base upon is arguably an advantage of the first family.
\end{itemize}

\section{Preliminary}
% \vspace{-2mm}
% Unless otherwise noted,  a random variable will be denoted by a capitalized letter, and  its realization is denoted by the corresponding lower-case letter. 
% We use $\mathbb{E}^Y[X] = \mathbb{E}[X|Y]$ as the the conditional expectation.
% For random variables $X$ and $Y$, we denote $\mathbb{E}^Y[X] = \mathbb{E}[X|Y]$ as the the conditional expectation.

\paragraph{Notation}
% Unless otherwise noted, a random variable will be denoted by a capitalized letter, and  its realization by the corresponding lower-case letter. 
The distribution of a random variable $X$ is denoted by $P_X$ (or $Q_X$), and the conditional distribution of $X$ given $Y$ is denoted by $P_{X|Y}$. When conditioning on a specific realization $y$, we use the shorthand $P_{X|Y=y}$ or simply $P_{X|y}$.
Denote by $\mathbb{E}_{X}$ expectation over $X \sim P_X$, and by $\mathbb{E}_{X|Y=y}$ (or $\mathbb{E}^y_{X}$) expectation over $X \sim P_{X|Y=y}$. We may omit the subscript of the expectation when there is no ambiguity.
% The entropy of a random variable $X$ is denoted by $H(X)$, and 
The KL divergence of probability distribution $Q$ with respect to $P$ is denoted by $\mathrm{D_{KL}}(Q||P)$.
The mutual information (MI) between random variables $X$ and $Y$ is denoted by $I(X;Y)$, and the conditional mutual information between $X$ and $Y$ given $Z$ is denoted by $I(X;Y|Z)$. In addition, for a matrix $A\in\mathbb{R}^{d\times d}$, we let $\tr{A}$ denote the trace of $A$ and we use $\tr{\log{A}}$ to indicate $\sum_{k=1}^d\log{A_{k,k}}$.

% $ $\tr{\mathrm{diag}\{\log{A_{1,1}},\log{A_{2,2}}, \dots, \log{A_{d,d}}\}}=\sum_{k=1}^d\log{A_{k,k}}$.
% \textcolor{red}{trace of log of a matrix} 
% We also define the disintegrated mutual information as $I^z(X;Y) \triangleq \mathrm{D_{KL}}(P_{X,Y|Z=z}||P_{X|Z=z}P_{Y|Z=z})$, following the notation in \citep{negrea2019information}. Note that $I(X;Y|Z) = \mathbb{E}_{Z}[I^Z(X;Y)]$.

% \subsection{Expected Generalization Error}
\paragraph{Expected Generalization Error}
Let $\mathcal{Z}$ be the instance space and let $\mu$ be an unknown distribution on $\mathcal{Z}$, specifying the random variable $Z$. We
let ${\mathcal W}\subseteq \mathbb{R}^d$ be the space of hypotheses. In the information-theoretic analysis framework, there is a training sample $S=\{Z_i\}_{i=1}^n$ drawn i.i.d. from $\mu$ and a stochastic learning algorithm $\mathcal{A}$ takes the training sample $S$ as its input and outputs a hypothesis $W\in \mathcal{W}$ according to some  conditional distribution 
$Q_{W|S}$. 
% mapping ${\mathcal Z}^n$ to ${\cal W}$. 
Given a loss function $\ell: \mathcal{W}\times\mathcal{Z}\rightarrow \mathbb{R}^{+}$, where  $\ell(w, z)$ measures the ``unfitness'' or ``error'' of any $z\in {\mathcal Z}$ with respect to a hypothesis $w\in {\mathcal W}$. The goal of learning is to find a hypothesis $w$ that minimizes the population risk, and for any $w\in {\mathcal W}$, the population risk is defined as
$
L_\mu(w) \triangleq \mathbb{E}_{Z\sim \mu}[\ell(w,Z)]
$.
 In practice, since $\mu$ is only partially accessible via the sample $S$,  we instead turn to  use the empirical risk, 
 %as a proxy of the population risk, which is
 defined as
$
L_S(w) \triangleq \frac{1}{n}\sum_{i=1}^n \ell(w,Z_i)
$.
Then, the expected generalization error of %the learning algorithm 
$\mathcal{A}$ is defined as 
$
\mathcal{E}_\mu(\mathcal{A})\triangleq\mathbb{E}_{W,S}[L_\mu(W)-L_S(W)]
$,
where the expectation is taken over $(S,W)\sim\mu^n\otimes Q_{W|S}$.

%(adopting the usual notion ``surrogate loss'' \cite{shalev2014understanding})
Throughout this paper, 
% we take $\ell$ as a continuous function (adopting the usual notion ``surrogate loss'').
% Additionally, 
we assume that $\ell$ is differentiable almost everywhere with respect to $w$.
% In some cases we will assume that $\ell(w, Z)$ is $R$-subgaussian
% \footnote{A random variable $X$ is $R$-subgaussian if for any $\rho\in \mathbb{R}$, $\log {\mathbb E} \exp\left( \rho \left(X- {\mathbb E}X\right) \right) \le \rho^2R^2/2$.}
% for any $w\in\mathcal{W}$. Note that a bounded loss is guaranteed to be subgaussian. 
% for example, if $\ell(w, Z) \in [0,M]$ then $R=M/2$. 
%To simplify notations, whenever $w$ can be clearly inferred from the context, 
In addition, we will denote $\ell(w,Z_i)$ by $\ell_i$ when there is no ambiguity.
% \vspace{-2mm}

% \subsection{SGD and SDE}
\paragraph{SGD and SDE}
At each time step $t$, given the current state $W_{t-1}=w_{t-1}$, let $B_t$ be a random subset that is drawn uniformly from  $\{1,2,\dots,n\}$ and $|B_t|=b$ is the batch size. Let $\widetilde{G}_{t}\triangleq\frac{1}{b}\sum_{i\in B_t}\nabla \ell(w_{t-1},Z_i)$ be the mini-batch gradient. The SGD updating rule with learning rate $\eta$ is then
\begin{eqnarray}
 W_{t} = w_{t-1} - \eta \widetilde{G}_{t}.
 \label{eq:sgd-update}
\end{eqnarray}
The full batch gradient is $G_t\triangleq \frac{1}{n}\sum_{i=1}^n\nabla \ell_i$. 
%Then Eq.  (\ref{eq:sgd-update}) becomes 
It follows that
\begin{eqnarray}
 W_{t} = w_{t-1} - \eta G_t + \eta V_t,
\label{eq:sgd-update-2} 
\end{eqnarray}
where $V_t \triangleq G_t-\widetilde{G}_{t}$ is the mini-batch \textit{gradient noise}. Since $\ex{B_t}{V_t}=0$, $\widetilde{G}_{t}$ is an unbiased estimator of the full batch gradient $G_t$. 
Moreover, the single-draw (i.e. $b=1$) SGD gradient noise covariance (GNC) and the mini-batch GNC are $\Sigma_t=\frac{1}{n}\sum_{i=1}^n\nabla \ell_i\nabla \ell_i^\mathrm{\bf T}-G_tG_t^\mathrm{\bf T}$ and $\quad C_t = \frac{n-b}{b(n-1)}\Sigma_t$, respectively.
% \[
% \Sigma_t=\frac{1}{n}\sum_{i=1}^n\nabla \ell_i\nabla \ell_i^T-G_tG_t^T \qquad\text{and}\quad C_t = \frac{n-b}{b(n-1)}\Sigma_t.
% \] 
If $n\gg b$, then $
C_t = {1}/{b}\Sigma_t.
$
Notice that $\Sigma_t$ (or $C_t$) is state-dependent, i.e. it depends on $w_{t-1}$. If $t$ is not specified, we use $\Sigma_w$ (or $C_w$) to represent its dependence on $w$. In addition, the population GNC at time $t$ is 
\begin{align}
    \Sigma_t^\mu\triangleq&\ex{Z}{\nabla \ell(w_{t-1},Z)\nabla \ell(w_{t-1},Z)^\mathrm{\bf T}}\notag\\
    &\quad-\ex{Z}{\nabla \ell(w_{t-1},Z)}\ex{Z}{\nabla \ell(w_{t-1},Z)^\mathrm{\bf T}}.
\label{eq:population-gradient-noise}
\end{align}
% We will approximate $V_t$ up to its second moment. %then by the central limit theorem, 
%we can make the following assumption.
% , which lets the type of the gradient noise distribution be Gaussian.
We assume that the initial parameter $W_0$ is independent of all other random variables, and SGD stops after $T$ updates, outputting $W_T$ as the learned parameter.

We now approximate $V_t$ up to its second moment, e.g., $V_t\sim \mathcal{N}(0,C_t)$, then we have the following continuous-time evolution, i.e. It\^o SDE:
\begin{align}
    d \omega = - \nabla L_S(\omega) dt + [\eta C_{\omega}]^{\frac{1}{2}} d\theta_t,
    \label{eq:ito-sde}
\end{align}
where $C_{\omega}$ is the GNC at $\omega$ and $\theta_t$ is a Wiener process. 
% This SDE is also called .
Furthermore, the {\em Euler-Maruyama} discretization, as the simplest approximation scheme to It\^o SDE in Eq.~(\ref{eq:ito-sde}), is
\begin{align}
  W_{t} = w_{t-1} - \eta G_t + \eta C_t^{1/2}N_t,
\label{eq:sgd-update-gaussian}
\end{align}
where $N_t\sim\mathcal{N}(0,\mathrm{I}_d)$ is the standard Gaussian random variable. 
% \vspace{-2mm}

\paragraph{Validation of SDE} It is important to understand how accurate is the SDE in Eq.~(\ref{eq:ito-sde}) for approximating the SGD process in Eq.~(\ref{eq:sgd-update}). 
% It is important to understand how accurate of SDE in Eq.~(\ref{eq:ito-sde}) for approximating the SGD process in Eq.~(\ref{eq:sgd-update}).
Previous research, such as \citep{li2017stochastic,li2019stochastic}, has provided theoretical evidence supporting the idea that SDE can approximate SGD in a ``weak sense''. That is, the SDE processes closely mimic the original SGD processes, not on an individual sample path basis, but rather in terms of their distributions (see Lemma~\ref{lem:sde-weak} for a formal result).

Additionally, concerning the validation of the discretization of SDE in Eq.~(\ref{eq:sgd-update-gaussian}), \citet[Theorem~2]{wu2020noisy} has proved that Eq.~(\ref{eq:sgd-update-gaussian}) is {\em an order $1$ strong approximation} to SDE in Eq.~(\ref{eq:ito-sde}). 
% We defer their theoretical result to Appendix. 
Moreover, 
we direct interested readers to the comprehensive investigations carried out by 
\citep{wu2020noisy,li2021validity}, where the authors empirically verify that SGD and Eq.~(\ref{eq:sgd-update-gaussian}) can achieve the similar testing performance in the deep learning scenarios, suggesting that non-Gaussian noise may not be essential to SGD performance. In other words, studying Eq.~(\ref{eq:sgd-update-gaussian}) is arguably sufficient to understand generalization properties of SGD. In Figure~\ref{fig:Acc-Dynamics}, we also empirically verify the approximation of Eq.~(\ref{eq:sgd-update-gaussian}), and show that it can effectively  capture the behavior of SGD.
\paragraph{Two Information-Theoretic Bounds} 
% Before we delve into the formal justification, we present below two typical information-theoretic bounds in the literature.
% Below are two typical information-theoretic bounds.
The original 
% version of mutual information (MI) based 
information-theoretic bound in \cite{xu2017information} is a sample-based MI bound, whose main component is the mutual information between the output $W$ and the entire input sample $S$. This result is given as follows:
\begin{lem}[{\citet[Theorem~1.]{xu2017information}}]
Assume the loss $\ell(w,Z)$ is $R$-subgaussian\footnote{A random variable $X$ is $R$-subgaussian if for any $\rho\in \mathbb{R}$, $\log {\mathbb E} \exp\left( \rho \left(X- {\mathbb E}X\right) \right) \le \rho^2R^2/2$. Note that a bounded loss is guaranteed to be subgaussian. }
for any $w\in\mathcal{W}$, then
% \vspace{-5pt}
\[
|\mathcal{E}_\mu(\mathcal{A})|\leq \sqrt{\frac{2R^2}{n}I(W;S)}.
\]
% \vspace{-10pt}
% where $I(W;S)=\mathrm{D_{KL}}(Q_{W,S}||Q_W\otimes Q_S)$ is the mutual information and $\mathrm{D_{KL}}$ denotes the KL divergence.
% \cite{cover2012elements} between $W$ and $S$.
\label{lem:xu's-bound}
\end{lem}

This bound is further 
% improved by 
refined to a data-dependent prior based bound. Following the setup in \citet{negrea2019information}, let $J$ be a random subset uniformly drawn from $\{1,\dots,n\}$ and $|J|=m>b$. Let $S_J = \{Z_i\}_{i\in J}$.
% and $L_{S_J}(W) = \frac{1}{m}\sum_{i\in J}\ell(W,Z_i)$. 
Typically, we choose $m=n-1$, then the following result is known.%\looseness=-1  
% and \citet[Theorem~1.]{wang2021optimizing}.
\begin{lem}[{\citet[Theorem~2.5]{negrea2019information}}]
\label{lem:data-dependent-prior}
Let $Q_{W|S}$ be the posterior distribution of $W$ given the training sample $S$, and let $P_{W|S_J}$ be the posterior distribution of $W$ given the training sample $S_J$. Assume the loss $\ell(w,Z)$ is bounded in $[0,M]$, then for any $P_{W|S_J}$,
\[
\mathcal{E}_\mu(\mathcal{A})\leq\frac{M}{\sqrt{2}}\mathbb{E}_{S,J}{\sqrt{\mathrm{D_{KL}}(Q_{W|S}||P_{W|S_J})}}.
\]
\end{lem}
% \begin{rem}
% %In Lemma~\ref{lem:data-dependent-prior}, 
% We may use a subset $S_J$ drawn from the training sample $S$ to conduct a parallel training process based on the same algorithm $\mathcal{A}$ (e.g. SGD) to obtain a data-dependent prior ($P_{W|S_J}$). Then Lemma~\ref{lem:data-dependent-prior} indicates that the generalization error can be bounded by the KL divergence between the distribution of $W$ that obtained by the real training algorithm and this data-dependent prior.
% \end{rem}
 Note that $J$ is drawn before the training starts and is independent of $\{W_t\}_{t=0}^T$. 
 We use the subset $S_J$ to conduct a parallel SGD training process based 
 on the same algorithm $\mathcal{A}$ (e.g. SGD) 
 to obtain a data-dependent prior ($P_{W|S_J}$). When $m=n-1$,  we call this prior process the leave-one-out (LOO) prior.

% We also present the variational representation of mutual information below.
% \begin{lem}[{\citet[Corollary~3.1.]{polyanskiy2019lecture}}]
% \label{lem:mi-center-gravity}
% For two random variables $X$ and $Y$, we have
% \[
% I(X;Y) = \inf_{P} \ex{X}{\mathrm{D_{KL}}(Q_{Y|X}||P)},
% \]
% where the infimum is achieved at $P=Q_Y$.
% \end{lem}

% \vspace{-2mm}
\section{Generalization Bounds Via Full Trajectories}
\label{sec:itb-sde}
We now discuss the generalization of SGD under the approximation of Eq.~(\ref{eq:sgd-update-gaussian}). 
We first unroll the terminal parameters' mutual information $I(W_T;S)$ to the full trajectories' mutual information via the lemma below.
\begin{lem}
\label{lem:mi-unroll}
The MI term in Lemma~\ref{lem:xu's-bound} is upper bounded by
$I(W_T;S)\leq\sum_{t=1}^T I(- G_t + C_t^{1/2}N_t;S|W_{t-1} ).
$
\end{lem}

This lemma can be proved by recurrently applying the data processing inequality (DPI) and chain rule of the mutual information \citep{polyanskiy2019lecture}.

Define $\widehat{G}_t = -G_t+C_t^{1/2}N_t$. Let $Q_{\hat{G}_t|s,w_{t-1}}$ and $Q_{\hat{G}_t|w_{t-1}}$ be the conditional and marginal distributions fully characterized by the algorithm, respectively. In addition, let $P_{\widehat{G}_t|w_{t-1}}$ be any prior distribution of $\widehat{G}_t$, satisfying $\mathrm{D_{KL}}(Q_{\widehat{G}_t|s,w_{t-1}}||P_{\widehat{G}_t|w_{t-1}})<\infty$.
we first have the following lemma.
% By Lemma \ref{lem:mi-center-gravity}, we have
\begin{lem}
\label{lem:cmi-golden formula}
% Let $\widehat{G}_t = -G_t+C_t^{1/2}N_t$. 
% then
% Let $P_{\widehat{G}_t|w_{t-1}}$ be any prior, 
% satisfying $\mathrm{D_{KL}}(P_{\widehat{G}_t|S,W_{t-1}}||P_{\widehat{N}_t|W_{t-1}})<\infty$.
% At every time step $t$, let $P_{\widehat{G}_t|W_{t-1}}$ be any distribution satisfying $\mathrm{D_{KL}}(P_{\widehat{G}_t|W_{t-1}}||P_{\widehat{G}_t|W_{t-1}})<\infty$, 
For any time step $t$,
we have
% $
% I(\widehat{G}_t;S|W_{t-1})\!\!=\!\! \ex{}{\inf_{P_{\widehat{G}_t|W_{t-1}}} \cex{}{}{\mathrm{D_{KL}}(Q_{\widehat{G}_t|S,W_{t-1}}||P_{\widehat{G}_t|W_{t-1}})}}$,
$
I(\widehat{G}_t;S|W_{t-1}) = \ex{W_{t-1}}{\inf_{P_{\widehat{G}_t|W_{t-1}}} \cex{S}{W_{t-1}}{\mathrm{D_{KL}}(Q_{\widehat{G}_t|S,W_{t-1}}||P_{\widehat{G}_t|W_{t-1}})}}$,
where the infimum is achieved when the prior distribution $P_{\widehat{G}_t|w_{t-1}} = Q_{\widehat{G}_t|w_{t-1}}$ .
\end{lem}
% \begin{rem}
% Lemma \ref{lem:cmi-golden formula} suggests that for every step $t$, the conditional MI between $\widehat{G}_t$ and $S$ could be upper bounded by the expected KL divergence between $Q_{\widehat{G}_t|S,W_{t-1}}$ and some prior $P_{\widehat{G}_t|W_{t-1}}$. 

% In Lemma~\ref{lem:cmi-golden formula}, $\mathrm{D_{KL}}(Q_{\widehat{G}_t|S,W_{t-1}}||P_{\widehat{G}_t|W_{t-1}})$ may be viewed as an estimate of the sensitivity of the full batch gradient to a specific training sample $S=s$. Later on we will show that if full batch gradient is close to the population gradient (i.e. $\ex{Z}{\nabla\ell(w,Z)}$), 
 %indicating less sensitivity of the full batch gradient w.r.t $S$, 
 % this sensitivity is small, leading to better generalization performance.
 
% \end{rem}
% While $\widehat{G}_t$ is the gradient signal at step $t$, 
 
% Moreover, 
Lemma~\ref{lem:cmi-golden formula} suggests that every choice of the prior $P_{\widehat{G}_t|W_{t-1}}$ gives rise to an upper bound of the MI of interest via 
% $I(\widehat{G}_t;S|W_{t-1}) \le \ex{W_{t-1}}{ \cex{S}{W_{t-1}}{\mathrm{D_{KL}}(Q_{\widehat{G}_t|S,W_{t-1}}||P_{\widehat{G}_t|W_{t-1}})}}$
$I(\widehat{G}_t;S|W_{t-1}) \le \ex{}{{\mathrm{D_{KL}}(Q_{\widehat{G}_t|S,W_{t-1}}||P_{\widehat{G}_t|W_{t-1}})}}$. The closer is $P_{\widehat{G}_t|W_{t-1}}$ to $Q_{\widehat{G}_t|W_{t-1}}$, the tighter is the bound. As the simplest choice, we will first choose an isotropic Gaussian prior, $P_{\widehat{G}_t|w_{t-1}} = \mathcal{N}(\tilde{g}_t, \sigma^2_t\mathrm{I}_d)$ (where both $\tilde{g}_t$ and $\sigma_t$ are only allowed to depend on $W_{t-1}$), and optimize the KL divergence in Lemma~\ref{lem:cmi-golden formula} over $\sigma_t$ for a fixed $\tilde{g}_t$. 
% Then combined with Lemma \ref{lem:mi-unroll}, 
The following result is obtained.
\begin{thm}
\label{thm:isotropic-prior-bound}
Under the conditions of Lemma~\ref{lem:xu's-bound} and assume $C_t$ is a positive-definite matrix. For any $t\in [T]$, let $\tilde{g}_t$ be any constant vector for a given $w_{t-1}$, then
\begin{align}
    \mathcal{E}_{\mu}(\mathcal{A})\leq\sqrt{\frac{R^2}{n}\sum_{t=1}^T\ex{W_{t-1}}{d\log{\frac{h_1(W_{t-1})}{d}}-h_2(W_{t-1})}},
\label{ineq:iso-gen-bound}
\end{align}
% \[
% \mathcal{E}_{\mu}(\mathcal{A})\leq\sqrt{\frac{R^2}{n}\sum_{t=1}^T\ex{W_{t-1}}{d\log{\frac{\cex{S}{W_{t-1}}{\left|\left|G_t-\tilde{g}_t\right|\right|^2 +tr\left\{C_t\right\}}}{d}}-\cex{S}{W_{t-1}}{tr\left\{\log{C_t}\right\}}}},
% \]
where 
$h_1(w) = \cex{S}{w}{\left|\left|G_t-\tilde{g}_t\right|\right|^2 +tr\left\{C_t\right\}}$ and
  $h_2(w) = \cex{S}{w}{tr\left\{\log{C_t}\right\}}$. 
  % and
% \begin{eqnarray*}
%   A_1(t) &=& \ex{}{(\left|\left|G_t-\tilde{g}\right|\right|^2 +tr\left\{C_t\right\}},\\
%   A_2(t) &=& \ex{}{tr\left\{\log{C_t}\right\}},
% \end{eqnarray*}
  % $tr\{\cdot\}$ denotes the trace of a matrix. 
  % and $\mathbb{E}_X^Y$ is the conditional expectation. 
%  Further, the optimal $\sigma_t^* = \sqrt{A_1(t)/d}$ for each step $t$.

Furthermore, if $\tilde{g}_t = \ex{Z}{\nabla\ell(w_{t-1},Z)}$, then
$h_1(w) = \frac{1}{b}\tr{\Sigma_t^\mu}$.
% \begin{align}
%     \mathcal{E}_{\mu}(\mathcal{A})\leq\sqrt{\frac{R^2}{n}\sum_{t=1}^T\ex{W_{t-1}}{d\log{\frac{\tr{\Sigma_t^\mu}}{bd}}-h_2(W_{t-1})}}.
%     \label{ineq:iso-pop-bound}
% \end{align}
\end{thm}

Notice that $\tilde{g}_t$ is any reference ``gradient'' independent of $S$, then the first term in $h_1(W_{t-1})$, $\eucd{G_t-\tilde{g}_t}^2$, characterizes the sensitivity of the full-batch gradient to some variation of the training set $S$,  while the second term in $h_1(W_{t-1})$, i.e. $\tr{C_t}$, reflects the gradient noise magnitude induced by the mini-batch based training. For example, if $\tilde{g}_t = \ex{Z}{\nabla\ell(w_{t-1},Z)}$, then $\cex{S}{w_{t-1}}{\eucd{G_t-\tilde{g}_t}^2}$ is the variance of the gradient sample mean, and such $\tilde{g}_t$ will eventually convert $h_1(W_{t-1})$ to the population GNC, namely $h_1(W_{t-1}) = \frac{1}{b}\tr{\Sigma_t^\mu}$.
% (i.e. Eq.~(\ref{eq:population-gradient-noise})) as shown in Eq.~(\ref{ineq:iso-pop-bound}). 

% In addition, 
% seen from Eq.~(\ref{ineq:iso-gen-bound}) in Theorem~\ref{thm:isotropic-prior-bound}, the GNC impacts the generalization error non-monotonically. On the one hand, we hope the diagonal elements in $C_t$ have small values so that $h_1(W_{t-1})$ is small. This corresponds to the case where most gradient signals align with each other or the gradient norm is small, a phenomenon typically occurring when training approaches local minima. On the other hand,  large diagonal elements in $C_t$ can increase the value of $h_2(W_{t-1})$, giving a smaller bound value. This coincides with the intuition from the optimization perspective, where larger gradient noise magnitude helps to escape saddle points. Reportedly this non-monotonicity is controlled by the ratio of the learning rate to the batch size in practice \citep{hoffer2017train,jastrzkebski2017three}. We note that such trade-off of gradient noise has not been observed in the previous information-theoretic analyses for SGLD.

% \begin{rem}
Moreover, if we simply let $\tilde{g}_t = 0$, then Theorem \ref{thm:isotropic-prior-bound} indicates that one can control the generalization performance via controlling the gradient norm along the entire training trajectories, e.g., if we further let $b=1$, then $h_1(W_{t-1})=\frac{1}{n}\sum_{i=1}^n||\nabla\ell_i||^2$. 
% Note that controlling gradient norm can also control the magnitude of the trace of gradient noise covariance. 
This is consistent with the existing practice, for example, applying gradient clipping \citep{wang2022generalization,geiping2021stochastic} and gradient penalty \citep{jastrzebski2021catastrophic,barrett2020implicit,smith2020origin,geiping2021stochastic} as regularization 
techniques to improve generalization. 

As a by-product, we recover previous information-theoretic bounds for the Gradient Langevin dynamics (GLD) with noise distribution $\mathcal{N}(0,\eta^2 \mathrm{I}_d)$ below.
% It is possible to compare Theorem~\ref{thm:isotropic-prior-bound} with some previous  bounds by letting $C_t=\mathrm{I}_d$, in which Eq.~(\ref{eq:sgd-update-gaussian}) reduces to the Gradient Langevin dynamics (GLD) with noise distribution $\mathcal{N}(0,\eta^2 \mathrm{I}_d)$ and the corresponding generalization bound is stated below.
\begin{cor}
\label{cor:langevin-dynamic}
If $C_t=\mathrm{I}_d$, then
\[
\mathcal{E}_{\mu}(\mathcal{A})\leq\sqrt{\frac{R^2d}{n}\sum_{t=1}^T{{\mathbb{E}_{W_{t-1}}{\log\left({\mathbb{E}_{S}^{W_{t-1}}{\left|\left|G_t-\tilde{g}_t\right|\right|^2}}/{d}+1\right)}}}}.
\]
\end{cor}
% It is easy to verify 
Note that the bound in Corollary \ref{cor:langevin-dynamic} can recover the bound in \citet[Proposition~3.]{neu2021information} by using the inequality
% the inequality 
$\log(x+1)\leq x$. 
% for $x>0$ 
Furthermore, it can also recover the bound in \citet{pensia2018generalization} because we use a state-dependent quantity $\cex{S}{w_{t-1}}{||G_t-\tilde{g}_t||^2}$, 
% rather than 
which is smaller than the global Lipschitz constant used in \citet{pensia2018generalization}.

%Intuitively, if the gradient noise distribution is an isotropic Gaussian, independent of the parameter and the distribution, then our optimal isotropic Gaussian prior used in Lemma \ref{lem:cmi-golden formula} will match the posterior distribution better, leading to a tightest bound. 

% It's important to note that this does not indicate that GLD can outperform SGD. On the one hand, the isotropic gradient noise also has some implicit impact on the term $\ex{}{||G_t-\tilde{g}_t||^2}$, and on the other hand, Langevin dynamics may require longer training time to converge to a minimum than mini-batch SGD, i.e., having larger $T$.

% Theorem \ref{thm:isotropic-prior-bound} and Corollary \ref{cor:langevin-dynamic} both indicate that one can control the generalization performance via controlling the gradient norm along the entire training trajectories. Note that controlling gradient norm can also control the magnitude of the trace of gradient noise covariance. This is consistent with many previous practical applications, for example, applying gradient clipping \citep{wang2022generalization,geiping2021stochastic} and gradient penalty \citep{jastrzebski2021catastrophic,barrett2020implicit,smith2020origin,geiping2021stochastic} as regularization 
% techniques to improve the generalization performance. 
\begin{figure*}[!ht]
    \centering
    \begin{subfigure}[b]{0.245\textwidth}
\includegraphics[scale=0.28]{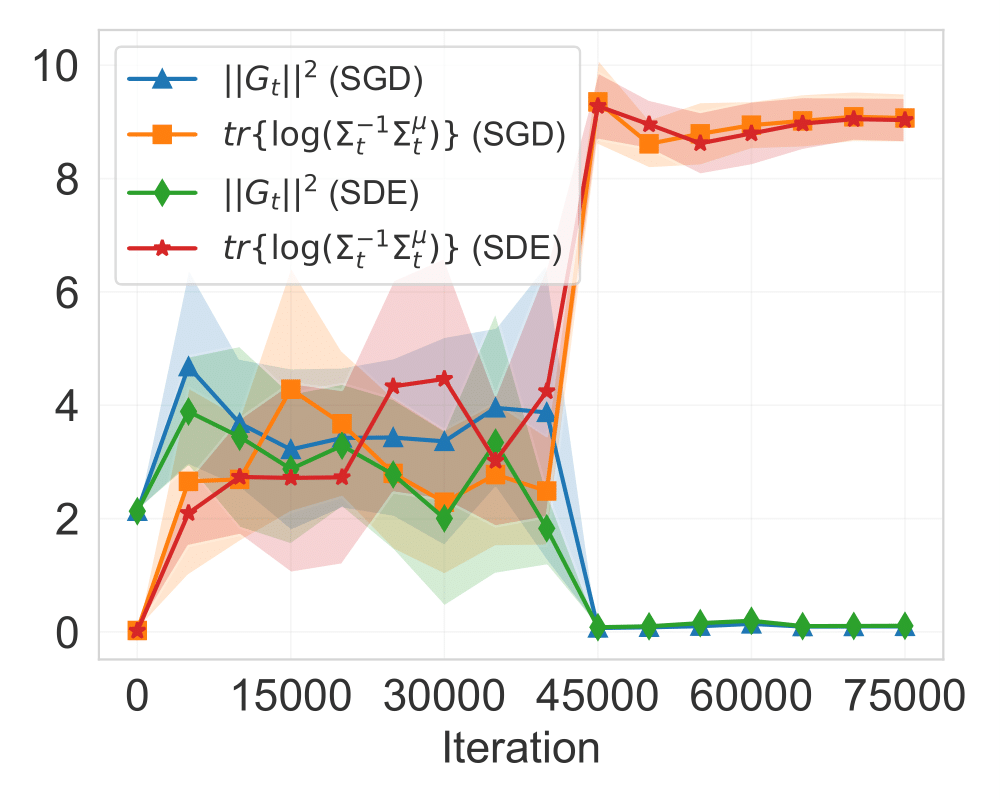}    
\caption{VGG on (small) SVHN}            \label{fig:vgg-svhn-cov}
    \end{subfigure}
\begin{subfigure}[b]{0.245\textwidth}
\includegraphics[scale=0.28]{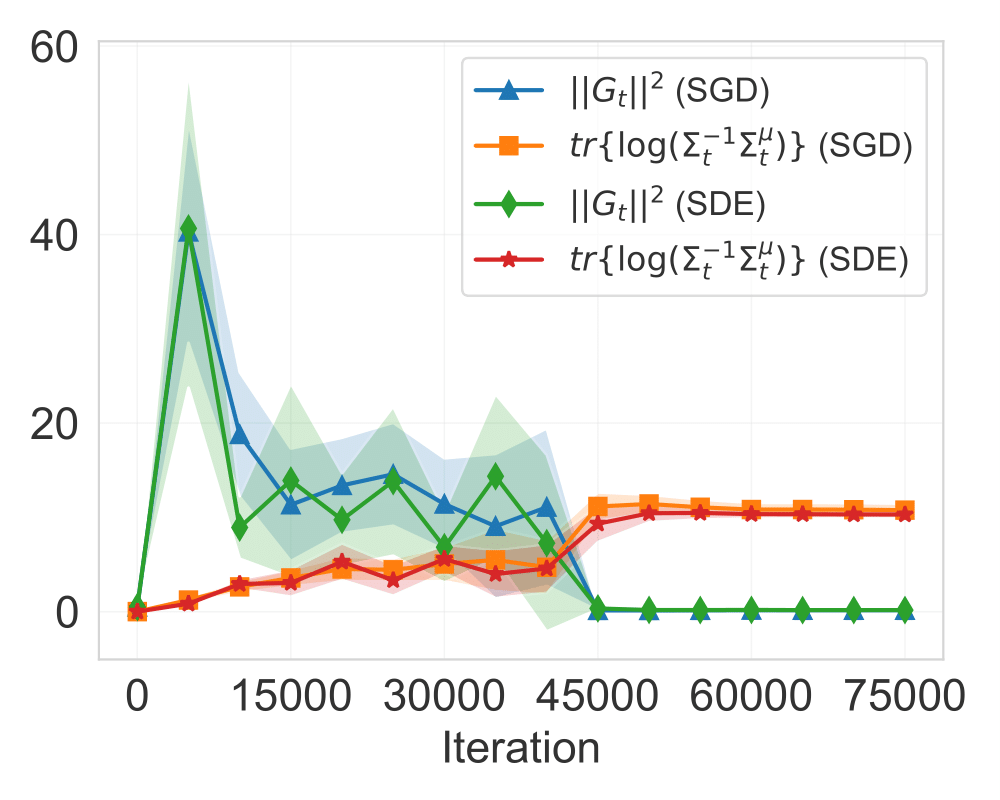}
\caption{VGG on CIFAR10}
    \label{fig:vgg-cifa10-cov}
\end{subfigure}
 \begin{subfigure}[b]{0.245\textwidth}
\includegraphics[scale=0.28]{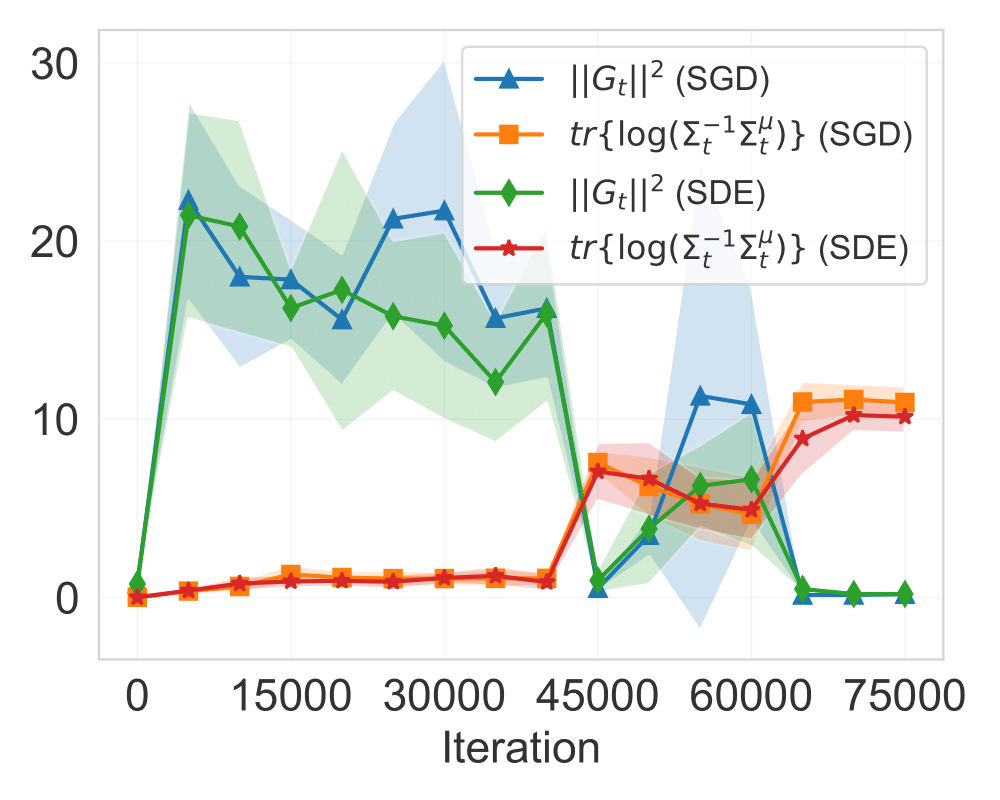}
\caption{ResNet on CIFAR10}
\label{fig:resnet-cifa10-cov}
    \end{subfigure}
\begin{subfigure}[b]{0.245\textwidth}
\includegraphics[scale=0.28]{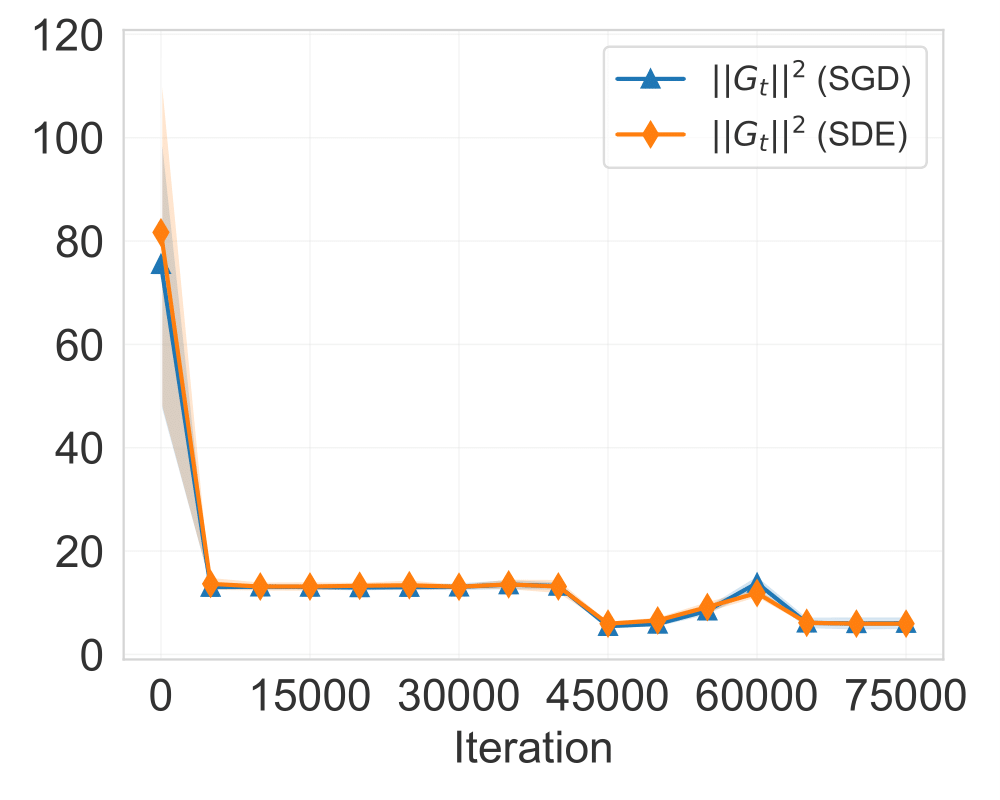}
\caption{ResNet on CIFAR100}
\label{fig:resnet-cifa100-cov}
\end{subfigure}
\caption{Gradient-related quantities of SGD or its discrete SDE approximation. In (d), since per-sample gradient is ill-defined when BatchNormalization is used, we do not track $\tr{\log\pr{\Sigma^{-1}_t\Sigma^{\mu}_t}}$.}\label{fig:Cov-Dynamics}
% \vspace{-4mm}
\end{figure*}
% Arguably, the  upper  bound  in Theorem \ref{thm:isotropic-prior-bound}  has  the  deficiency  of being dependent on an isotropic Gaussian prior. 
While choosing the isotropic Gaussian prior is common in the GLD or SGLD setting, given that we already know $C_t$ is an anisotropic covariance, one can select an anisotropic prior to better incorporate the geometric structure in the prior distribution. A natural choice of the covariance is a scaled population GNC, namely $\tilde{c}_t\Sigma_t^{\mu}$, where $\tilde{c}_t$ is some positive state-dependent scaling factor. Let $\tilde{g}_t = \ex{Z}{\nabla\ell(w_{t-1},Z)}$ be the state-dependent mean. 
% The following bound is achieved 
By optimizing over $c_t$, we have the bound below.

\begin{thm}
\label{thm:anisotropic-prior-bound}
Under the conditions of Lemma~\ref{lem:xu's-bound} and assume $C_t$ and $\Sigma^\mu_t$ are positive-definite matrices,
% Let $P_{\widehat{G}_t|w_{t-1}}=\mathcal{N}\pr{\ex{Z}{\nabla\ell(w_{t-1},Z)},\tilde{c}_t\Sigma_t^{\mu}}$, 
then
% \[
% \mathcal{E}_{\mu}(\mathcal{A})\leq\sqrt{\frac{R^2}{n}\sum_{t=1}^T\ex{W_{t-1}}{A_3(t)}+\frac{R^2dT}{n}\log\frac{1}{b}},
% \]
% where $A_3(t)=tr\left\{\log\Sigma^\mu_t-\cex{S}{W_{t-1}}{\log C_t}\right\}$.
\[
\mathcal{E}_{\mu}(\mathcal{A})\leq\sqrt{\frac{R^2}{n}\sum_{t=1}^T\ex{W_{t-1},S}{\tr{\log\frac{\Sigma^\mu_tC_t^{-1}}{b}}}}.
\]
% where $A_2(t)$ is defined as the same in Theorem~\ref{thm:isotropic-prior-bound}.
\end{thm}
\begin{rem}
    If we let the diagonal element of ${\Sigma^\mu_t}$ in dimension $k$ be $\alpha_t(k)$ and let the corresponding diagonal element of $\Sigma_t$ be $\beta_t(k)$, and assume $n\gg b$ (so $\Sigma_t = bC_t$), then $\tr{\log({\Sigma^\mu_tC^{-1}_t}/{b})}=\sum_{k=1}^d\log{\frac{\alpha_t(k)}{\beta_t(k)}}$.
Thus, Theorem~\ref{thm:anisotropic-prior-bound} implies that a favorable alignment between the diagonal values of $\Sigma_t$ and $\Sigma^\mu_t$ will positively impact generalization performance. In other words, the perfect alignment of these two matrices indicates that SGD is insensitive to the randomness of $S$. Recall the key quantity in Lemma~\ref{lem:xu's-bound}, $I(W;S)$, which also measures the dependence of $W$ with the randomness of $S$, the term $\Sigma^{\mu}_{t}\Sigma^{-1}_t$ conveys a similar intuition in this context.
\end{rem}

Compared with Theorem~\ref{thm:isotropic-prior-bound} under the same choice of $\tilde{g}_t$, we notice that the main difference is that the term $tr\left\{\log({\Sigma^\mu_t}/{b})\right\}$, instead of  $d\log({tr\left\{\Sigma^\mu_t\right\}}/{bd})$, appears in the bound of Theorem~\ref{thm:anisotropic-prior-bound}. The following lemma demonstrates that Theorem~\ref{thm:anisotropic-prior-bound} is tighter than the bound in Theorem~\ref{thm:isotropic-prior-bound}.

\begin{lem}
    \label{lem:compare-iso-noniso}
    For any $t$, we have $\tr{\log\frac{\Sigma^\mu_t}{b}}\!\!\leq\! d\log{\frac{\tr{\Sigma_t^\mu}}{bd}}$, with the equality holds when all the diagonal elements in $\Sigma^\mu_t$ have the same value, i.e. $\alpha_t(1)=\alpha_t(2)=\cdots=\alpha_t(d)$.
\end{lem}

% Let the diagonal element of ${\Sigma^\mu_t}/{b}$ in dimension $k$ be $a_k$, then
% \begin{align*}
%     \sum_{k=1}^d\log a_k\leq
%     (\sum_{k=1}^d 1) \cdot \log{(\sum_{k=1}^d a_k)}/{(\sum_{k=1}^d 1)}=d\log({tr\left\{\Sigma^\mu_t\right\}}/{bd}),
% \end{align*}
% where we invoked a variant of the Log sum inequality \citep[Theorem~2.7.1]{thomas2006elements}, See Lemma~\ref{lem:log-sum-ineq} in Appendix.
% i.e., $\sum_{i=1}^n b_i\log(a_i/b_i)\leq (\sum_{i=1}^n b_i)\log(\sum_{i=1}^n a_i/\sum_{i=1}^n b_i)$ for non-negative numbers $\{a_i\}_{i=1}^n$ and $\{b_i\}_{i=1}^n$ (See Lemma~\ref{lem:log-sum-ineq} in Appendix).
% Thus, the bound in Theorem~\ref{thm:anisotropic-prior-bound} is tighter than the bound in Eq.~(\ref{ineq:iso-pop-bound}), with the equality holds when all the diagonal elements in $\Sigma^\mu_t$ have the same value.

% \vspace{-0.3in}

The trajectory-based bounds in Theorem~\ref{thm:isotropic-prior-bound} and Theorem~\ref{thm:anisotropic-prior-bound} emphasize the significance of gradient-related information along entire trajectories, including metrics such as gradient norm and gradient covariance alignment, in comprehending the generalization dynamics of 
 understanding the generalization of SGD. In Figure~\ref{fig:Cov-Dynamics}, we visually show that these key gradient-based measures during SDE training closely mirror the dynamics observed in SGD.  
 % This further showcases that we can analyze the generalization of SGD via analyzing SDE.

 % \textcolor{red}{
% \paragraph{Dependency on Time} 
Notably, these trajectory-based information-theoretic bounds are time-dependent, indicating that these bounds may grow with the training iteration number $T$, unless the gradient norm becomes negligible at some point during training. While the stability-based bounds for GD/SGD are also time-dependent \citep{hardt2016train, bassily2020stability} (in the convex learning case), the learning rate in these bounds helps mitigate the growth of $T$. However, the learning rate does not appear in our trajectory-based information-theoretic bounds, making the dependency on $T$ even worse.
% } 

% \textcolor{red}{
Note that \citet{wang2021analyzing} uses the strong data processing inequality to reduce this deficiency, but the bound still increases with $T$. To tackle this weakness, we will invoke some asymptotic SDE results on the terminal parameters of the algorithm,  which will give us a crisp way to characterize the expected generalization gap  without decomposing the mutual information.
\section{Generalization Bounds Via Terminal State}
\label{sec:pac-bayes}
% The connection between information-theoretic bounds and PAC-Bays bounds have already been discussed in many previous works \citep{bassily2018learners,hellstrom2020generalization,alquier2021user}. Roughly speaking, the most significant component of a PAC-Bayes bound is the KL divergence between the posterior distribution of a randomized algorithm output and a prior distribution, i.e. $\mathrm{D_{KL}}(Q_{W_T|S}||P_{N})$ for some prior $P_N$. In essence, information-theoretic bounds can be view as having the same spirit. For concreteness, in Lemma \ref{lem:xu's-bound}, 
% $I(W_T;S)=\mathbb{E}_S[\mathrm{D_{KL}}(Q_{W_T|S}||P_{W_T})]$, in which case the marginal $P_{W_T}$ is used as a prior of the algorithm output. Furthermore, by using  Lemma \ref{lem:mi-center-gravity}, we have $I(W_T;S) \leq \inf_{P_N} \mathbb{E}_S[\mathrm{D_{KL}}(Q_{W_T|S}||P_N)]$. Hence, Lemma \ref{lem:xu's-bound} can be regarded as a PAC-Bayes bound with the optimal prior. In addition, the PAC-Bayes framework is usually used to provide a high-probability bound, %(with respect to the randomness of $S$), 
% while information-theoretic analysis is applied to bounding the expected generalization error. In this sense, information-theoretic framework is closer to another concept called MAC-Bayes \citep{grunwald2021pac}.
\begin{figure*}[!ht]
% \vspace{-2mm}
    \centering
    \begin{subfigure}[b]{0.245\textwidth}
\includegraphics[scale=0.28]{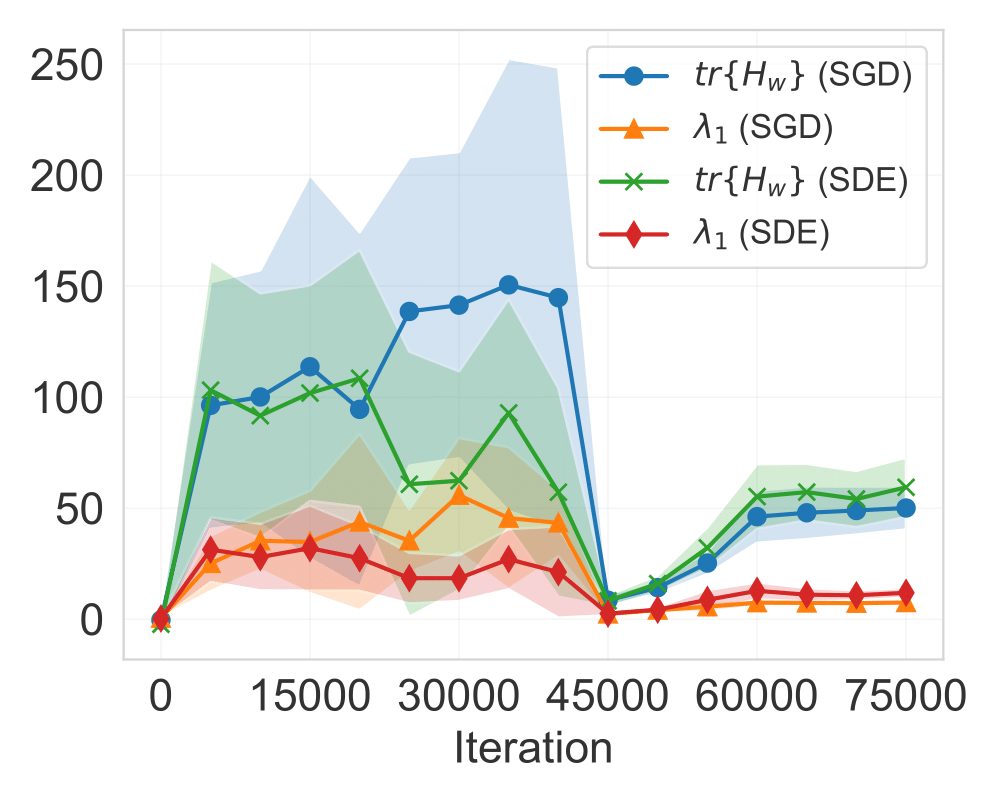}    
\caption{VGG on (small) SVHN}            \label{fig:vgg-svhn-hess}
    \end{subfigure}
\begin{subfigure}[b]{0.245\textwidth}
\includegraphics[scale=0.28]{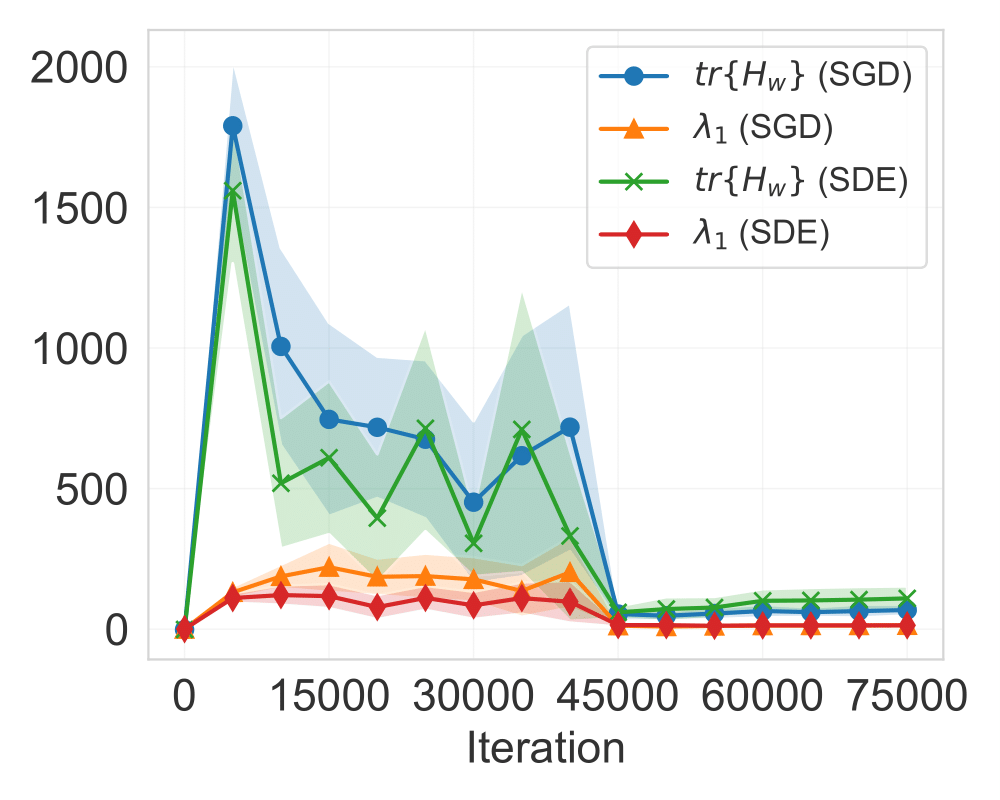}
\caption{VGG on CIFAR10}
    \label{fig:vgg-cifa10-hess}
\end{subfigure}
 \begin{subfigure}[b]{0.245\textwidth}
\includegraphics[scale=0.28]{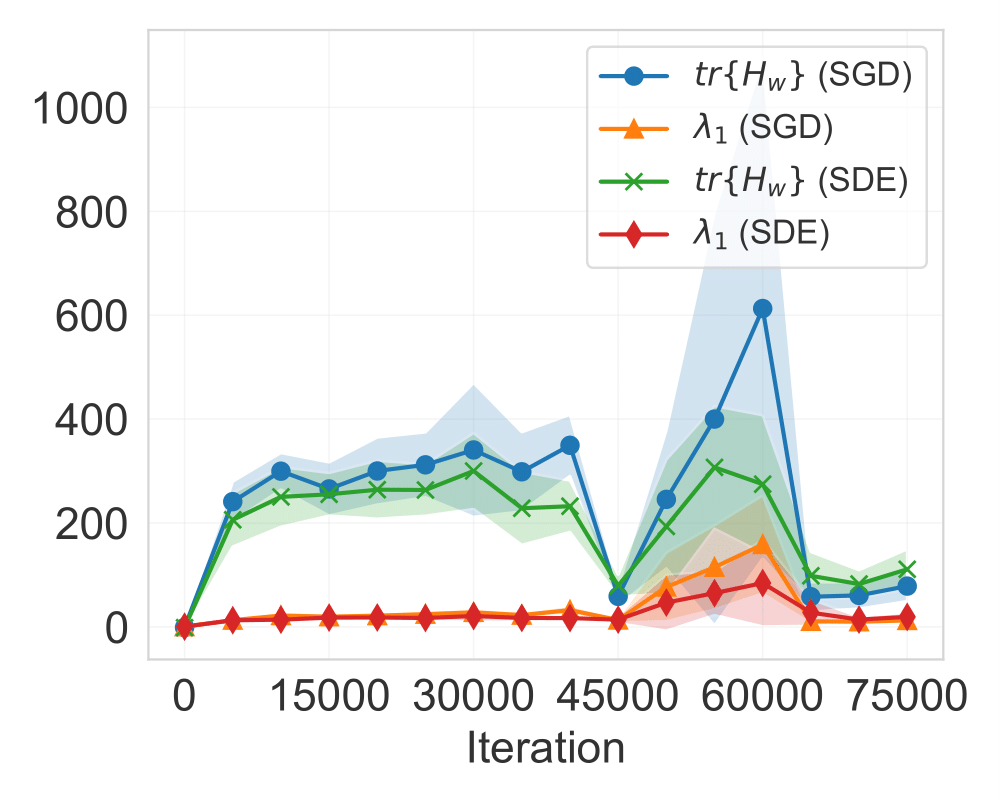}
\caption{ResNet on CIFAR10}
\label{fig:resnet-cifa10-hess}
    \end{subfigure}
\begin{subfigure}[b]{0.245\textwidth}
\includegraphics[scale=0.28]{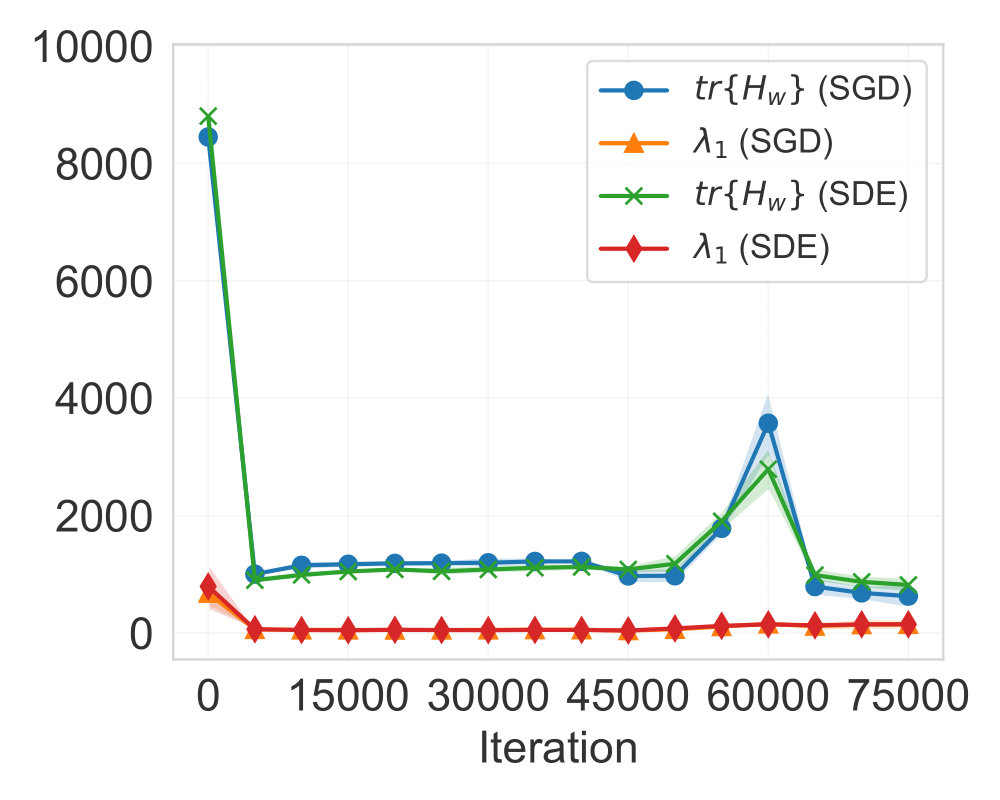}
\caption{ResNet on CIFAR100}
\label{fig:resnet-cifa100-hess}
\end{subfigure}
\caption{Hessian-related quantities of SGD or its discrete SDE approximation.}\label{fig:Hess-Dynamics}
% \vspace{-4mm}
\end{figure*}

In this section, we directly upper-bound the generalization error by the properties of the terminal state instead of using the full training trajectory information. Particularly, we will first use the stationary distribution of weights at the end of training as $Q_{W_T|S}$. 
To overcome the explicit time-dependence present in the bounds discussed in Section~\ref{sec:itb-sde}, one has to introduce additional assumptions, with these assumptions being the inherent cost. For example, an important approximation used  in this section is the quadratic approximation of the loss. Specifically, let $w_s^*$ be a local minimum for a given training sample $S=s$, when $w$ is close to $w_s^*$, we can use a second-order Taylor expansion to approximate the value of the loss at $w$, 
\begin{eqnarray}
  L_s(w) = L_s(w_s^*) + \frac{1}{2}(w-w_s^*)^\mathrm{\bf T} H_{w_s^*}(w-w_s^*).
  \label{eq:second-order-taylor}
\end{eqnarray}
where $H_{w_s^*}$ is the Hessian matrix of $s$ at $w^*$. Note that in this case, when $w_t\to w_s^*$, we have $G_t=\nabla L_s(w_t)=H_{w_s^*}\pr{w_t-w_s^*}$. Our remaining analysis assumes the validity of Eq.~(\ref{eq:second-order-taylor}).
% two typical choices of the covariance of posterior $Q_{W_T|S}$, the first one is the stationary distribution of weights at the end of training, in which case the model fluctuates around a local minimum, and the second one is the inverse Fisher information matrix (FIM).

% \subsection{Steady-State Covariance as Posterior Covariance}

% Given that SGD can be well approximated by SDE (e.g., Eq.~\ref{eq:ito-sde}, 
In view of Eq.~(\ref{eq:second-order-taylor}), a classical result by \citet{mandt2017stochastic} shows that the posterior distribution $Q_{W|s}$ around $w_s^*$ is a Gaussian distribution $\mathcal{N}(w_s^*, \Lambda_{w_s^*})$, where $\Lambda_{w^*}\triangleq\ex{}{(W-w^*)(W-w^*)^\mathrm{\bf T}}$ is the covariance of the stationary distribution (see Appendix~\ref{sec:Gaussian-local} for an elaboration).  Furthermore, in the context of nonconvex learning, such as deep learning, where multiple local minima exist, we have multiple $w_s^*$ for a give $S=s$. Therefore, it is necessary to treat the local minimum itself as a random variable for a fixed $s$, denoted as $W_s^*\sim Q_{W_s^*|s}$. In this case, we have $Q_{W|s,w_s^*}=\mathcal{N}(w_s^*, \Lambda_{w_s^*})$ and the posterior distribution $Q_{W|s}=\cex{W_s^*}{s}{\mathcal{N}(W_s^*, \Lambda_{W_s^*})}$ should be a mixture of Gaussian distributions.

In addition, recall that $I(W_T;S)=\inf_{P_{W_{T}}}\mathbb{E}_{S}{\mathrm{D_{KL}}(Q_{W_T|S}||P_{W_{T}})}$ where $P_{W_{T}}=Q_{W_{T}}$ achieves the infimum. Here, the oracle prior $Q_{W_{T}}=\ex{S,W_S^*}{\mathcal{N}(W_S^*, \Lambda_{W_S^*})}$ is also a mixture of Gaussian distributions. From a technical standpoint, given that the KL divergence between two mixtures of Gaussian distributions does not have a closed-form expression, we turn to analyze its upper bound, namely
% we may not be able to direct compute or lower bound $I(W_T;S)$ without further knowledge even if we know $Q_{W|s,w_s^*}$ is a Gaussian distribution. Thus, 
% we turn to analyze its upper bound, namely 
$\inf_{P_{W_{T}}}\mathbb{E}_{S,W_S^*}{\mathrm{D_{KL}}(Q_{W_T|S,W_S^*}||P_{W_{T}})}$.
% which is an upper bound of $I(W_T;S)$. 
When each $s$ has only one local minimum, $I(W;S)$ reaches this upper bound. \looseness=-1

We are ready to give the terminal-state-based bounds.

\begin{thm}
\label{thm:opt-state-inde-bound}
    % Let $P_{W_{T}}=\mathcal{N}\pr{w^*_{\mu}, \Lambda^{\mu}_{w^*_{\mu}}}$, where $w^*_{\mu}=\ex{}{W_S^*}$ is the average ERM solution, and $\Lambda_{w^*_{\mu}}=\ex{W_T}{\pr{W_T-w^*_{\mu}}\pr{W_T-w^*_{\mu}}^T}$ is the population stationary covariance. 
    % Under the conditions in Lemma~\ref{lem:xu's-bound}. 
    Let $w^*_{\mu}=\ex{}{W_S^*}$ be the expected ERM solution and let $\Lambda_{w^*_{\mu}}=\ex{}{\pr{W_T-w^*_{\mu}}\pr{W_T-w^*_{\mu}}^\mathrm{\bf T}}$ be its corresponding stationary covariance, then
    \begin{align*}
        \mathcal{E}_{\mu}(\mathcal{A})\leq\frac{R}{\sqrt{2n}}\sqrt{\ex{S,W_S^*}{\tr{\log\pr{\Lambda^{-1}_{W^*_S}\Lambda_{w^*_\mu}}}}}. 
        % \label{ineq:state-bound-1}
        % \\
        % &\left.\sqrt{d\log\pr{\frac{\ex{}{\mathrm{d}^2_{\mathrm{M}}\pr{W_S^*,w^*_\mu;\ex{}{\Lambda_{W^*_S}}}}}{d}+1}+\ex{}{\tr{\log\pr{\Lambda_{W^*_S}^{-1}\ex{}{\Lambda_{W^*_S}}}}}}\right\},\label{ineq:state-bound-2}
    \end{align*}
    % where $\mathrm{d}_{\mathrm{M}}\pr{x,y;\Sigma}\triangleq \sqrt{(x-y)^T\Sigma^{-1}(x-y)}$ is the Mahalanobis distance, and all the expectation above are taken over $\pr{S,W_S^*}\sim Q_{S,W_S^*}$.
\end{thm}
This result bears resemblance to Theorem~\ref{thm:anisotropic-prior-bound} since both involve the alignment between a population covariance matrix and a sample (or batch) covariance matrix.

Note that $\Lambda_{w^*_{\mu}}\!=\!\ex{}{\pr{W^*_S-w^*_{\mu}}\pr{W^*_S-w^*_{\mu}}^\mathrm{\bf T}}\!+\!\ex{}{\Lambda_{W^*_S}}$. By Jensen's inequality, we can move the expectation over $W_s^*$  inside the logarithmic function. Additionally, if $\ex{W^*_s}{\Lambda_{W^*_s}^{-1}\ex{}{\Lambda_{W^*_s}}}$ is close to the identity matrix---especially evident in scenarios where each $s$ has only one minimum, as in convex learning---we obtain the upper bound 
$
\mathcal{O}\pr{\sqrt{{\ex{}{\mathrm{d}^2_{\mathrm{M}}\pr{W_S^*,w^*_\mu;{\Lambda_{W^*_S}}}}}/{n}}},
$
where $\mathrm{d}_{\mathrm{M}}\pr{x,y;\Sigma}\triangleq \sqrt{(x-y)^{\bf T}\Sigma^{-1}(x-y)}$ is the Mahalanobis distance. Intuitively, this quantity measures the sensitivity of a local minimum to the combined randomness introduced by both the algorithm and the training sample, relative to its local geometry.

In practice, one can estimate $\Lambda_{w^*_{\mu}}$ and $\Lambda_{w^*_{s}}$ by repeatedly conducting training processes and storing numerous checkpoints at the end of each training run. This is still much easier than estimating $I(W;S)$ directly. As an alternative strategy, one may leverage the analytical expression available for $\Lambda_{w_s^*}$.
% Unlike Theorem~\ref{thm:anisotropic-prior-bound}, where we can utilize $\cex{S}{w_{t-1}}{C_t}=\frac{n-b}{bn}\Sigma_t^\mu$ to further develop a dimension-dependent lower bound in Theorem~\ref{thm:lower-bound-traj}. In Theorem~\ref{thm:opt-state-inde-bound}, it's not clear the relationship between $\ex{S,W_S^*}{\Lambda_{W_S^*}}$ and $\Lambda_{w_\mu^*}$. Thus, we need more analytic form for $\Lambda_{W_S^*}$.
\citet{mandt2017stochastic} provides such analysis and give a equation to solve for $\Lambda_{w_s^*}$. However, the result in \citet{mandt2017stochastic} relies on the unrealistic small learning rate, and the GNC in their analysis is regarded as a state-independent covariance matrix. To overcome these limitations, we give the following result under a quadratic approximation of the loss, which is refined from \citet[Theorem~1]{liu2021noise} by using the state-dependent GNC.

\begin{lem}
\label{lem:stationary-real}
% Let $H_{w^*}$ be the Hessian matrix of $s$ at $w^*$. If $ L_s(w) \approx L_s(w^*) + \frac{1}{2}(w-w^*)^T H_{w^*}(w-w^*)$ holds when $w$ is close to any local minimal $w^*$,
% If Eq. (\ref{eq:second-order-taylor}) holds, 
% then i
In the long term limit, we have
% the covariance $\Lambda_{w^*}$ satisfies
$
\Lambda_{w^*} H_{w^*} + H_{w^*} \Lambda_{w^*}-\eta H_{w^*} \Lambda_{w^*}H_{w^*} = \eta C_{T}.
$
Moreover, consider the following conditions: 
% \begin{itemize}
\begin{enumerate}[topsep=-0.1cm, parsep=-0cm, align=parleft, labelsep=0.6cm, label=(\roman*)]
    \item $H_{w^*}\Lambda_{w^*}=\Lambda_{w^*}H_{w^*}$;
    \item $H_{w^*}^{-1}\Sigma_T= \mathrm{I}_d$;
    \item $\frac{2}{\eta}\gg \lambda_1$, where $\lambda_1$ is the top-$1$ eigenvalue of $H_{w^*}$.
\end{enumerate}
% \end{itemize}
% (i) $H_{w^*}\Lambda_{w^*}=\Lambda_{w^*}H_{w^*}$; 
% % and $\Lambda_{w^*}$ commute; 
% (ii) $H_{w^*}^{-1}\Sigma_T= \mathrm{I}_d$; (iii) $\frac{2}{\eta}\gg \lambda_1$ where $\lambda_1$ is the top-$1$ eigenvalue of $H_{w^*}$. 
Then, given (i), we have $\Lambda_{w^*}\!=\! \br{H_{w^*}\pr{\frac{2}{\eta}\mathrm{I}_d\!-\! H_{w^*}}}^{-1}\!C_{T}$; given (i-ii), we have $\Lambda_{w^*}= (\frac{2}{\eta}\mathrm{I}_d-H_{w^*})^{-1}$; given (i-iii), we have $\Lambda_{w^*}  = \frac{\eta}{2b} \mathrm{I}_d$.
% If Eq.~\ref{eq:approx-hessian-gradient} holds, then
% $
% \Lambda_{w^*}=\frac{\eta}{b}(2\mathrm{I}_d-\eta H_{w^*})^{-1}.
% $
\end{lem}
% \begin{rem}
%     Note that this result does not rely on Eq.~(\ref{eq:sgd-update-gaussian}), and it holds for any type of gradient noise $V_t$ as long as its covariance is $C_t$. 
%     % In addition, the classical result of  a continues analysis of SGD in \cite[Eq.~(13)]{mandt2017stochastic} (see Lemma~\ref{lem:posterior-covariance}) also gives us $\Lambda_{w^*}  = \frac{\eta}{2b} \mathrm{I}_d$. We recover this result under a mild condition instead of assuming $\eta\ll 1$.
% \end{rem}

\begin{figure*}[!ht]
% \vspace{-3mm}
    \centering
    \begin{subfigure}[b]{0.245\textwidth}
\includegraphics[scale=0.28]{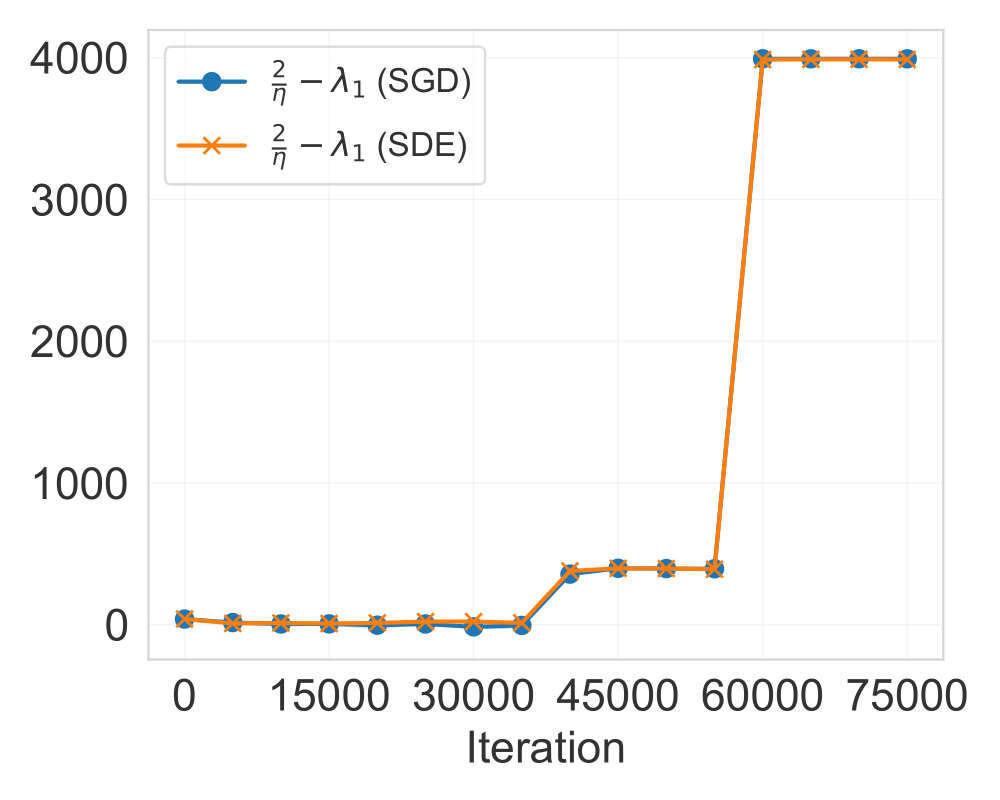}    
\caption{VGG on (small) SVHN}            \label{fig:vgg-svhn-stable}
    \end{subfigure}
\begin{subfigure}[b]{0.245\textwidth}
\includegraphics[scale=0.28]{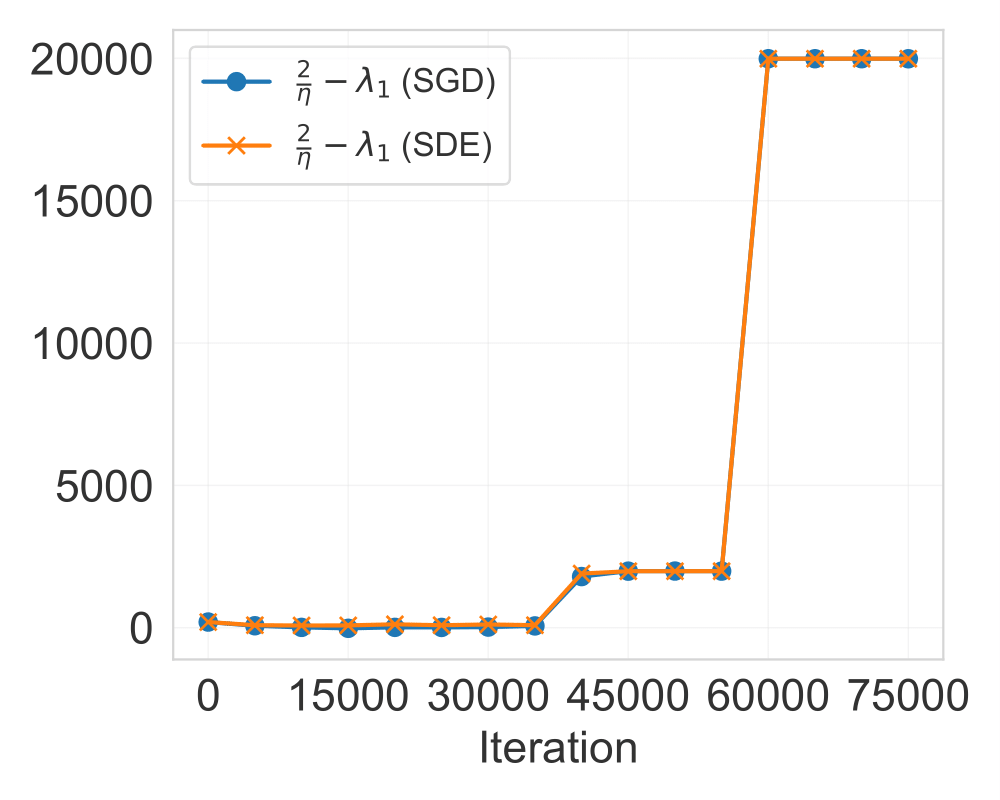}
\caption{VGG on CIFAR10}
    \label{fig:vgg-cifa10-stable}
\end{subfigure}
 \begin{subfigure}[b]{0.245\textwidth}
\includegraphics[scale=0.28]{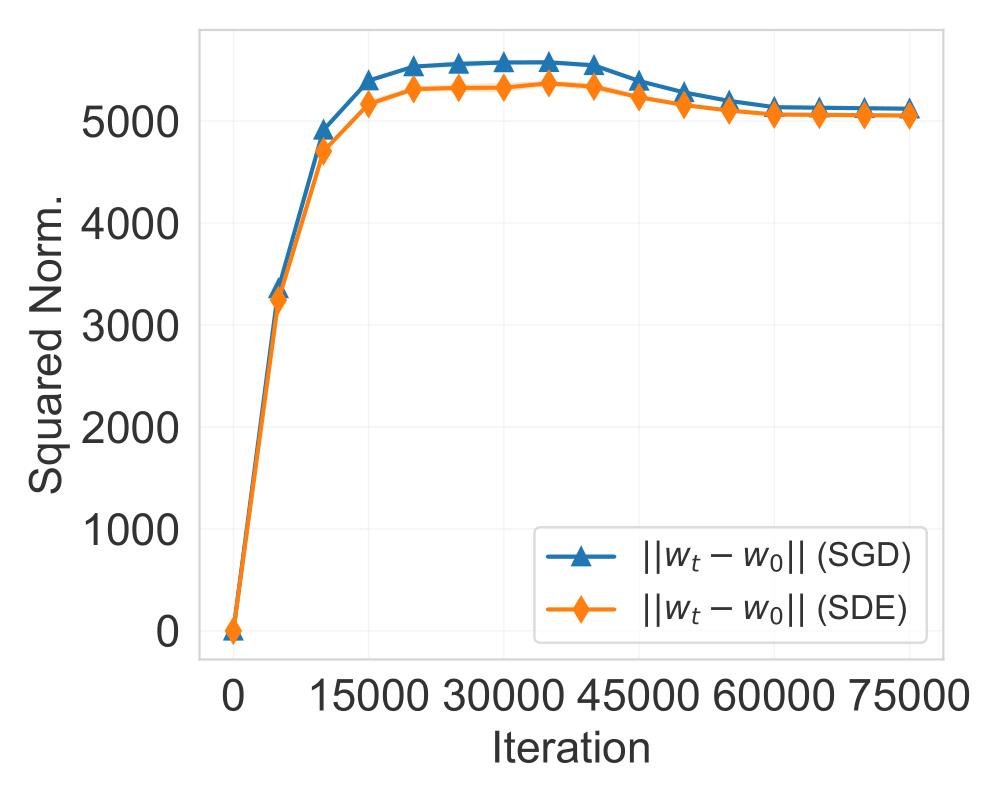}
\caption{VGG on (small) SVHN}
\label{fig:vgg-svhn-weight}
    \end{subfigure}
\begin{subfigure}[b]{0.245\textwidth}
\includegraphics[scale=0.28]{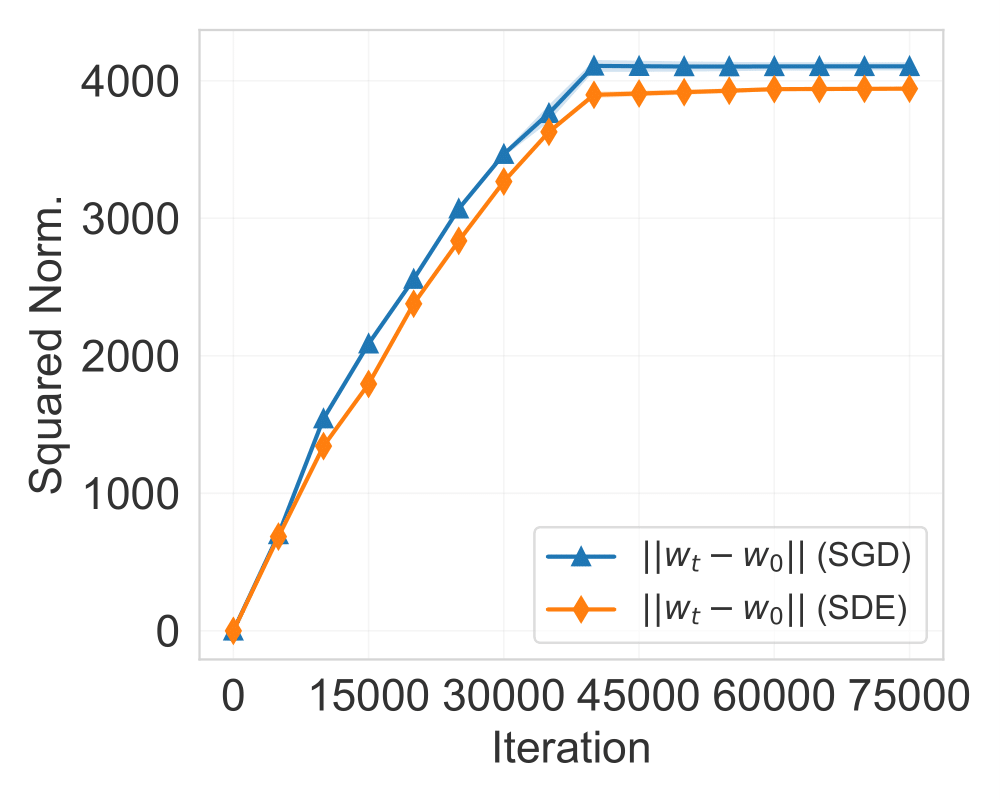}
\caption{VGG on CIFAR10}
\label{fig:vgg-cifar10-weight}
\end{subfigure}
\begin{subfigure}[b]{0.245\textwidth}
\includegraphics[scale=0.28]{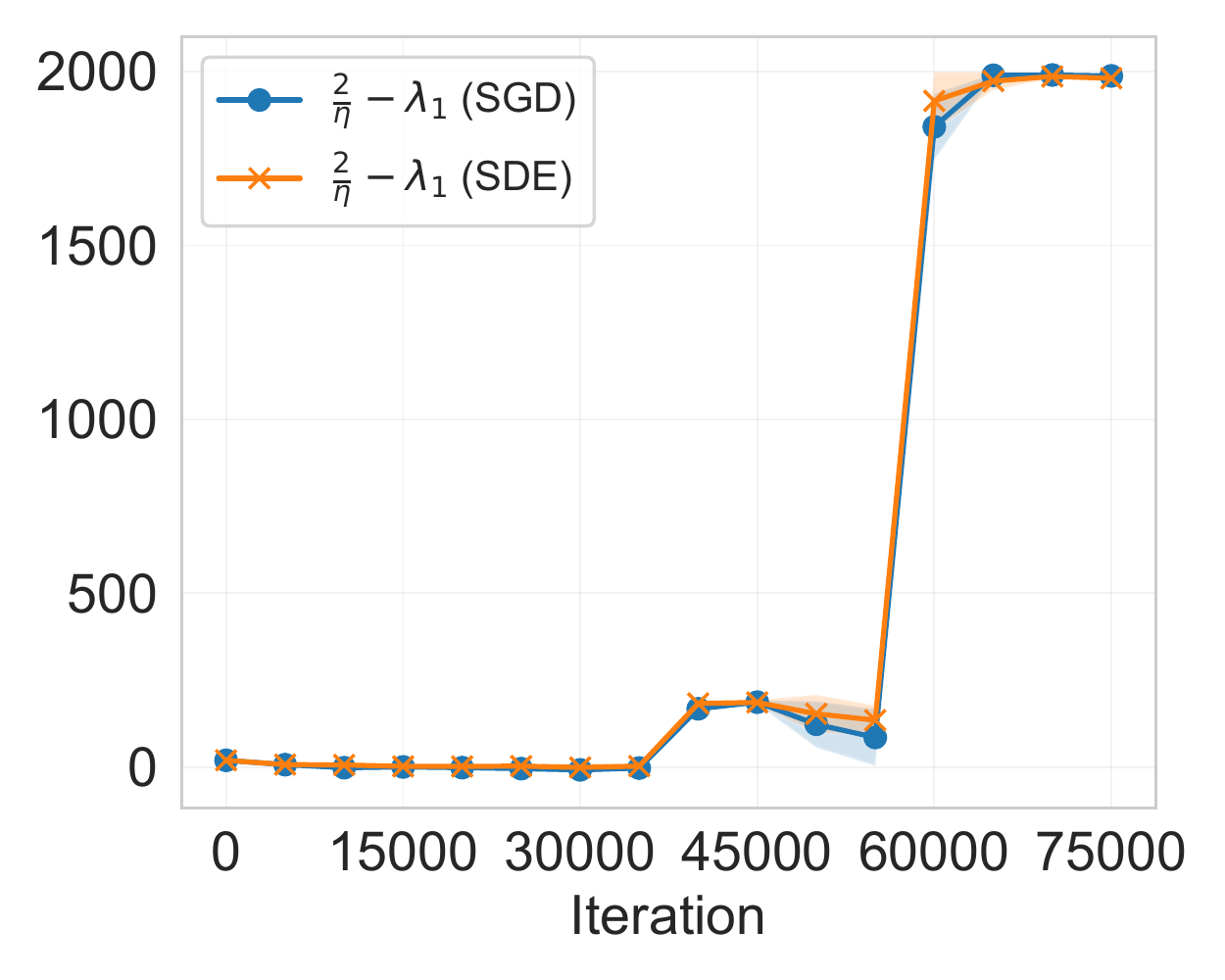}    
\caption{ResNet on CIFAR10}            \label{fig:resnet-cifar10-stable}
    \end{subfigure}
\begin{subfigure}[b]{0.245\textwidth}
\includegraphics[scale=0.28]{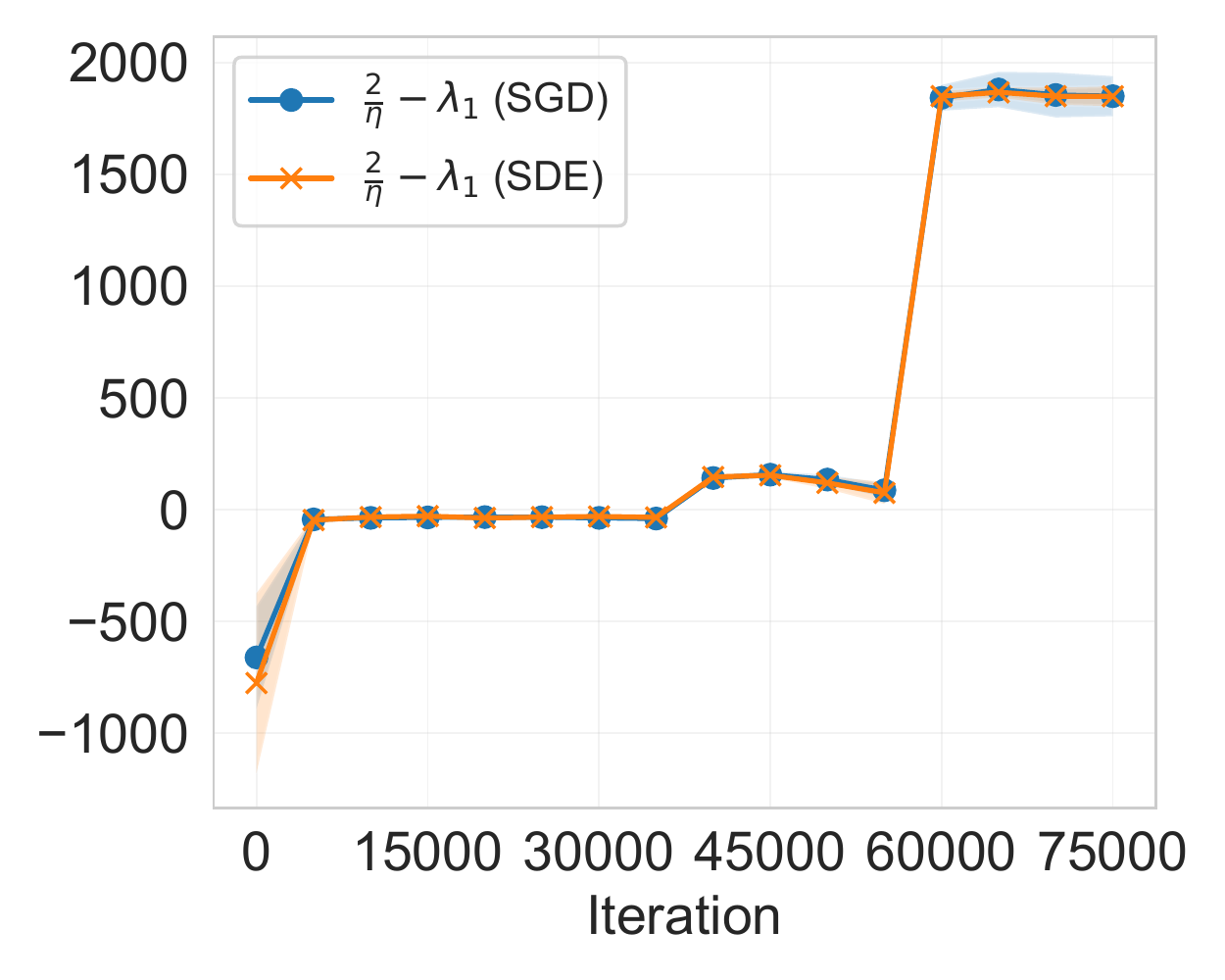}
\caption{ResNet on CIFAR100}
    \label{fig:resnet-cifa100-stable}
\end{subfigure}
 \begin{subfigure}[b]{0.245\textwidth}
\includegraphics[scale=0.28]{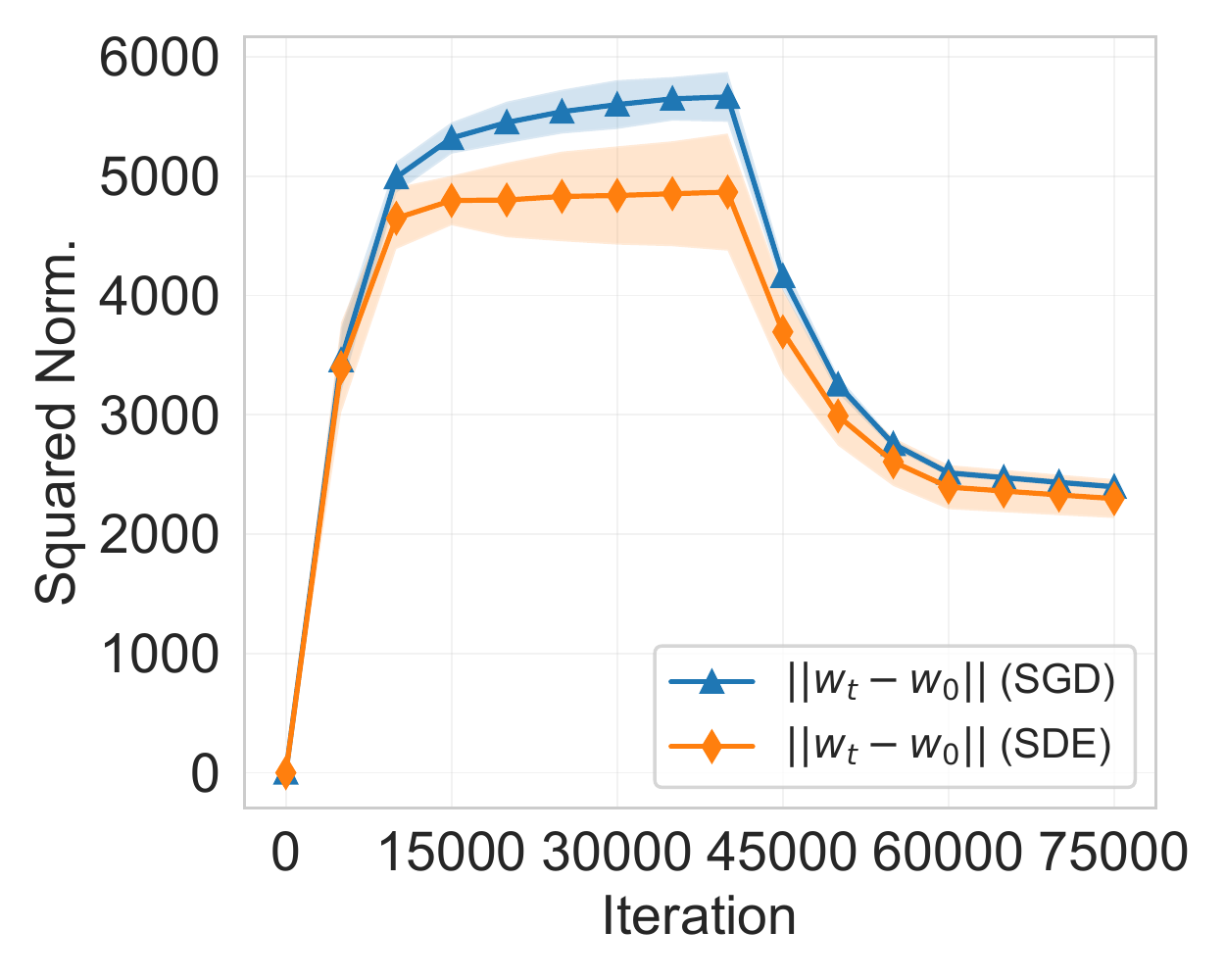}
\caption{ResNet on CIFAR10}
\label{fig:resnet-cifar10-weight}
    \end{subfigure}
\begin{subfigure}[b]{0.245\textwidth}
\includegraphics[scale=0.28]{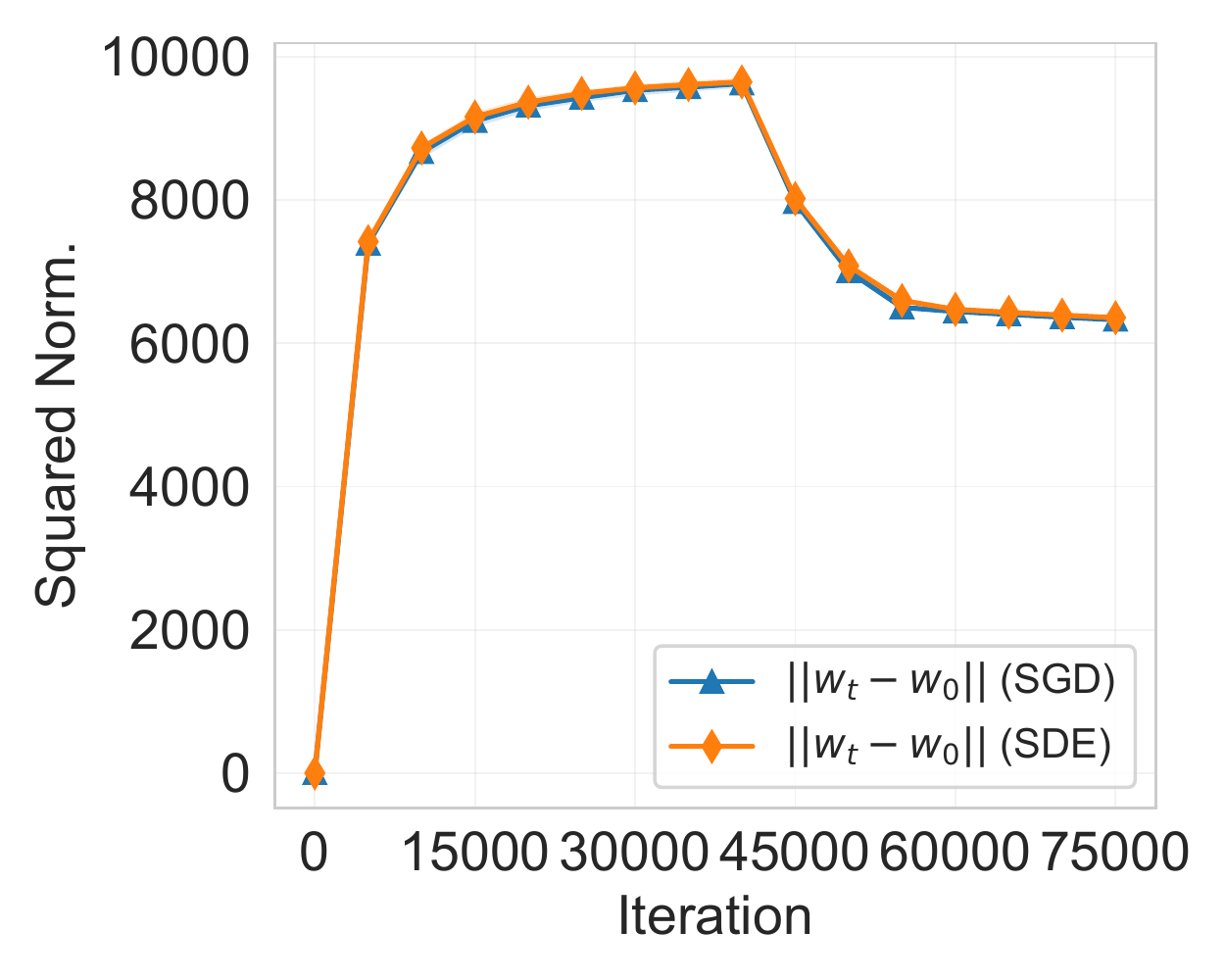}
\caption{ResNet on CIFAR100}
\label{fig:resnet-cifar100-weight}
\end{subfigure}
\caption{(a-b,e-f) The dynamics of $\eta/2-\lambda_1$. Note that learning rate decays by $0.1$ at the $40,000^{\rm th}$ and the $60,000^{\rm th}$ iteration. (c-d,g-h) The distance of current model parameters from its initialization.}\label{fig:Sta-Dynamics}
% \vspace{-4mm}
\end{figure*}
% \begin{figure*}[!ht]
%     \centering
%     \begin{subfigure}[b]{0.245\textwidth}
% \includegraphics[scale=0.28]{figs/stable-plot-cifar10-resnetwobn.pdf}    
% \caption{ResNet on CIFAR10}            \label{fig:resnet-cifar10-stable}
%     \end{subfigure}
% \begin{subfigure}[b]{0.245\textwidth}
% \includegraphics[scale=0.28]{figs/stable-plot-cifa100-resnet.pdf}
% \caption{ResNet on CIFAR100}
%     \label{fig:resnet-cifa100-stable}
% \end{subfigure}
%  \begin{subfigure}[b]{0.245\textwidth}
% \includegraphics[scale=0.28]{figs/weight-plot-cifar10-resnetwobn.pdf}
% \caption{ResNet on CIFAR10}
% \label{fig:resnet-cifar10-weight}
%     \end{subfigure}
% \begin{subfigure}[b]{0.245\textwidth}
% \includegraphics[scale=0.28]{figs/weight-plot-cifa100-resnet.pdf}
% \caption{ResNet on CIFAR100}
% \label{fig:resnet-cifar100-weight}
% \end{subfigure}
% \caption{(a-b) The dynamics of $\eta/2-\lambda_1$. Note that learning rate decays by $0.1$ at the $40,000^{\rm th}$ and the $60,000^{\rm th}$ iteration. (c-d) The distance of current model parameters from its initialization.}\label{fig:Sta-Dynamics-2}
% \end{figure*}

Notably, all the conditions in Lemma~\ref{lem:stationary-real} are only discussed in the context of the terminal state of SGD training. Regarding the condition (ii), as being widely used in the literature \citep{jastrzkebski2017three,zhu2019anisotropic,li2020hessian,xie2020diffusion,xie2021positive,liu2021noise}, Hessian is proportional to the GNC near local minima when the loss is the negative log likelihood, i.e. cross-entropy loss. To see this, when $w_t\to w^*$, we have $
\Sigma_{w^*}=\frac{1}{n}\sum_{i=1}^n\nabla \ell_i\nabla \ell_i^\mathrm{\bf T}-G_tG_t^\mathrm{\bf T}\approx \frac{1}{n}\sum_{i=1}^n\nabla \ell_i\nabla \ell_i^\mathrm{\bf T}=F_{w^*}$, where $F_{w^*}$ is the \textit{Fisher information matrix} (FIM). This approximation is true because gradient noise dominates over gradient mean near local minima. Moreover, FIM is close to the Hessian near local minima with the log-loss \citep[Chapter~8]{pawitan2001all}, namely, $F_{w^*}\approx H_{w^*}$. Let $n\gg b$, we have
$H_{w^*} \approx \Sigma_{w^*} =  b C_{w^*}$. Consequently, when $\Sigma_{T}$ is sufficiently close to $\Sigma_{w^*}$, condition (ii) is satisfied. It's important to note that the debate surrounding $H_{w^*}\approx F_{w^*}$ arises when the loss function deviates from cross-entropy \citep{ziyin2022strength}.
% Similar approximation is widely used in the literature \citep{jastrzkebski2017three,zhu2019anisotropic,li2020hessian,xie2020diffusion,xie2021positive,liu2021noise}. 

For condition (iii), the initial learning rate is typically set at a high value, and this condition may not be satisfied until the learning rate undergoes decay in the later stages of SGD training. This observation is evident in Figure~\ref{fig:vgg-svhn-stable}-\ref{fig:vgg-cifa10-stable}, where the condition becomes easily met at the terminal state following the learning rate decay. Moreover, the interplay between $\frac{2}{\eta}$ and $\lambda_1$ is extensively explored in the context of the {\em edge of stability} \citep{wu2018sgd,cohen2021gradient,arora2022understanding}, which suggests that during the training of GD, $\lambda_1$ approaches $\frac{2}{\eta}$ and hovers just above it in the ``edge of stability'' regime. 
In the context of Theorem~\ref{thm:opt-state-inde-bound},  as indicated by Lemma~\ref{lem:stationary-real}, the diagonal elements of $\Lambda_{w_s^*}$ tend to be close to zero before reaching the ``edge of stability'' regime, the bound presented in Theorem~\ref{thm:opt-state-inde-bound} diverges to infinity.
% In this case, as indicated by Lemma~\ref{lem:stationary-real}, the diagonal elements of $\Lambda_{w_s^*}$ tend to be close to zero before reaching the ``edge of stability''. Consequently, the bound presented in Theorem~\ref{thm:opt-state-inde-bound} diverges to infinity. 
This, as a by-product, provides an alternative explanation to the failure mode of $I(W;S)\to\infty$ in the deterministic algorithm 
% aligns with the fact that $I(W;S)$ may approach infinity for deterministic algorithms, 
(e.g., GD with a fixed initialization).

The following results can be obtained by combining Theorem~\ref{thm:opt-state-inde-bound} and Lemma~\ref{lem:stationary-real}.
% gives a special case that Theorem~\ref{thm:opt-state-inde-bound} is dimension-independent.
\begin{cor}
\label{cor:pacbayes-anisotropic-prior}
Under (i,iii) in Lemma~\ref{lem:stationary-real}, then \[
\mathcal{E}_{\mu}(\mathcal{A})\leq\frac{R}{\sqrt{n\eta}}\sqrt{\ex{}{\tr{\log\pr{\br{H_{w^*}C^{-1}_{T}}\Lambda_{w^*_\mu}}}}}.
\]
% where $\tilde{w}=w_\mu^* =\ex{}{W_S^*}$ and the optimal $\sigma^*=\sqrt{\ex{}{||W_S^*-\tilde{w}||^2/d+\frac{\eta }{2b}}}$.
\end{cor}

\begin{cor}
\label{cor:pacbayes-isotropic-prior}
Under (i-iii) in Lemma~\ref{lem:stationary-real}, then 
\[
\mathcal{E}_{\mu}(\mathcal{A})\leq\sqrt{\frac{dR^2}{n}\log\left(\frac{2b}{\eta d}\mathbb{E}{||W_S^*-w^*_{\mu}||^2}+1\right)}.
\]
% where $\tilde{w}=w_\mu^* =\ex{}{W_S^*}$ and the optimal $\sigma^*=\sqrt{\ex{}{||W_S^*-\tilde{w}||^2/d+\frac{\eta }{2b}}}$.
\end{cor}

By $\log(x+1)\leq x$, the bound in Corollary~\ref{cor:pacbayes-isotropic-prior} is dimension-independent if the weight norm does not grow with $d$. Furthermore, the information-theoretic bound becomes a norm-based bound in Corollary~\ref{cor:pacbayes-isotropic-prior}, which is widely studied in the generalization literature \citep{bartlett2017spectrally,neyshabur2018pac}. In fact, $w^*_\mu$ can be replaced by any data-independent vector, for example, the initialization, $w_0$ (see Corollary~\ref{cor:pacbayes-isotropic-prior-init}). In this case,
% when $\tilde{w}$ is the initialization of the parameters, $w_0$, then
the corresponding bound suggests that generalization performance can be characterized by the ``distance from initialization'', namely, given that SGD achieves satisfactory performance on the training data, a shorter distance from the initialization tends to yield better generalization. \citet{nagarajan2019generalization} also derived a ``distance from initialization'' based generalization bound by using Rademacher complexity, and \citet{Hu2020Simple} use ``distance from initialization'' as a regularizer to improve the generalization performance on noisy data.\looseness=-1

\begin{table*}[bt!]
    \centering
    % \vspace{-0.7em}
    \caption{\small Comparison of the results in this work 
   % The previous $\Sigma$'s for SGD and SGDM are taken from \cite{Mandt2017}, and the other four previous results are due to \cite{gitman2019understanding}. There is no existing continuous-time result about DNM and NGD.
    } %\\
    %$^*$SGD: stochastic gradient descent with learning rate $\lambda$. $^*$SGDM: stochastic gradient descent with momentum hyperparameter $\mu$. $^*$QHM: quasi-hypobolic momentum. $^*$DNM: damped newton's method. $^*$NGD: natural gradient descent.}
    \label{tab:summary}
    \vspace{-0.8EM}
    {\small
    \resizebox{\textwidth}{!}{
    {\begin{tabular}{c|c|c}
    \hline\hline
    &Bounds& Remarks\\
    \hline
    \multicolumn{3}{c}{Trajectory-based Bounds. Pros: less assumptions, can track training dynamics; Cro: Time-Dependent}\\
    \hline%\hline 
     Theorem~\ref{thm:isotropic-prior-bound} &  $\mathcal{O}\pr{\sqrt{\frac{d}{n}\ex{}{\log{\frac{h_1}{d}}-\frac{h_2}{d}}}}$  
     & Isotropic covariance for Gaussian prior     \\
      Corollary~\ref{cor:langevin-dynamic} &  $\mathcal{O}\pr{\sqrt{\frac{d}{n}\sum_{t=1}^T{{\mathbb{E}_{}{\log\left(\frac{\mathbb{E}{\left|\left|G_t-\tilde{g}_t\right|\right|^2}}{d}+1\right)}}}}}$  
     & Bound for langevin dynamic; tighter than \citet[Prop.~3.]{neu2021information}   \\
     Theorem~\ref{thm:anisotropic-prior-bound} &  $\mathcal{O}\pr{\sqrt{\frac{1}{n}\sum_{t=1}^T\ex{}{\tr{\log\frac{\Sigma^\mu_tC_t^{-1}}{b}}}}}$ 
     &  Population GNC for prior;  tighter than  Thm.~\ref{thm:isotropic-prior-bound} \\
     \hline
     \multicolumn{3}{c}{Terminal-State-based Bounds. Pro: time-indepedent; Cro: more assumptions, cannot track training dynamics}\\
    \hline%\hline 
    % NAG \\
     Theorem~\ref{thm:opt-state-inde-bound} &  $\mathcal{O}\pr{\sqrt{\frac{1}{n}\ex{}{\tr{\log\pr{\Lambda^{-1}_{W^*_S}\Lambda_{w^*_\mu}}}}}}$ 
     &  General result; hard to measure in practice \\
     Corollary~\ref{cor:pacbayes-anisotropic-prior} & $\mathcal{O}\pr{\sqrt{\frac{1}{n\eta}\ex{}{\tr{\log\pr{\br{H_{w^*}C^{-1}_{T}}\Lambda_{w^*_\mu}}}}}}$ & Under conditions: $H_{w^*}\Lambda_{w^*}=\Lambda_{w^*}H_{w^*}$ and $H_{w^*}\Sigma_T=\mathrm{I}_d$\\
     Corollary~\ref{cor:pacbayes-isotropic-prior} & $\mathcal{O}\pr{\sqrt{\frac{d}{n}\log\left(\frac{b}{\eta d}\mathbb{E}{||W_S^*-\hat{w}||^2}+1\right)}}$ & $\hat{w}$ is flexible; $\frac{2}{\eta}\gg \lambda_1$; other conditions same  as Cor.~\ref{cor:pacbayes-anisotropic-prior} \\
     Theorem~\ref{thm:pacbayes-data-dependent-prior} & $\mathcal{O}\pr{\mathbb{E}{\sqrt{\frac{M^2b}{\eta}\mathbb{E}_{}{||W^*_{S}-W^*_{S_J}||^2}}}}$ & Bounded loss; $\Lambda(W_{s_j}^*)=\Lambda(W_{s}^*)$; other conditions same as Cor.~\ref{cor:pacbayes-isotropic-prior}  \\
    \hline\hline
    \end{tabular}}}}
    %}
%\caption*{{\footnotesize } }
% \vspace{-1.5em}
\end{table*}

In the sequel, we use the data-dependent prior bound, namely, Lemma~\ref{lem:data-dependent-prior}, to derive new results.
% Note that the LOO prior process in Lemma~\ref{lem:data-dependent-prior} is also approximated by SDE with the same learning rate and batch size, so using $S_J$ to train a model will return the same steady-state covariance $\Lambda(W_{S}^*)$ of the final solution by combining Lemma \ref{lem:posterior-covariance} and Eq~(\ref{eq:approx-hessian-gradient}). 

% In the LOO training, to be consistent with the notations used in the algorithmic stability literature, we let $S^{\setminus i}=S_j$ where $i=[n]\setminus j$ is the instance index that is not selected in $j$. 
% Then, we assume the steady-state covariance of SGD remains constant after removing one training instance (\textcolor{red}{Who did this also?}). 
% Hence $\Lambda(W_{S}^*)\approx\Lambda(W_{S_j}^*)$ for any $j$. 
% Let $P_{W_T|S_J=s_j} = \mathcal{N}(W^*_{s_j},\Lambda(W_{s_j}^*))$ where $W^*_{s_j}$ is the local minimum found by the LOO training.
\begin{thm}
\label{thm:pacbayes-data-dependent-prior}
Let $P_{W_T|S_J=s_j} = \mathcal{N}(W^*_{s_j},\Lambda(W_{s_j}^*))$ where $W^*_{s_j}$ is the local minimum found by the LOO training.
Under the same conditions in Lemma~\ref{lem:data-dependent-prior} and (i-iii) in Lemma~\ref{lem:stationary-real}, assuming $\Lambda(W_{s_j}^*)=\Lambda(W_{s}^*)$ for a given $s$, then
% Assume Eq~(\ref{eq:approx-hessian-gradient}) holds, then
\[
\mathcal{E}_{\mu}(\mathcal{A})\leq\mathbb{E}_{S,J}{\sqrt{\frac{M^2b}{2\eta}\mathbb{E}_{W^*_{S},W^*_{S_J}}^{S,J}{||W^*_{S}-W^*_{S_J}||^2}}}.
\]
\end{thm}
% \begin{rem}
This bound implies a strong connection between generalization and the algorithmic stability exhibited by SGD. Specifically, if the hypothesis output does not change much (in the squared $L_2$ distance sense) upon the removal of a single training instance, the algorithm is likely to generalize effectively. In fact, $\mathbb{E}_{W^*_{S},W^*_{S_J}}^{S,J}{||W^*_{S}-W^*_{S_J}||^2}$ can be regarded as an average version of squared {\em argument stability} \citep{liu2017algorithmic}. Moreover, stability-based bounds often demonstrate a fast decay rate in the convex learning cases \citep{hardt2016train,bassily2020stability}. It is worth noting that if argument stability achieves  the fast rate, e.g., $\sup_{s,j}||w^*_{s}-w^*_{s_j}||\leq\mathcal{O}(1/n)$,  then Theorem~\ref{thm:pacbayes-data-dependent-prior} can also achieve the same rate. In addition, note that the stability-based bound usually contains a Lipshitz constant, while the bound in Theorem~\ref{thm:pacbayes-data-dependent-prior} discards such undesired constant.

% \end{rem}
% \begin{rem}
% While the ratio $b/{\eta}$ explicitly appears in both Theorem \ref{thm:pacbayes-isotropic-prior} and Theorem \ref{thm:pacbayes-data-dependent-prior}, it may be tempting to assert that large learning rate and small batch size will improve the generalization performance. Although this argument is consistent with many empirical observations such as \cite{jastrzkebski2017three}, it's worth mentioning that this ratio also has some implicit impact on the norm of the terminal parameters, so it's sill unclear on the role of this ratio in generalization from Theorem~\ref{thm:pacbayes-data-dependent-prior}.
% \end{rem}

% \begin{figure*}[ht!]
% % \vspace{-5pt}
% \centering
% \input{TrajPlot}
% \caption{SGD training dynamics on MNIST (first column) and CIFAR10 (second column). Some quantities in  are re-scaled, see Appendix for more details.
% %(a)(b) show the bound decaying with the network width. (c)(d) show the bound increasing with the noise level.
% }
% \label{fig:train-dynamic}
% % \vspace{-5pt}
% \end{figure*}

Ideally, to estimate the distance of $||w^*_{s}-w^*_{s_j}||^2$, one can use the influence function \citep{hampel1974influence,cook1982residuals,koh2017understanding}, namely $w^*_{s_j}-w^*_{s}\approx\frac{1}{n}H^{-1}_{W^*_{s}}\nabla\ell(w^*_{s},z_i)$,
where $i$ is the instance index that is not selected in $j$. However, for deep neural network training, the approximation made by influence function is often erroneous \citep{basu2021influence}. While this presents a challenge, it motivates further 
% exploration and 
refinement, seeking to enhance the practical application of Theorem~\ref{thm:pacbayes-data-dependent-prior} in deep learning.

The main generalization bounds obtained in this paper are summarized in Table~\ref{tab:summary}. In the remainder of this paper, we will empirically verify our theoretical results.

\section{Empirical Study}
% \vspace{-3mm}
In this section, we present some empirical results including tracking training dynamics of SGD and SDE, along with the estimation of several obtained generalization bounds.
% the evolution of some key quantities in our bounds. Additionally, we also empirically compare the trajectory-based bound with the bound in \cite{wang2022generalization}, and also estimate the terminal-based bound. 
\begin{figure*}[!ht]
    \centering
    \begin{subfigure}[b]{0.245\textwidth}
\includegraphics[scale=0.28]{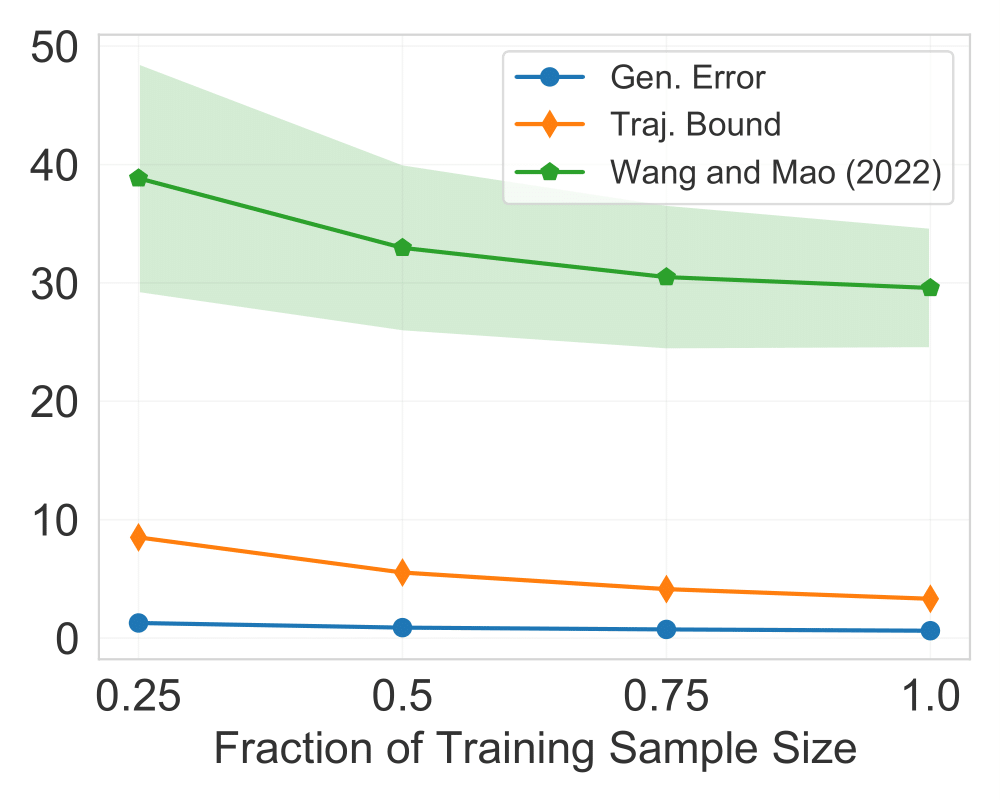}    
\caption{VGG on (small) SVHN}            \label{fig:vgg-svhn-bound}
    \end{subfigure}
\begin{subfigure}[b]{0.245\textwidth}
\includegraphics[scale=0.28]{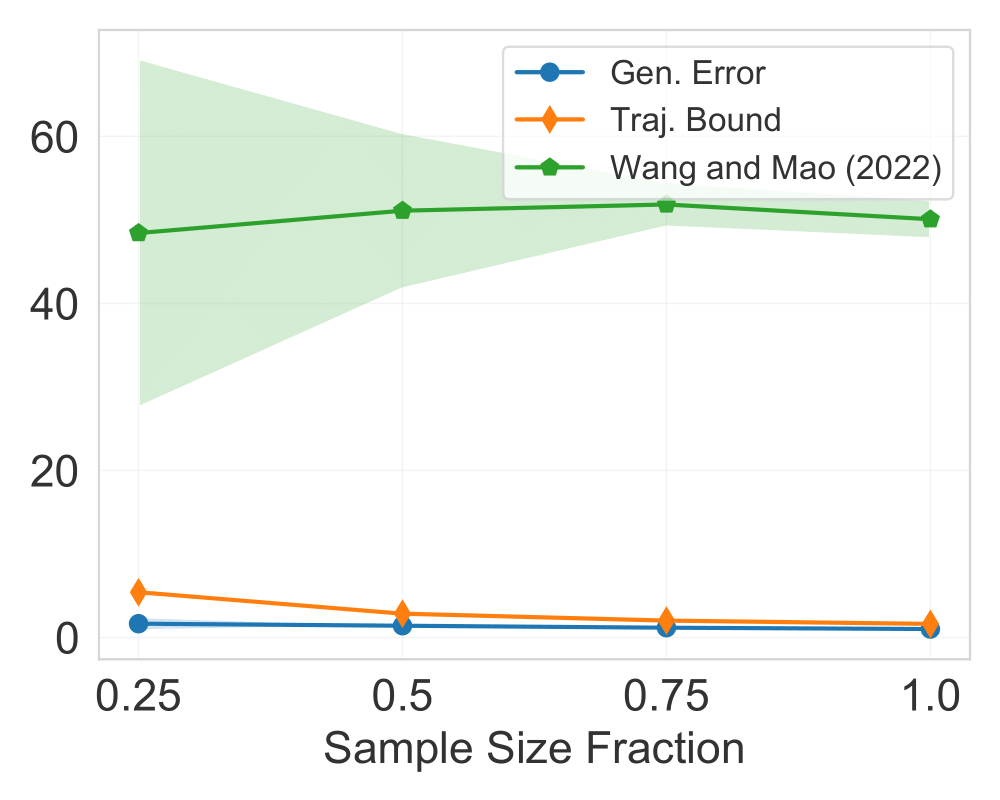}
\caption{VGG on CIFAR10}
    \label{fig:vgg-cifa10-bound}
\end{subfigure}
 \begin{subfigure}[b]{0.245\textwidth}
\includegraphics[scale=0.28]{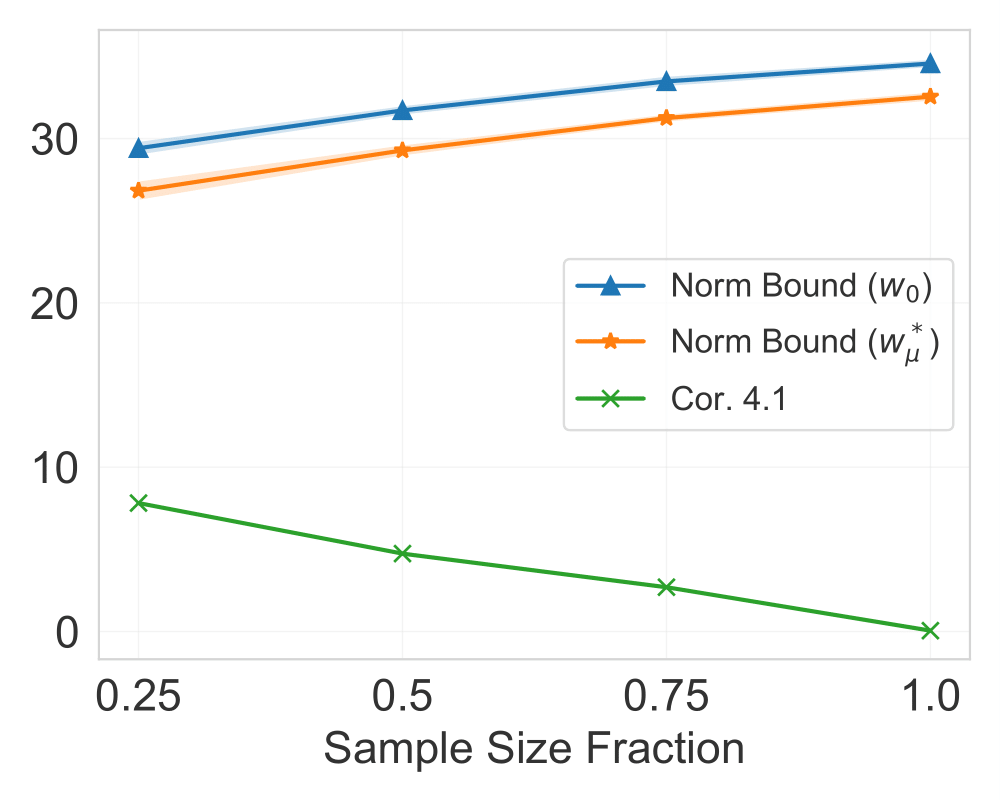}
\caption{VGG on (small) SVHN}
\label{fig:vgg-svhn-TM-bound}
    \end{subfigure}
\begin{subfigure}[b]{0.245\textwidth}
\includegraphics[scale=0.28]{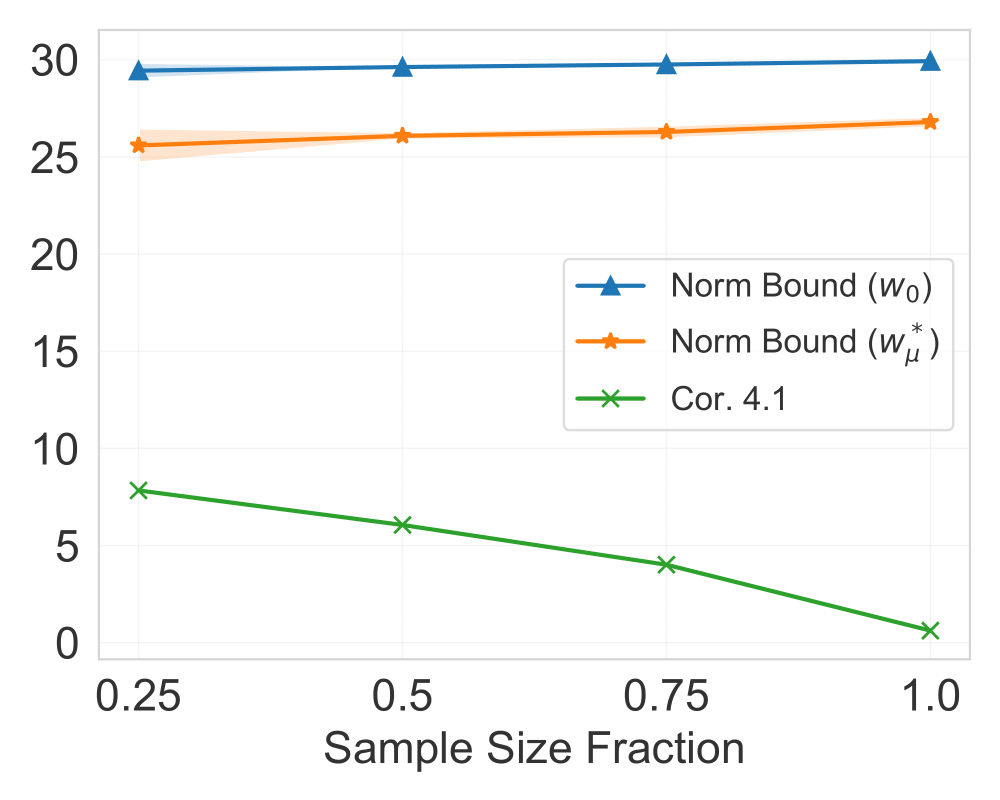}
\caption{VGG on CIFAR10}
\label{fig:vgg-cifar10-TM-bound}
\end{subfigure}
\caption{Estimated trajectory-based bound and terminal-state based bound, with $R$ excluded. Zoomed-in figures of generalization error are given in Figure~\ref{fig:errs} in Appendix.}\label{fig:bounds}
% \vspace{-3mm}
\end{figure*}

\paragraph{Implementation and Hyperparameters}
The implementation in this paper is on PyTorch  \citep{paszke2019pytorch}, and all the experiments are carried out on NVIDIA Tesla V100 GPUs (32 GB). Most experiment settings follow \cite{wu2020noisy}, and the code is also based their implementation, which is available at:  \href{https://github.com/uuujf/MultiNoise}{https://github.com/uuujf/MultiNoise}. For CIFAR 10, the initial learning rates used for VGG-11 and ResNet-18 are $0.01$ and $0.1$, respectively. For SVHN, the initial learning rate is $0.05$. For CIFAR100, the initial learning rate is $0.1$.
The learning rate is then decayed by $0.1$ at iteration $40, 000$ and $60, 000$. If not stated otherwise, the batch size of SGD is $100$. 

% \vspace{-1mm}
\paragraph{SGD and SDE Training Dynamics} We implement the SDE training by following the same algorithm given in \cite[Algorithm~1]{wu2020noisy}. Our experiments involved training a VGG-11 architecture without BatchNormalization on a subset of SVHN (containing $25$k training images) and CIFAR10. Additionally, we trained a ResNet-18 on both CIFAR10 and CIFAR100.  Data augmentation is only used in the experiments related to CIFAR100. We ran each experiment for ten different random seed, maintaining a fixed initialization of the model parameters. 
Further details about the experimental setup can be found in \cite{wu2020noisy}.
The results are depicted in Figure~\ref{fig:Acc-Dynamics}. As mentioned earlier, SDE exhibits a performance dynamics akin to that of SGD, reinforcing the similarities in their training behaviors.

% \vspace{-0.15in}

% \vspace{-1mm}
\paragraph{Evolution of Key Quantities for SGD and SDE} We show $||G_t||^2$ and $\tr{\log\pr{\Sigma_t^{-1}\Sigma_{\mu}}}$ in Figure~\ref{fig:Cov-Dynamics}. Recognizing the computational challenges associated with computing $\tr{\log\pr{\Sigma_t^{-1}\Sigma_{\mu}}}$, we opted to draw estimates based on 100 training and 100 testing samples.  Notably, both SGD and SDE exhibit similar behaviors in these gradient-based metrics. It is noteworthy that despite the absence of the learning rate in the trajectory-based bounds, we observed that modifications to the learning rate at the $40,000^{\rm th}$ and $60,000^{\rm th}$ steps had discernible effects on these gradient-based quantities.  Additionally, in Figure~\ref{fig:Hess-Dynamics}, we examine the trace of the Hessian and its largest eigenvalue during training, leveraging the PyHessian library \citep{yao2020pyhessian}.  Note that we still use only $100$ training data to estimate the Hession for efficiency. Notice that the Hessian-related quantities of SGD and SDE are nearly perfectly matched in the terminal state of training. Furthermore, Figures~\ref{fig:vgg-svhn-weight}-\ref{fig:vgg-cifar10-weight} illustrate the ``distance to initialization'', revealing a consistent trend shared by both SGD and SDE.

% \vspace{-1mm}
\paragraph{Bound Comparison} We vary the size of the training sample and empirically estimate several of our bounds in Figure~\ref{fig:bounds}, with the subgaussian variance proxy $R$ excluded for simplicity. Thus, the estimated values in Figure~\ref{fig:bounds} don't accurately represent the true order of the bounds. Despite the general unbounded nature of cross-entropy loss, common training strategies, such as proper weight initialization, training techniques, and appropriate learning rate selection, ensure that the cross-entropy loss remains bounded in practice. Therefore, it is reasonable to assume subgaussian behavior of the cross-entropy loss under SGD training. In Figure~\ref{fig:vgg-svhn-bound}-\ref{fig:vgg-cifa10-bound}, we compare our Theorem~\ref{thm:anisotropic-prior-bound} with \citet[Theorem~2]{wang2022generalization}. Since both bounds incorporate the same $R$, the results in Figures\ref{fig:vgg-svhn-bound} to \ref{fig:vgg-cifa10-bound} show that our Theorem~\ref{thm:anisotropic-prior-bound} outperforms \citet[Theorem2]{wang2022generalization}. This aligns with expectations (see Appendix~\ref{sec:IT-SGD} for additional discussions), considering that the isotropic Gaussian used in the auxiliary weight process of \citet[Theorem~2]{wang2022generalization} is suboptimal, as demonstrated in Lemma~\ref{lem:compare-iso-noniso}. Moreover, Figures~\ref{fig:vgg-svhn-TM-bound} to \ref{fig:vgg-cifar10-TM-bound} hint that norm-based bounds Corollary~\ref{cor:pacbayes-isotropic-prior} (and Corollary~\ref{cor:pacbayes-isotropic-prior-init}) exhibit growth with $n$, which are also observed in \cite{nagarajan2019uniform}. In contrast, Corollary~\ref{cor:pacbayes-anisotropic-prior} effectively captures the trend of generalization error, emphasizing the significance of the geometric properties of local minima. Additionally, while trajectory-based bounds may appear tighter, terminal-state-based bounds seem to have a faster decay rate.

\section{Other Related Literature}
% \section{Other Related Literature}
% Information-theoretic generalization bounds are typically useful to noisy iterative algorithms.  For example,  \citet{pensia2018generalization} first apply the information-theoretic bound given by \citet{xu2017information} to analyze the generalization property of SGLD. Since the noise used in SGLD is usually an isotropic Gaussian, by utilizing the closed form of KL divergence between two Gaussian distributions, the information-theoretic generalization bound for SGLD is shown to have a tractable form. Their result is then improved by stronger bounds in  \citep{bu2019tightening,negrea2019information,haghifam2020sharpened,wang2022generalization}. 

% MI bounds, however, can not be directly used to analyze the generalization property of vanilla SGD since the MI term in the bound may go to infinity in this case and one could not obtain a tractable form by using following the approach of \cite{pensia2018generalization}. 
% Recently,  \cite{neu2021information} and \cite{wang2022generalization} have studied the generalization of models trained with SGD and obtained new MI bounds, using a technique via constructing an auxiliary perturbed weight process; additional complexity must be delt with in that analysis. Thus there appears significant room for improved understanding of the generalization of SGD.
% \textcolor{red}{
Recently, \cite{simsekli2019tail,nguyen2019first,simsekli2020hausdorff,meng2020dynamic,gurbuzbalaban2021heavy}, and \cite{gurbuzbalaban2021heavy} challenge the traditional assumption that gradient noise is a Gaussian and argue that the noise is heavy-tailed (e.g., L\'{e}vy noise). In contrast, \citet{xie2020diffusion} and \citet{li2021validity} 
claim that non-Gaussian noise is not essential to SGD performance, and SDE with Gaussian gradient noise can well characterize the behavior of SGD. They also 
argue that the empirical evidence shown in \cite{simsekli2019tail} relies on a hidden strong assumption that gradient noise is isotropic and each dimension  has the same distribution. 
For other works on SGD and SDE, see \citep{hoffer2017train,xing2018walk,panigrahi2019non,wu2020noisy,zhu2019anisotropic,li2020hessian,ziyin2022strength}.

In addition, there are some generalization bounds using fractal dimensions \citep{simsekli2020hausdorff,camuto2021fractal,dupuis2023generalization}, which are also trajectory-based generalization bounds. Notably, \cite{dupuis2023generalization} improves previous works by removing the Lipschitz continuity assumption, yet direct comparison of our results with theirs remains challenging. Specifically, one notable difference is that in both Section~\ref{sec:itb-sde} and Section~\ref{sec:pac-bayes} of our work, we provide in-expectation generalization bounds, while they present high-probability generalization guarantees, which require additional developments for comparison. Moreover, some key components in their bounds are not directly comparable to our gradient noise covariance or Hessian-based quantities, such as their upper and lower box-counting dimensions. On one hand, we believe our results have several advantages. For instance, while the boundedness of loss is essential in their work, we can relax it to a sub-Gaussian condition in ours. Additionally, our bound is easier to estimate for more complex models. On the other hand, their utilization of intrinsic dimension in the analysis is inspiring and may be possible to incorporate into our analysis for obtaining better results. 

\section{Limitations and Future works}
\label{sec:concluds}
% In this paper, we invoke the SDE approximation of SGD so that information-theoretic generalization bounds are directly applicable to SGD with two opportunities. First, dynamics characterized by SDE enable us to obtain trajectory-based bounds by the step-wise analysis of mutual information. 
% These results mainly suggest that the trace of gradient noise covariance is significant for studying the generalization ability of SGD. 
% In addition, with some mild assumptions, we also 
% % apply the PAC-Bayes-like approach to 
% obtain some new bounds based on the terminal state of SGD. 

 While our current work exhibits certain limitations, such as the requirement of positive definiteness for $C_t$ in our trajectory-based bounds, it's worth noting that recent studies \citep{frankle2018the,li2018measuring,gur2018gradient,larsen2022how} indicate that many parameters in deep neural networks might be dispensable without affecting generalization. This implies that GD/SGD could potentially occur in a subspace of $\mathbb{R}^d$ termed the ``intrinsic dimension'' $d_{\mathrm{int}}$. Defining $C_t$ within this invertible subspace, utilizing $d_{\mathrm{int}}$, could potentially overcome our current limitations. Theoretical characterization of intrinsic dimension, however, remains an open problem, and further exploration in this direction is poised to significantly improve our work. In addition, there are also some other promising directions for further improving this work, for example, 
% via developing new theoretic tools to analyzing non-Gaussian type gradient noise (e.g., heavy tailed noise) and 
via using other posterior and prior covariance instead of steady-state covariance (e.g., we also give one in Theorem~\ref{thm:IF-pacbayes-FIM} in Appendix), and via extending the similar analytical approach used in this work to other optimizer (e.g., Adam, Adagrad, etc.). 

\begin{acknowledgements} % will be removed in pdf for initial submission,
						 % (without ‘accepted’ option in \documentclass)
                         % so you can already fill it to test with the
                         % ‘accepted’ class option
    This work is supported partly by an NSERC Discovery
grant. Ziqiao Wang is also supported in part by the NSERC CREATE program through the Interdisciplinary Math and Artificial Intelligence (INTER-MATH-AI) project. The authors would like to thank all the anonymous reviewers and the ACs for their careful reading and valuable suggestions.
\end{acknowledgements}

% References
\bibliography{ref}

\newpage

\onecolumn

\appendix
\title{Supplementary Material}
\maketitle
% \begin{appendices}
% \addcontentsline{toc}{section}{Appendix} % Add the appendix text to the document TOC
% \part{Appendix} % Start the appendix part
% \parttoc

% \addcontentsline{toc}{section}{Appendix} % Add the appendix text to the document TOC
% \part{Appendix} % Start the appendix part
% \parttoc

% \section{Summaries for the Bounds}

% \section{Notation}
% % Unless otherwise noted, a random variable will be denoted by a capitalized letter, and  its realization by the corresponding lower-case letter. 
% The distribution of a random variable $X$ is denoted by $P_X$ (or $Q_X$), and the conditional distribution of $X$ given $Y$ is denoted by $P_{X|Y}$. When conditioning on a specific realization $y$, we use the shorthand $P_{X|Y=y}$ or simply $P_{X|y}$.
% Denote by $\mathbb{E}_{X}$ expectation over $X \sim P_X$, and by $\mathbb{E}_{X|Y=y}$ (or $\mathbb{E}^y_{X}$) expectation over $X \sim P_{X|Y=y}$. We may omit the subscript of the expectation when there is no ambiguity.
% % The entropy of a random variable $X$ is denoted by $H(X)$, and 
% The KL divergence of probability distribution $Q$ with respect to $P$ is denoted by $\mathrm{D_{KL}}(Q||P)$.
% The mutual information (MI) between random variables $X$ and $Y$ is denoted by $I(X;Y)$, and the conditional mutual information between $X$ and $Y$ given $Z$ is denoted by $I(X;Y|Z)$. In addition, for a matrix $A\in\mathbb{R}^{d\times d}$, we let $\tr{A}$ denote the trace of $A$ and we use $\tr{\log{A}}$ to indicate $\sum_{k=1}^d\log{A_{k,k}}$.

\section{Additional Background}
\subsection{Information-Theoretic Bounds for SGD}
\label{sec:IT-SGD}

% \textcolor{red}{
Recently, \citep{neu2021information,wang2022generalization} apply information-theoretic analysis to the generalization of models trained with SGD by invoking an auxiliary weight process (AWP). We now denote this auxiliary weight process by  $\mathcal{A}_{AWP}$. Let $\mathcal{A}_{SGD}$ be the original algorithm of SGD, \citep{neu2021information,wang2022generalization} obtain generalization bounds by the following construction,
\begin{align}
\mathcal{E}_{\mu}\pr{\mathcal{A}_{SGD}}=&\mathcal{E}_{\mu}\pr{\mathcal{A}_{SGD}}+\mathcal{E}_{\mu}\pr{\mathcal{A}_{AWP}}-\mathcal{E}_{\mu}\pr{\mathcal{A}_{AWP}}\notag\\
    \leq& \underbrace{\mathcal{O}\pr{\sqrt{\frac{I(W_{\rm AWP};S)}{n}}}}_{\text{Lemma~\ref{lem:xu's-bound}}}+\underbrace{\abs{\mathcal{E}_{\mu}\pr{\mathcal{A}_{SGD}}-\mathcal{E}_{\mu}\pr{\mathcal{A}_{AWP}}}}_{\text{residual term}}, \label{ineq:sgd-ax}
\end{align}
where $W_{\rm AWP}$ is the output hypothesis by $\mathcal{A}_{AWP}$.

Notably, it remains uncertain whether the residual term is sufficiently small for the information-theoretic bounds of $\mathcal{A}_{AWP}$ to yield meaningful insights into SGD. Although there exists an optimal $\mathcal{A}_{AWP}$ that tightens the bound in Eq.~(\ref{ineq:sgd-ax}), finding such an optimal $\mathcal{A}_{AWP}$ beyond the isotropic Gaussian noise covariance case is challenging. 
It's worth noting that \cite{wang2022generalization} provides an optimal bound for the time-invariant isotropic Gaussian noise case. Nevertheless, our empirical results, as illustrated in Figure~\ref{fig:bounds}, demonstrate that the bounds presented in this paper outperform the isotropic Gaussian noise case.

In this paper, we do not attempt to find an optimal $\mathcal{A}_{AWP}$, but instead,  we invoke the SDE approximation (i.e. Eq.~(\ref{eq:sgd-update-gaussian})), denoted as $\mathcal{A}_{SDE}$. Formally,

\begin{align}
\mathcal{E}_{\mu}\pr{\mathcal{A}_{SGD}}=&\mathcal{E}_{\mu}\pr{\mathcal{A}_{SGD}}+\mathcal{E}_{\mu}\pr{\mathcal{A}_{SDE}}-\mathcal{E}_{\mu}\pr{\mathcal{A}_{SDE}}\notag\\
    \leq& \underbrace{\mathcal{O}\pr{\sqrt{\frac{I(W_{\rm SDE};S)}{n}}}}_{\text{Lemma~\ref{lem:xu's-bound}}}+\underbrace{\abs{\mathcal{E}_{\mu}\pr{\mathcal{A}_{SGD}}-\mathcal{E}_{\mu}\pr{\mathcal{A}_{SDE}}}}_{\text{residual term}}, \label{ineq:sde-res}
\end{align}
where $W_{\rm SDE}$ is the output hypothesis by $\mathcal{A}_{SDE}$.

Empirical evidence from \citep{wu2020noisy,li2021validity} and our Figure~\ref{fig:Acc-Dynamics} suggests that the residual term in Eq.~(\ref{ineq:sde-res}) is small. This observation motivates our investigation into the generalization of SGD using the information-theoretic bounds of SDE directly.

\subsection{Theoretical Validation of SDE}

To theoretically  assess the validation of SDE  in approximating SGD, two essential technical definitions are necessary.
\begin{defn}[Test Functions]
\label{defn:test-func}
    Let $\mathcal{F}$ denote the set of continuous functions ($\mathbb{R}^d\to\mathbb{R}$) with polynomial growth, i.e. if $\forall$ $f \in \mathcal{F}$, there exists constants $K, \kappa > 0$ s.t. $|f(x)|<K(1+|x|^{\kappa})$ for all $x\in\mathbb{R}$.
\end{defn}

\begin{defn}[Order$-\alpha$ weak approximation]
\label{defn:weak-approx}
     Let $\eta\in(0,1)$, $T>0$ and $N=\lfloor T/\eta \rfloor$. Let $\mathcal{F}$ be the set of test Functions.
     We say that the SDE in Eq.~(\ref{eq:ito-sde}) is an order $\alpha$ weak approximation of the SGD in Eq.~(\ref{eq:sgd-update}) if for every $f\in\mathcal{F}$, there exists $C>0$, independent of $\eta$, s.t. for all $k=0, 1, \dots, N$, 
     \[
     \abs{\ex{}{f(\omega_{k\eta})}-\ex{}{f(W_{k})}}\leq C\eta^{\alpha}.
     \]
\end{defn}
Below is a classical result.

\begin{lem}[{\citet[Theorem~1]{li2017stochastic}}]
    Assume $\nabla\ell$ is Lipschitz continuous, has at most linear asymptotic growth and has sufficiently high derivatives belonging to $\mathcal{F}$, then SDE in Eq.~(\ref{eq:ito-sde}) is an order $1$ weak approximation of the SGD in Eq.~(\ref{eq:sgd-update}). Or equivalently, for every
$f \in \mathcal{F}$, there exists $C>0$, independent of $\eta$, s.t. $\max_{k=0, 1, \dots, N}\abs{\ex{}{f(\omega_{k\eta})}-\ex{}{f(W_{k})}}\leq C\eta$.
\label{lem:sde-weak}
\end{lem}

This theorem suggests that SGD and SDE closely track each other when they result in similar distributions of outcomes, such as the returned hypothesis $W$. In addition, the closeness of distributions is formulated through expectations of suitable classes of test functions, as defined in Definition~\ref{defn:test-func}. As mentioned in \cite{li2021validity}, of particular interest for machine learning are test functions like generalization error $\mathcal{E}_\mu$, which may not adhere to formal conditions such as differentiability assumed in classical theory but are still valuable for experimental use. Other typical choices of test functions includes weight norm, gradient norm, and the trace of noise covariance.

\subsection{Gaussian Distribution around Local Minimum}
\label{sec:Gaussian-local}
A multi-dimensional Ornstein-Uhlenbeck process is defined as
\begin{align}
    dx_t=-\mathbf{H}x_tdt+\mathbf{B}d\theta_t,
    \label{eq: OU-sde}
\end{align}
where $x_t\in\mathbb{R}^d$, $\mathbf{H}$, $\mathbf{B}$ are $d\times d$ matrices and $\theta_t$ is an $d$-dimensional Wiener process.

% The solution is

Denote the density function of $x_t$ as
$P(x, t)$, then the corresponding Fokker-Planck equation describes the evolution of $P(x, t)$:
\[
\frac{\partial P(x,t)}{\partial t}=\sum_{i=1}^d\sum_{j=1}^d\frac{\partial}{\partial x_i}\pr{P(x,t)\sum_{j=1}^d\mathbf{H}_{i,j}x_j}+\sum_{i=1}^d\sum_{j=1}^d\mathbf{D}_{i,j}\frac{\partial^2 P(x,t)}{\partial x_i \partial x_j},
\]
where $\mathbf{D}={\mathbf{B}\mathbf{B}^{\bf T}}/{2}$.

Moreover, if $\mathbf{H}$ is positive define, then a stationary solution of $P$ is given by \citep{freidlin2012randomper}:
\begin{align}
\label{eq:OU-stationary}
    P(x)=\frac{1}{\sqrt{\pr{2\pi}^{d}\mathrm{det}\pr{\Sigma}}}\exp\pr{-\frac{1}{2}x^{\bf T}\Sigma^{-1} x},
\end{align}
where $\Sigma=\ex{}{xx^{\bf T}}$ is the covariance matrix of $x$.

When $w$ is close to any local minimum $w^*$, we can use a second-order Taylor expansion to approximate the value of the loss at $w$, 
\begin{eqnarray}
  L_s(w) \approx L_s(w^*) + \frac{1}{2}(w-w^*)^\mathrm{\bf T} H_{w^*}(w-w^*).
  % \label{eq:second-order-taylor}
\end{eqnarray}
In this case, when $w_t\to w^*$, we have $G_t=\nabla L_s(w_t)=H_{w^*}\pr{w_t-w^*}$.
Recall Eq.~(\ref{eq:sgd-update-2}), then
\begin{align*}
    w_{t} =& w_{t-1} - \eta G_t + \eta V_t\notag
    =w_{t-1} - \eta H_{w^*}\pr{w_{t-1}-w^*} + \eta V_t.
\end{align*}
Let $W'_t\triangleq W_t-w^*$ and recall Eq~(\ref{eq: OU-sde}), we thus have the Ornstein-Uhlenbeck process for $x_t=W'_t$ as
\begin{align}
    dW'_t=-\eta H_{w^*}W'_tdt+\eta\sqrt{C_t}d\theta_t.
\end{align}
By Eq.~(\ref{eq:OU-stationary}), we have
\[
P(W')\propto \exp\pr{-\frac{1}{2}W'^{\bf T}\Lambda_{w^*}^{-1} W'}.
\]
Consequently, the stationary distribution of $W$ for a given $w^*$ is $\mathcal{N}(w^*,\Lambda_{w^*})$.

For discrete case, we have
\begin{align*}
    w'_{t} =&\pr{\mathrm{I}_d - \eta H_{w^*}}w'_{t-1} + \eta V_t\\
    =&\pr{\mathrm{I}_d - \eta H_{w^*}}^2w'_{t-2} +  \eta\pr{\pr{\mathrm{I}_d - \eta H_{w^*}}V_{t-1}+V_t}\\
    \vdots& \\
    =& \bar{H}^tw'_0+\eta\sum_{i=0}^t\bar{H}^iV_{t-i},
\end{align*}
where $\bar{H}=\mathrm{I}_d - \eta H_{w^*}$.
Notably, when $t$ is sufficiently large, then the first term is negligible, especially with a small learning rate, we have $w_t'=w_t-w^*=\eta\sum_{i=0}^t\bar{H}^iV_{t-i}$. When $C_t$ does not change in the long time limit, then  $W_t'$ is the weighted sum of independent Gaussian random variables, which follows a Gaussian distribution, namely  $w_t\sim\mathcal{N}(w^*,\Lambda_{w^*})$. We refer readers to \cite[Theorem~1-2.]{liu2021noise} for a relaxed analysis in the discrete case.

\section{Some Useful Facts}

We present the variational representation of mutual information below.
\begin{lem}[{\citet[Corollary~3.1.]{polyanskiy2019lecture}}]
\label{lem:mi-center-gravity}
For two random variables $X$ and $Y$, we have
\[
I(X;Y) = \inf_{P} \ex{X}{\mathrm{D_{KL}}(Q_{Y|X}||P)},
\]
where the infimum is achieved at $P=Q_Y$.
\end{lem}

The following lemma is inspired by the classic Log-sum inequality in \citet[Theorem~2.7.1]{cover2012elements}.
\begin{lem}
\label{lem:log-sum-ineq}
For non-negative numbers $\{a_i\}_{i=1}^n$ and $\{b_i\}_{i=1}^n$,
\[
\sum_{i=1}^n b_i\log\frac{a_i}{b_i}\leq \left(\sum_{i=1}^n b_i \right)\log\frac{\sum_{i=1}^n a_i}{\sum_{i=1}^n b_i},
\]
with equality if and only if $\frac{a_i}{b_i}=const$.
\end{lem}
\begin{proof}
Since $\log$ is a concave function, according to Jensen's inequality, we have
\[
\sum_{i=1}^n\alpha_i \log(x_i)\leq \log(\sum_{i=1}^n\alpha_i x_i),
\]
where $\sum_{i=1}^n \alpha_i=1$.

Let $\alpha_i=\frac{b_i}{\sum_{i=1}^n b_i}$ and $x_i=\frac{a_i}{b_i}$, and plugging them into the inequality above, we have
\[
\sum_{i=1}^n\frac{b_i}{\sum_{i=1}^n b_i} \log(\frac{a_i}{b_i})\leq \log\left(\sum_{i=1}^n\frac{b_i}{\sum_{i=1}^n b_i} \frac{a_i}{b_i}\right)=\log\left(\frac{\sum_{i=1}^n a_i}{\sum_{i=1}^n b_i} \right),
\]
which implies
\[
{\sum_{i=1}^n b_i} \log(\frac{a_i}{b_i})\leq\left(\sum_{i=1}^n b_i\right)\log\left(\frac{\sum_{i=1}^n a_i}{\sum_{i=1}^n b_i} \right).
\]
This completes the proof.
\end{proof}

Below is the KL divergence between two Gaussian distributions $p=\mathcal{N}(\mu_p,\Sigma_p)$ and $q=\mathcal{N}(\mu_q,\Sigma_q)$, where $\mu_p,\mu_q\in\mathbb{R}^d$ and $\Sigma_p,\Sigma_q\in\mathbb{R}^{d\times d}$.
\begin{align}
    \mathrm{D_{KL}}(p||q) = \frac{1}{2}\left[\log\frac{\det(\Sigma_q)}{\det(\Sigma_p)} - d + ({\mu_p}-{\mu_q})^{\bf T}\Sigma_q^{-1}({\mu_p}-{\mu_q}) + tr\left\{\Sigma_q^{-1}\Sigma_p\right\}\right].
    \label{eq:kl-two-gau}
\end{align}

\section{Omitted Proofs and Additional Results in Section %``Generalization Bounds Via Full Trajectories''
\ref{sec:itb-sde}
}

% \subsection{Proof of a variant of Log sum inequality}

\subsection{Proof of Lemma~\ref{lem:mi-unroll}}

% We first unroll the terminal parameters' mutual information $I(W_T;S)$ to the full trajectories' mutual information via the lemma below.
% \begin{lem}
% \label{lem:mi-unroll}
% $I(W_T;S)\leq\sum_{t=1}^T I(- G_t + C_t^{1/2}N_t;S|W_{t-1} ).
% $
% \end{lem}

% This lemma can be proved by recurrently applying the data processing inequality (DPI) and chain rule of the mutual information \citep{polyanskiy2019lecture}. 

\begin{proof}
Recall the SDE approximation of SGD, i.e., Eq (\ref{eq:sgd-update-gaussian}), we then have,
\begin{eqnarray}
  I(W_T;S)&=&I(W_{T-1} - \eta G_T + \eta C_T^{1/2}N_T;S)\notag\\
  &\leq&I(W_{T-1},- \eta G_T + \eta C_T^{1/2}N_T;S )\label{ineq:dpi}\\
  &=&I(W_{T-1};S)+I(- \eta G_T + \eta C_T^{1/2}N_T;S|W_{T-1} )\label{eq:chain-rule}\\
  &\vdots&\notag\\
  &\leq&\sum_{t=1}^T I(- \eta G_t + \eta C_t^{1/2}N_t;S|W_{t-1} )\notag\\
  &=&\sum_{t=1}^T I(- G_t + C_t^{1/2}N_t;S|W_{t-1} ).\notag
\end{eqnarray}
where Eq. (\ref{ineq:dpi}) is by the data processing inequality (e.g., $Z - (X,Y) - (X+Y)$ form a markov chain then $I(X+Y,Z)\leq I(X,Y;Z)$), Eq. (\ref{eq:chain-rule}) is by the chain rule of the mutual information, and learning rate $\eta$ is dropped since mutual information is scale-invariant.
\end{proof}

\subsection{Proof of Lemma \ref{lem:cmi-golden formula}}
\begin{proof}
% Given $P_{\tilde{N}_t}=\mathcal{N}(\tilde{g},\sigma^2 \mathrm{I}_d)$,
For any $t\in[T]$, similar to the proof of Lemma \ref{lem:mi-center-gravity} in \cite{polyanskiy2019lecture}:
\begin{eqnarray}
  &&I(- G_t + C_t^{1/2}N_t;S|W_{t-1}=w_{t-1}) \notag\\
  &=& \cex{S}{w_{t-1}}{\mathrm{D_{KL}}(Q_{\widehat{G}_t|w_{t-1},S}||Q_{\widehat{G}_t|w_{t-1}})}\notag\\
  &=&\cex{S}{w_{t-1}}{\mathrm{D_{KL}}(Q_{\widehat{G}_t|w_{t-1},S}||P_{\widehat{G}_t|w_{t-1}})-\mathrm{D_{KL}}(Q_{\widehat{G}_t|w_{t-1}}||P_{\widehat{G}_t|w_{t-1}})}\notag\\
  &\leq&\cex{S}{w_{t-1}}{\mathrm{D_{KL}}(Q_{\widehat{G}_t|w_{t-1},S}||P_{\widehat{G}_t|w_{t-1}})},\label{ineq:kl-nonnegative}
%   \\
%   &\leq&\inf_{\tilde{g}_t,\sigma_t}\ex{S}{\mathrm{D_{KL}}(P_{- G_t + C_t^{1/2}N_t|W_{t-1}=w_{t-1},S=s}||P_{-\tilde{g}_t+\sigma_t N_t|W_{t-1}=w_{t-1}})},\label{ineq:kl-nonnegative}
\end{eqnarray}
where Eq. (\ref{ineq:kl-nonnegative}) is due to the fact that KL divergence is non-negative, and the equality holds when ${P_{\widehat{G}_t|w_{t-1}}}=Q_{\widehat{G}_t|w_{t-1}}$ for $W_{t-1}=w_{t-1}$.

Thus, we conclude that
\[
I(\widehat{G}_t;S|W_{t-1}=w_{t-1})=\inf_{P_{\widehat{G}_t|w_{t-1}}}\cex{S}{w_{t-1}}{\mathrm{D_{KL}}(Q_{\widehat{G}_t|w_{t-1},S}||P_{\widehat{G}_t|w_{t-1}})}.
\]

Taking expectation over $W_{t-1}$ for both side above, we have
\[
I(\widehat{G}_t;S|W_{t-1})=\ex{W_{t-1}}{\inf_{P_{\widehat{G}_t|W_{t-1}}}\cex{S}{W_{t-1}}{\mathrm{D_{KL}}(Q_{\widehat{G}_t|W_{t-1},S}||P_{\widehat{G}_t|W_{t-1}})}}.
\]
% \[
% I(- G_t + C_t^{1/2}N_t;S|W_{t-1}) \leq \ex{W_{t-1}}{\inf_{\tilde{g}_t,\sigma_t} \ex{S}{\mathrm{D_{KL}}(P_{- G_t + C_t^{1/2}N_t|W_{t-1},S}||P_{-\tilde{g}_t+\sigma_t N_t|W_{t-1}})}}.
% \]
This completes the proof.
\end{proof}

\subsection{Proof of Theorem \ref{thm:isotropic-prior-bound}}

\begin{proof}

We first prove Eq.~(\ref{ineq:iso-gen-bound}).
Recall Lemma \ref{lem:cmi-golden formula} and assume $C_t$ is a positive-definite matrix,  for any $t\in[T]$, we have
\begin{align}
    &I(- G_t + C_t^{1/2}N_t;S|W_{t-1}=w_{t-1}) \notag\\
  \leq&\inf_{\tilde{g}_t,\sigma_t}\cex{S}{w_{t-1}}{\mathrm{D_{KL}}(Q_{- G_t + C_t^{1/2}N_t|w_{t-1},S}||P_{-\tilde{g}_t+\sigma_t N_t|w_{t-1}})}\notag\\
  =&\inf_{\tilde{g}_t,\sigma_t}\cex{S}{w_{t-1}}{\frac{1}{2}\left[\log\frac{\det(\sigma_t^2\mathrm{I}_d)}{\det(C_t)} - d + \frac{1}{\sigma_t^2}((G_t-\tilde{g}_t)^{\bf T} \mathrm{I}_d^{-1}(G_t-\tilde{g}_t)) + \frac{1}{\sigma_t^2}tr\left\{\mathrm{I}_d^{-1} C_t\right\}\right]}\label{eq:kl-gaussian}\\
  =&\frac{1}{2}\inf_{\tilde{g}_t,\sigma_t}\cex{S}{w_{t-1}}{\frac{1}{\sigma_t^2} \left(||G_t-\tilde{g}_t||^2 +tr\left\{C_t\right\}\right)+d \log{\sigma_t^2}-d -tr\left\{\log{C_t}\right\}},\label{eq:trace-form}
\end{align}
where Eq. (\ref{eq:kl-gaussian}) is by Eq.~(\ref{eq:kl-two-gau}),
% \[
% D_{KL}(p||q) = \frac{1}{2}\left[\log\frac{\det(\Sigma_q)}{\det(\Sigma_p)} - k + ({\mu_p}-{\mu_q})^T\Sigma_q^{-1}({\mu_p}-{\mu_q}) + tr\left\{\Sigma_q^{-1}\Sigma_p\right\}\right],
% \]
 Eq. (\ref{eq:trace-form}) is due to the fact that 
 % $G_t^TG_t=tr\{G_tG_t^T\}$ and 
 $\log\det(C_t)=tr\{\log C_t\}$ when $C_t$ is positive definite. 

Recall that $h_1(w) = \cex{S}{w}{ ||G_t-\tilde{g}_t||^2 +tr\left\{C_t\right\}}$ and $h_2(w) = \cex{S}{w}{tr\left\{\log{C_t}\right\}}$, 
% (here we can fix $\tilde{g}_t=\ex{S}{G_t|W_{t-1}=w_{t-1}}$), 
then we have
\begin{eqnarray}
  &&\frac{1}{2}\inf_{\tilde{g}_t,\sigma_t} \frac{1}{\sigma_t^2} \cex{S}{w_{t-1}}{||G_t-\tilde{g}_t||^2 +tr\left\{C_t\right\}}+d \log{\sigma_t^2}-d-\cex{S}{w_{t-1}}{tr\left\{\log{C_t}\right\}}\notag\\
  &\leq& \frac{1}{2}\inf_{\sigma_t>0} \frac{1}{\sigma_t^2}h_1(w_{t-1})+d \log{\sigma_t^2}-d-h_2(w_{t-1})\notag\\
  &=&\frac{1}{2}d\log{\frac{h_1(w_{t-1})}{d}}-\frac{1}{2}h_2(w_{t-1})\notag,
\end{eqnarray}
where we fix an arbitrary $\tilde{g}_t$ and use the optimal $\sigma^*=\sqrt{\frac{h_1(w_{t-1})}{d}}$.

Plugging everything into Lemma \ref{lem:mi-unroll} and Lemma \ref{lem:xu's-bound} will obtain Eq.~(\ref{ineq:iso-gen-bound}).
% Eq. (\ref{ineq:log-x}) is by $log(x+1)\leq x$. In Eq. (\ref{eq:optim-sample-mi}), we let $\tilde{g}_t=\ex{W_{t-1},S}{G_t}$, then optimize over $\sigma^2$.

We then prove the second part. 
% namely Eq.~(\ref{ineq:iso-pop-bound}). 
Let $\tilde{g}_t=\ex{Z}{\nabla\ell(w_{t-1},Z)}$, then 
\begin{align}
    h_1(W_{t-1}) =& \cex{S}{W_{t-1}}{\left|\left|G_t-\tilde{g}_t\right|\right|^2 +tr\left\{C_t\right\}}\notag\\
    =&\cex{S}{W_{t-1}}{tr\left\{(G_t-\tilde{g}_t)((G_t-\tilde{g}_t)^{\bf T}\right\}}+tr\left\{\cex{S}{W_{t-1}}{C_t}\right\}\notag\\
    =&\frac{1}{n}tr\left\{\Sigma_t^\mu\right\}+\frac{n-b}{b(n-1)}tr\left\{\cex{S}{W_{t-1}}{\Sigma_t}\right\}\label{eq:trace-norm-1}\\
    =&\frac{1}{n}tr\left\{\Sigma_t^\mu\right\}+\frac{n-b}{bn}tr\left\{\Sigma_t^\mu\right\}\label{eq:estimate-variance-1}\\
    =&\frac{1}{b}tr\left\{\Sigma_t^\mu\right\},\notag
\end{align}
where Eq.~(\ref{eq:trace-norm-1}) is by $\ex{S}{(G_t-\tilde{g}_t)((G_t-\tilde{g}_t)^{\bf T}}=\frac{1}{n}\Sigma_t^\mu$ for a given $W_{t-1}=w_{t-1}$ and $C_t=\frac{n-b}{b(n-1)}\Sigma_t$, and Eq.~(\ref{eq:estimate-variance-1}) is by $\ex{S}{\Sigma_t}=\frac{n-1}{n}\Sigma_t^\mu$. 
% Plugging the last equation into the bound in Eq.~(\ref{ineq:iso-gen-bound}) will concludes the proof.
This completes the proof.
\end{proof}

\subsection{Proof of Corollary \ref{cor:langevin-dynamic}}

\begin{proof}
Let $C_t=\mathrm{I}_d$, by Theorem \ref{thm:isotropic-prior-bound}, 
\begin{eqnarray}
  \mathcal{E}_{\mu}(\mathcal{A})&\leq&\sqrt{\frac{R^2}{n}\sum_{t=1}^Td\ex{W_{t-1}}{\log{\frac{\cex{S}{W_{t-1}}{ \left|\left|G_t-\tilde{g}_t\right|\right|^2 +tr\left\{C_t\right\}}}{d}}}-\ex{W_{t-1},S}{tr\left\{\log{C_t}\right\}}}\notag\\
  % &=&\sqrt{\frac{R^2}{n}\sum_{t=1}^Td\ex{W_{t-1}}{\log{\frac{\cex{S}{W_{t-1}}{ \left|\left|G_t-\tilde{g}_t\right|\right|^2 +d}}{d}}}}\notag\\
  &=&\sqrt{\frac{R^2}{n}\sum_{t=1}^Td\ex{W_{t-1}}{\log{\frac{\cex{S}{W_{t-1}}{ \left|\left|G_t-\tilde{g}_t\right|\right|^2 }}{d}+1}}}.\notag
\end{eqnarray}
This completes the proof.
% where Eq. (\ref{ineq:sum-root}) is by $\sqrt{\sum_i x_i}\leq \sum_i \sqrt{x_i}$. This completes the proof.
\end{proof}

\subsection{Proof of Theorem~\ref{thm:anisotropic-prior-bound}}
\begin{proof}
Recall Lemma \ref{lem:cmi-golden formula}, we have
\begin{align}
    &I(- G_t + C_t^{1/2}N_t;S|W_{t-1}=w_{t-1}) \notag\\
  \leq&\inf_{\tilde{c}_t}\cex{S}{w_{t-1}}{\mathrm{D_{KL}}(Q_{\widehat{G}_t|w_{t-1},S}||P_{\widehat{G}_t|w_{t-1}})}\notag\\
  =&\inf_{\tilde{c}_t}\cex{S}{w_{t-1}}{\frac{1}{2}\left[\log\frac{\det(\tilde{c}_t\Sigma^\mu_t)}{\det(C_t)} - d + \frac{1}{\tilde{c}_t}((G_t-\tilde{g}_t)^{\bf T} \left(\Sigma^\mu_t\right)^{-1}(G_t-\tilde{g}_t)) + \frac{1}{\tilde{c}_t}tr\left\{\left(\Sigma^\mu_t\right)^{-1} C_t\right\}\right]}\notag\\
  =&\frac{1}{2}\inf_{\tilde{c}_t} \frac{1}{\tilde{c}_t}tr\left\{\left(\Sigma^\mu_t\right)^{-1}\cex{S}{w_{t-1}}{(G_t-\tilde{g}_t)((G_t-\tilde{g}_t)^{\bf T}}\right\}\notag\\
  &\qquad +\frac{1}{\tilde{c}_t}tr\left\{\left(\Sigma^\mu_t\right)^{-1} \cex{S}{w_{t-1}}{C_t}\right\}+tr\left\{\log\Sigma^\mu_t-\cex{S}{w_{t-1}}{\log C_t}\right\}+d \log{\tilde{c}_t}-d\notag\\
  =&\frac{1}{2}\inf_{\tilde{c}_t} \frac{1}{\tilde{c}_tn}tr\left\{\left(\Sigma^\mu_t\right)^{-1}\Sigma^\mu_t\right\}+\frac{n-b}{\tilde{c}_tbn}tr\left\{\left(\Sigma^\mu_t\right)^{-1} \Sigma^\mu_t\right\}+tr\left\{\log\Sigma^\mu_t-\cex{S}{w_{t-1}}{\log C_t}\right\}+d \log{\tilde{c}_t}-d\label{eq:biased-sample-covariance}\\
  =&\frac{1}{2}\inf_{\tilde{c}_t} \frac{d}{\tilde{c}_tn}+\frac{(n-b)d}{\tilde{c}_tbn}+tr\left\{\log\Sigma^\mu_t-\cex{S}{w_{t-1}}{\log C_t}\right\}+d \log{\tilde{c}_t}-d\notag\\
  =&\frac{1}{2}\inf_{\tilde{c}_t} 
  \frac{d}{b\tilde{c}_t}+d \log{\tilde{c}_t}+tr\left\{\log\Sigma^\mu_t-\cex{S}{w_{t-1}}{\log C_t}\right\}-d\notag\\
  =&\frac{d}{2}\log{\frac{1}{b}}+\frac{1}{2}tr\left\{\log\Sigma^\mu_t-\cex{S}{w_{t-1}}{\log C_t}\right\},\notag
\end{align}
where the last equality hold when $\tilde{c}^*_t=1/{b}$ and Eq.~(\ref{eq:biased-sample-covariance}) is by 
\[
\cex{S}{w_{t-1}}{(G_t-\tilde{g}_t)((G_t-\tilde{g}_t)^{\bf T}}=\frac{1}{n}\Sigma_t^\mu, \quad\text{and}
\]
\[
\cex{S}{w_{t-1}}{C_t}=\frac{n-b}{b(n-1)}\cex{S}{w_{t-1}}{\Sigma_t}=\frac{n-b}{b(n-1)}\frac{n-1}{n}\Sigma_t^\mu=\frac{n-b}{bn}\Sigma_t^\mu.
\]
This completes the proof.
\end{proof}

\subsection{Proof of Lemma~\ref{lem:compare-iso-noniso}}
\begin{proof}
Let the diagonal element of ${\Sigma^\mu_t}/{b}$ in dimension $k$ be $a_k$, then
\begin{align*}
    \sum_{k=1}^d\log a_k\leq
    (\sum_{k=1}^d 1) \cdot \log{(\sum_{k=1}^d a_k)}/{(\sum_{k=1}^d 1)}=d\log({tr\left\{\Sigma^\mu_t\right\}}/{bd}),
\end{align*}
where we invoke Lemma~\ref{lem:log-sum-ineq}. 

This completes the proof.
\end{proof}

\subsection{Additional Result via Data-Dependent Prior}
\label{sec:other bounds}
With the same spirit of Lemma \ref{lem:mi-unroll}, to apply Lemma \ref{lem:data-dependent-prior} to iterative algorithms, we also need the lemma below, which using the full training trajectories KL divergence to upper bound the final output KL divergence.
\begin{lem}[{\citet[Proposition ~2.6.]{negrea2019information}}]
\label{lem:kl-decomposition}
Assume that $P_{W_0}=Q_{W_0}$, then $\mathrm{D_{KL}}(P_{W_T}||Q_{W_{T}})\leq\sum_{t=1}^T\ex{W_{0:t-1}}{\mathrm{D_{KL}}(P_{W_{t}|W_{0:t-1}}||Q_{W_{t}|W_{0:t-1}})}.$
\end{lem}

Let $G_{Jt}\triangleq\nabla L_{S_J}(W_{t-1})$, the SDE approximation of this prior updating is defined as:
\[
  W_t = W_{t-1} - \eta G_{Jt}+\eta C^{\frac{1}{2}}_{Jt} N_t,
\]
where $C_{Jt}  = \frac{1}{b}\left(\frac{1}{m}\sum_{i\in J}\nabla \ell_i\nabla \ell_i^{\bf T}-G_{Jt}G_{Jt}^{\bf T}\right)$ is the gradient noise covariance of the prior process. In this case, the prior distribution $P_{\mathcal{G}_{Jt}|W_{0:t-1}}$ will be an anisotropic Gaussian distribution. We also assume $n\gg b$, then $C_t=\frac{1}{b}\Sigma_t$.

We denote the difference between $G_t$ and $G_{Jt}$ by
\[
  \xi_t \triangleq  G_{Jt} - G_t.
\]
To see the relationship between $\xi_t$, $C_{Jt}$ and $C_t$, we present a useful lemma below.
\begin{lem}
If $m=n-1$, then the following two equations hold,
\[\ex{}{\xi_t\xi_t^{\bf T}} = \frac{b}{(n-1)^2}C_t, \quad \ex{}{C_{Jt}}=\frac{n(n-2)}{(n-1)^2}C_t,\]
where the expectation is taken over $J$.
\label{lem:disjoint-var-avg}
\end{lem}

Instead of using Lemma~\ref{lem:data-dependent-prior}, we invoke the following result which is a simple extension of \cite[Theorem~2.5]{negrea2019information}.
\begin{lem}[{\cite[Theorem~1.]{wang2021optimizing}}]
\label{lem:data-dependent-prior-2}
Assume the loss $\ell(w,Z)$ is bounded in $[0,M]$, the expected generalization gap is bounded by
\[
\mathcal{E}_{\mu}(\mathcal{A})\leq\frac{M}{\sqrt{2}}\mathbb{E}_{S,J}{\sqrt{\mathrm{D_{KL}}(P_{W|S_J}||Q_{W|S})}}
\]
\end{lem}
\paragraph{Comparison with the work of  \cite{wang2021optimizing}}
    \cite{wang2021optimizing} studies the algorithm of SGD with anisotropic noise, while our SDE analysis focuses on GD with anisotropic noise. This means that the discrete gradient noise arising from mini-batch sampling still exists in their analyzed algorithm, whereas the gradient noise is fully modeled as Gaussian in our Section~\ref{sec:IT-SGD}. Moreover, \cite{wang2021optimizing} uses matrix analysis tools to optimize the prior distribution. A significant distinction lies in their optimization analysis, which relies on the assumption that the trace of gradient noise covariance remains unchanged during training (see {\bf Constriant~1} in their paper). Additionally, their final optimal posterior covariance is derived based on the assumption that the posterior distribution of $W$ is invariant to the data index, see Assumption 1 in their paper. In contrast, our Section~\ref{sec:IT-SGD} avoids making these assumptions and demonstrates the superiority of population gradient noise covariance (GNC) in Lemma~\ref{lem:compare-iso-noniso}, by invoking a variant of the log-sum inequality. In summary, our proof is simpler and more straightforward, while \cite{wang2021optimizing} makes a stronger claim about the optimality of population GNC based on their additional assumptions.

As introduced in \cite{wang2021optimizing}, the subsequent analysis based on the data-dependent prior bound will rely on an additional assumption.
\begin{assum}
\label{ass:invariant}
When $m=n-1$, given dataset $S=s$, the distribution $P_{W_t|J,S_J}$ is invariant of $J$.
\end{assum}
In \cite{wang2021optimizing}, authors mention that in practice, $n$ is usually very large, so this assumption hints that changing one instance in $S_{J}$ will not make $P_{W_t|J,S_J}$ be too different.

We are now in a position to state the following theorem.
\begin{thm}
\label{thm:data-dependent-bound}
Assume the loss $\ell(w,Z)$ is bounded in $[0,M]$ and Assumption 
% \ref{assum-sde} and 
\ref{ass:invariant} hold,the expected generalization gap of SGD is bounded by
\[
\mathcal{E}_{\mu}(\mathcal{A})\leq\ex{S}{\sqrt{M^2\sum_{t=1}^T\ex{W_{t-1}}{\left(\frac{(b-1)d}{(n-1)^2}+tr\left\{\ex{J}{\log{C_tC^{-1}_{Jt}}}\right\}\right)}}}.
\]
% where $h_3(W_{t-1}) = tr\left\{\log{C_t}- \ex{}{\log{C_{Jt}}}\right\}$.
\end{thm}

\begin{proof}
By Lemma \ref{lem:data-dependent-prior-2} and Lemma \ref{lem:kl-decomposition}, we have
\begin{align}
  \mathcal{E}_{\mu}(\mathcal{A})\leq&\ex{S,J}{\sqrt{\frac{R'^2}{2}\sum_{t=1}^T\ex{W_{0:t-1}|S,J}{\mathrm{D_{KL}}(P_{W_t|W_{0:t-1},S_J}||Q_{W_t|W_{0:t-1},S})}}}\notag\\
  \leq&\ex{S}{\sqrt{\frac{R'^2}{2}\sum_{t=1}^T\ex{W_{0:t-1}|S,J}{\ex{J}{\mathrm{D_{KL}}(P_{W_t|W_{0:t-1},S_J}||Q_{W_t|W_{0:t-1},S})}}}},\label{ineq:exchange-expectation}
\end{align}
where Eq. (\ref{ineq:exchange-expectation}) is by Jensen's inequality and Assumption \ref{ass:invariant}.

Recall $\Sigma_t=\frac{1}{n}\sum_{i=1}^n\nabla \ell(W_{t-1},Z_i)\nabla \ell(W_{t-1},Z_i)^{\bf T}-\nabla L_S(W_{t-1})\nabla L_S(W_{t-1})^{\bf T}$ and $C_t = \frac{1}{b}\Sigma_t$.

By the KL divergence between two Gaussian distributions, for any $t\in [T]$, we have
\begin{eqnarray}
  &&\ex{J}{{\mathrm{D_{KL}}(P_{W_t|W_{0:t-1},S_J}||Q_{W_t|W_{0:t-1},S})}}\notag\\&=&\ex{J}{\frac{1}{2}\left(\xi_t^{\bf T} C_{t}^{-1} \xi_t+\log{\frac{\det(C_t)}{\det(C_{Jt})}}+tr\{C_t^{-1}C_{Jt}\}-d\right)}\label{eq:exp-j}\\
  &=&\frac{1}{2}\left(tr\{C_{t}^{-1}\ex{J}{\xi_t  \xi_t^{\bf T}}\}+\ex{J}{\log{\frac{\det(C_t)}{\det(C_{Jt})}}}+\ex{J}{tr\{C_t^{-1}C_{Jt}\}}-d\right)\notag\\
  &=&\frac{1}{2}\left(\frac{1}{(n-1)^2}tr\{C_{t}^{-1}\Sigma_t\}+\ex{J}{\log{\frac{\det(C_t)}{\det(C_{Jt})}}}+\ex{J}{tr\{C_t^{-1}C_{Jt}\}}-d\right)\label{eq:variance-disjoint-set}\\
  &=&\frac{1}{2}\left(\frac{b}{(n-1)^2}tr\{\Sigma_{t}^{-1}\Sigma_t\}+\ex{J}{\log{\frac{\det(C_t)}{\det(C_{Jt})}}}+tr\{C_t^{-1}\ex{J}{C_{Jt}}\}-d\right)\notag\\
%   &=&\frac{1}{2}\left(\frac{bd}{(n-1)^2}+\frac{n(n-2)d}{(n-1)^2}+tr\{\log C_t- \ex{J}{\log{C_{Jt}}}-\mathrm{I}_d\}\right)\\
  &=&\frac{1}{2}\left(\frac{bd}{(n-1)^2}+\frac{n(n-2)d}{(n-1)^2}-d+tr\{\log C_t- \ex{J}{\log{C_{Jt}}}\}\right)\label{eq:average-disjoint-set}\\
  &=&\frac{1}{2}\left(\frac{(b-1)d}{(n-1)^2}+tr\{\log C_t- \ex{J}{\log{C_{Jt}}}\}\right)\notag
%   &\leq&\frac{1}{2}\left(\frac{bd}{(n-1)^2}+tr\{C_t+ \ex{J}{{C^{-1}_{Jt}}}+C_t^{-1}\ex{J}{C_{Jt}}-3\mathrm{I}_d\}\right),
\end{eqnarray}
where Eq. (\ref{eq:variance-disjoint-set}) and Eq. (\ref{eq:average-disjoint-set}) are by Lemma \ref{lem:disjoint-var-avg}. This concludes the proof.
\end{proof}
\begin{rem}
    If the bound in \cite{negrea2019information} is used, then the first term in Eq.~(\ref{eq:exp-j}) is $\xi_t^{\bf T}C_{Jt}^{-1}\xi_t$, where both $C_{Jt}$ and $\xi_t$ dependent on $J$,  making the bound difficult to analyze.
\end{rem}
 The effect of $tr\{{\log{C_t}}\}$ on the magnitude of the bound can be decreased by the $tr\{\mathbb{E}_{J}{\log{C_{Jt}}}\}$. If we further consider Taylor expansion of the function $\log C_{Jt}$ around $\mathbb{E}_J[C_{Jt}]$, we have a well-known approximation 
 \[\ex{}{\log C_{Jt}}\approx\log\ex{}{C_{Jt}}-\mathrm{Var}(C_{Jt})/(2\mathbb{E}^2[C_{Jt}]).\]
 Thus, recall Lemma \ref{lem:disjoint-var-avg}, the difference between $tr\{{\log{C_{t}}}\}$ and  $tr\{\mathbb{E}_{J}{\log{C_{Jt}}}\}$ would become:
\[
\log{(1+1/(n^2-2n))}+\mathrm{Var}(C_{Jt})/(2\mathbb{E}^2[C_{Jt}]).
\]
Thus, the generalization gap should be characterized by the second term above.

When $n\rightarrow \infty$, the first term will converges to zero, and for the second term, $\mathbb{E}^2[C_{Jt}]$ will converge to a constant by Lemma \ref{lem:disjoint-var-avg}, and then the bound is  $\mathrm{Var}(C_{Jt})$ will also converges to zero.

% \subsection{Proof of Lemma \ref{lem:disjoint-var-avg}}
% 

% \subsection{Proof of Theorem \ref{thm:data-dependent-bound}}

\section{Omitted Proofs, Additional Results and Discussions in Section 
% ``Generalization Bounds Via Terminal State'' 
\ref{sec:pac-bayes}
}

In fact, this section provides a PAC-Bayes type analysis. The connection between information-theoretic bounds and PAC-Bays bounds have already been discussed in many previous works \citep{bassily2018learners,hellstrom2020generalization,alquier2021user}. Roughly speaking, the most significant component of a PAC-Bayes bound is the KL divergence between the posterior distribution of a randomized algorithm output and a prior distribution, i.e. $\mathrm{D_{KL}}(Q_{W_T|S}||P_{N})$ for some prior $P_N$. In essence, information-theoretic bounds can be view as having the same spirit. For concreteness, in Lemma \ref{lem:xu's-bound}, 
$I(W_T;S)=\mathbb{E}_S[\mathrm{D_{KL}}(Q_{W_T|S}||P_{W_T})]$, in which case the marginal $P_{W_T}$ is used as a prior of the algorithm output. Furthermore, by using  Lemma \ref{lem:mi-center-gravity}, we have $I(W_T;S) \leq \inf_{P_N} \mathbb{E}_S[\mathrm{D_{KL}}(Q_{W_T|S}||P_N)]$. Hence, Lemma \ref{lem:xu's-bound} can be regarded as a PAC-Bayes bound with the optimal prior. In addition, the PAC-Bayes framework is usually used to provide a high-probability bound, %(with respect to the randomness of $S$), 
while information-theoretic analysis is applied to bounding the expected generalization error. In this sense, information-theoretic framework is closer to another concept called MAC-Bayes \citep{grunwald2021pac}.

\subsection{Proof of Lemma~\ref{lem:stationary-real}}
% We note that this lemma can be recovered from \citet[Theorem~1. and Theorem~4.]{liu2021noise}, we provide a proof here for self-containing. 
\begin{proof}
%     When $w$ is close to any local minimum $w^*$, we can use a second-order Taylor expansion to approximate the value of the loss at $w$, 
% \begin{eqnarray}
%   L_s(w) \approx L_s(w^*) + \frac{1}{2}(w-w^*)^\mathrm{\bf T} H_{w^*}(w-w^*).
%   % \label{eq:second-order-taylor}
% \end{eqnarray}

% Then, when $w_t\to w^*$, we have 
Recall $G_t=\nabla L_s(w_t)=H_{w^*}\pr{w_t-w^*}$.
and Eq.~(\ref{eq:sgd-update-2}), then
\begin{align*}
    w_{t} =& w_{t-1} - \eta G_t + \eta V_t\notag\\
    =&w_{t-1} - \eta H_{w^*}\pr{w_{t-1}-w^*} + \eta V_t.
\end{align*}

Let $W'_t\triangleq W_t-w^*$. Thus, as $T\to \infty$,
\begin{align*}
    &\ex{W'_{T}}{W'_{T}{W'_{T}}^\mathrm{\bf T}}\\
    =&\ex{W'_{T-1}, V_T}{\pr{W'_{T-1}-\eta H_{w^*}W'_{T-1} + \eta V_t}\pr{W'_{T-1}-\eta H_{w^*}W'_{T-1} + \eta V_t}^\mathrm{\bf T}}\\
    =&\ex{W'_{T-1}}{W'_{T-1}{W'}^\mathrm{\bf T}_{T-1}-\eta H_{w^*}W'_{T-1}{W'}^\mathrm{\bf T}_{T-1}-\eta W'_{T-1}{W'}^\mathrm{\bf T}_{T-1}H_{w^*}+\eta^2H_{w^*}W'_{T-1}{W'}^\mathrm{\bf T}_{T-1}H_{w^*}}\\
     &\qquad\qquad\qquad +\eta^2\ex{V_T}{V_T{V_T}^\mathrm{\bf T}},%\label{eq:use-center-mean}
\end{align*}
where the last equation is by $\cex{V_T}{w_{T-1}}{V_{T}}=0$.

Recall that $\ex{V_T}{V_T{V_T}^\mathrm{\bf T}}=C_T$ and notice that $\ex{W'_{T}}{W'_{T}{W'_{T}}^\mathrm{\bf T}}=\ex{W'_{T-1}}{W'_{T-1}{W'_{T-1}}^\mathrm{\bf T}}=\Lambda_{w^*}$ when $T\to\infty$ (i.e. ergodicity), we have
\[
\Lambda_{w^*} H_{w^*} + H_{w^*} \Lambda_{w^*}-\eta H_{w^*} \Lambda_{w^*}H_{w^*} = \eta C_{T}.
\]

Furthermore,  if $H_{w^*}$ and $\Lambda_{w^*}$ commute, namely $\Lambda_{w^*}H_{w^*}=H_{w^*}\Lambda_{w^*}$, we have 
\[
\br{H_{w^*}\pr{2\mathrm{I}_d-\eta H_{w^*}}}\Lambda_{w^*}=\eta C_T, 
\]
which will give use $\Lambda_{w^*}=\eta \br{H_{w^*}\pr{2\mathrm{I}_d-\eta H_{w^*}}}^{-1}C_{T}$.

This completes the proof.
\end{proof}

% \subsection{Proof of Lemma~\ref{lem:solution-stationary}}
% \begin{proof}

%     Although controversy exists \citep{ziyin2022strength}, 
% % the approximation in Eq. (\ref{eq:approx-hessian-gradient}) below, where 
% Hessian is proportional to the GNC near local minima when the loss is the negative log likelihood. To see this, we first note that the remaining analysis is all based on selecting the log-loss, i.e. cross-entropy loss, as the loss function $\ell$. Thus, when $w_t\to w^*$, we have,
% \[
% \Sigma_{w^*}=\frac{1}{n}\sum_{i=1}^n\nabla \ell_i\nabla \ell_i^T-G_tG_t^T\approx \frac{1}{n}\sum_{i=1}^n\nabla \ell_i\nabla \ell_i^T=F_{w^*},
% \]
% where $F_{w^*}$ is the \textit{Fisher information matrix} (FIM). This approximation is true because gradient noise dominates over  gradient mean near local minima. Moreover, FIM is close to the Hessian near local minima with the log-loss \citep[Chapter~8]{pawitan2001all}, namely, $F_{w^*}\approx H_{w^*}$. Let $n\gg b$, we have
% \begin{align}
% \label{eq:approx-hessian-gradient}
%     H_{w^*} \approx \Sigma_{w^*} =  b C_{w^*}.
% \end{align}
% % \textcolor{red}{Need more justification of this approximation. For Cross-entropy loss?}
% Similar approximation is widely used in the literature \citep{jastrzkebski2017three,zhu2019anisotropic,li2020hessian,xie2020diffusion,xie2021positive,liu2021noise}. Therefore, if Eq.~(\ref{eq:approx-hessian-gradient}) holds, then the solution to the equation in Lemma~\ref{lem:stationary-real} is
% \[
%     \Lambda_{w^*}=\frac{\eta}{b}(2\mathrm{I}_d-\eta H_{w^*})^{-1}.
% \]
% This completes the proof.
% \end{proof}

\subsection{Theorem~\ref{thm:state-bound-dis-prior}: A General Bound}

The following bound can be easily proved by using Eq.~(\ref{eq:kl-two-gau}).
\begin{thm}
\label{thm:state-bound-dis-prior}
Under the same conditions in Lemma~\ref{lem:xu's-bound} and Lemma~\ref{lem:stationary-real}, then for any $P_{W_{T}}=\mathcal{N}\pr{\tilde{w}, \widetilde{\Lambda}}$, where $\tilde{w}$ and $\widetilde{\Lambda}$ are independent of $S$,  we have
\[
\mathcal{E}_{\mu}(\mathcal{A})\leq\sqrt{\frac{R^2}{2n}\inf_{\tilde{w}, \widetilde{\Lambda}}\ex{S,W_S^*}{\log\frac{\mathrm{det}\pr{\widetilde{\Lambda}}}{\mathrm{det}\pr{\Lambda_{W^*_{S}}}}+\tr{\widetilde{\Lambda}^{-1}\Lambda_{W^*_{S}}-\mathrm{I}_d}+\mathrm{d}_{\mathrm{M}}^2\pr{W^*_{S},\tilde{w};\widetilde{\Lambda}}}},
\]
where $\mathrm{d}_{\mathrm{M}}\pr{x,y;\Sigma}\triangleq \sqrt{(x-y)^{\bf T}\Sigma^{-1}(x-y)}$ is the Mahalanobis distance.
\end{thm}

\subsection{Proof of Theorem~\ref{thm:opt-state-inde-bound}}
\begin{proof}
    Let $P_{W_T}=\mathcal{N}\pr{w^*_{\mu},\Lambda_{w^*_{\mu}}}$, then
    \begin{align}
        &\ex{S,W_S^*}{\log\frac{\mathrm{det}\pr{\Lambda_{w^*_{\mu}}}}{\mathrm{det}\pr{\Lambda_{W^*_{S}}}}+\tr{\Lambda_{w^*_{\mu}}^{-1}\Lambda_{W^*_{S}}-\mathrm{I}_d}+\pr{W_S^*-w^*_{\mu}}^{\bf T}\Lambda_{w^*_{\mu}}^{-1}\pr{W_S^*-w^*_{\mu}}}\notag\\
        =&\ex{S,W_S^*}{\log\frac{\mathrm{det}\pr{\Lambda_{w^*_{\mu}}}}{\mathrm{det}\pr{\Lambda_{W^*_{S}}}}+\tr{\Lambda_{w^*_{\mu}}^{-1}\Lambda_{W^*_{S}}-\mathrm{I}_d}+\tr{\Lambda_{w^*_{\mu}}^{-1}\pr{W_S^*-w^*_{\mu}}\pr{W_S^*-w^*_{\mu}}^{\bf T}}}\notag\\
        =&\ex{S,W_S^*}{\log\frac{\mathrm{det}\pr{\Lambda_{w^*_{\mu}}}}{\mathrm{det}\pr{\Lambda_{W^*_{S}}}}}+\tr{\Lambda_{w^*_{\mu}}^{-1}\ex{S,W_S^*}{\Lambda_{W^*_{S}}}-\mathrm{I}_d+\Lambda_{w^*_{\mu}}^{-1}\ex{W_S^*}{\pr{W_S^*-w^*_{\mu}}\pr{W_S^*-w^*_{\mu}}^{\bf T}}}.\label{eq:kl-pop-stacov}
    \end{align}

    Denote $\widetilde{\Sigma}_{\mu}\triangleq\ex{S,W_S^*}{\pr{W_S^*-w^*_{\mu}}\pr{W_S^*-w^*_{\mu}}^{\bf T}}=\ex{W_S^*}{W_S^*{W_S^*}^{\bf T}}-w^*_{\mu}{w^*_{\mu}}^{\bf T}$.

    Notice that
    \begin{align*}
        \ex{S,W_S^*}{\Lambda_{W^*_{S}}}=&\ex{S,W_S^*,W_T}{\pr{W_T-W_S^*}\pr{W_T-W_S^*}^{\bf T}}\\
        =&\ex{W_T}{W_T{W_T}^{\bf T}}-\ex{W_S^*}{W_S^*{W_S^*}^{\bf T}}\\
        =&\ex{W_T}{W_T{W_T}^{\bf T}}-w^*_{\mu}{w^*_{\mu}}^{\bf T}-\pr{\ex{W_S^*}{W_S^*{W_S^*}^{\bf T}}-w^*_{\mu}{w^*_{\mu}}^{\bf T}}\\
        =&\Lambda_{w^*_{\mu}}-\widetilde{\Sigma}_{\mu}.
    \end{align*}

    Therefore,
    \begin{align*}
        &\tr{\Lambda_{w^*_{\mu}}^{-1}\ex{S,W_S^*}{\Lambda_{W^*_{S}}}-\mathrm{I}_d+\Lambda_{w^*_{\mu}}^{-1}\ex{W_S^*}{\pr{W_S^*-w^*_{\mu}}\pr{W_S^*-w^*_{\mu}}^{\bf T}}}\\
        =&\tr{\Lambda_{w^*_{\mu}}^{-1}\ex{S,W_S^*}{\Lambda_{W^*_{S}}}-\Lambda_{w^*_{\mu}}^{-1}\Lambda_{w^*_{\mu}}+\Lambda_{w^*_{\mu}}^{-1}\ex{W_S^*}{\pr{W_S^*-w^*_{\mu}}\pr{W_S^*-w^*_{\mu}}^{\bf T}}}\\
        =&\tr{\Lambda_{w^*_{\mu}}^{-1}\pr{\ex{S,W_S^*}{\Lambda_{W^*_{S}}}-\Lambda_{w^*_{\mu}}+\widetilde{\Sigma}_{\mu}}}\\
        =&0.
    \end{align*}

    Plugging this into Eq.~(\ref{eq:kl-pop-stacov}), we have
    \begin{align*}
        &\ex{S,W_S^*}{\log\frac{\mathrm{det}\pr{\Lambda_{w^*_{\mu}}}}{\mathrm{det}\pr{\Lambda_{W^*_{S}}}}+\tr{\Lambda_{w^*_{\mu}}^{-1}\Lambda_{W^*_{S}}-\mathrm{I}_d}+\pr{W_S^*-w^*_{\mu}}^{\bf T}\Lambda_{w^*_{\mu}}^{-1}\pr{W_S^*-w^*_{\mu}}}\\
        &=\ex{S,W_S^*}{\log\frac{\mathrm{det}\pr{\Lambda_{w^*_{\mu}}}}{\mathrm{det}\pr{\Lambda_{W^*_{S}}}}}=\ex{S,W_S^*}{\tr{\log\pr{\Lambda_{W^*_{S}}^{-1}\Lambda_{w^*_{\mu}}}}}.
    \end{align*}
    
    Finally, applying Theorem~\ref{thm:state-bound-dis-prior} will conclude the proof.
\end{proof}

\subsection{Proof of Corollary~\ref{cor:pacbayes-anisotropic-prior}}
\begin{proof}
% The first part in the statement
The proof is straightforward by plugging $\Lambda_{w^*}= \br{H_{w^*}\pr{\frac{2}{\eta}\mathrm{I}_d}}^{-1}C_{T}$ in Theorem~\ref{thm:opt-state-inde-bound}.
\end{proof}

\subsection{Proof of Corollary~\ref{cor:pacbayes-isotropic-prior}}

\begin{proof}
    By Lemma~\ref{lem:log-sum-ineq}, it's easy to obtain the following bound according to Theorem~\ref{thm:opt-state-inde-bound}.
    \[
\mathcal{E}_{\mu}(\mathcal{A})\leq\sqrt{\frac{R^2d}{2n}\log\pr{\frac{\ex{}{\mathrm{d}^2_{\mathrm{M}}\pr{W_S^*,w^*_\mu;\ex{}{\Lambda_{W^*_S}}}}}{d}+1}+\ex{}{\tr{\log\pr{\Lambda_{W^*_S}^{-1}\ex{}{\Lambda_{W^*_S}}}}}}.
\label{ineq:losser-state-bound}
\]

Then, plugging $\Lambda_{W^*_{S}}  = \frac{\eta}{2b} \mathrm{I}_d$ will conclude the proof.
\end{proof}

\subsection{Corollary~\ref{cor:pacbayes-isotropic-prior-init}: Distance to Initialization}

\begin{cor}
\label{cor:pacbayes-isotropic-prior-init}
Under (i-iii) in Lemma~\ref{lem:stationary-real}, then $
\mathcal{E}_{\mu}(\mathcal{A})\leq\sqrt{\frac{dR^2}{n}\log\left(\frac{2b}{\eta d}\mathbb{E}{||W_S^*-W_{0}||^2}+1\right)}
$.
% where $\tilde{w}=w_\mu^* =\ex{}{W_S^*}$ and the optimal $\sigma^*=\sqrt{\ex{}{||W_S^*-\tilde{w}||^2/d+\frac{\eta }{2b}}}$.
\end{cor}
\begin{proof} 
Notice that $I(W_T;S)\leq\mathbb{E}_{S}{\mathrm{D_{KL}}(Q_{W_T|S}||P_{W_T})}$ holds for any $\sigma>0$, then for a given $\tilde{w}$, we have
\begin{eqnarray}
      I(W_T;S)&=&\inf_{P_{W_T}} \ex{S}{\mathrm{D_{KL}}(Q_{W_T|S}||P_{W_T})}\notag\\ &\leq&\inf_{\sigma}\ex{S}{\mathrm{D_{KL}}(P_{W^*_{S}+\sqrt{\frac{\eta}{2b}}N, W^*_{S}|S}||P_{\tilde{w}+\sigma N})}\label{ineq:kl-chain}\\
      &=&\inf_{\sigma}\ex{S,W^*_{S}}{\mathrm{D_{KL}}(P_{W^*_{S}+\sqrt{\frac{\eta}{2b}}N, |S,W^*_{S}}||P_{\tilde{w}+\sigma N})}\notag\\
      &=& \inf_{\sigma} \frac{1}{2}\ex{S,W^*_{S}}{\frac{1}{\sigma^2}(W^*_{S}-\tilde{w})^{\bf T}(W^*_{S}-\tilde{w})+\log\frac{\sigma^{2d}}{(\eta/2b)^d}+tr\{\frac{\eta}{2b\sigma^2}\mathrm{I}_d\}-d}\notag\\
      &=& \frac{1}{2}\inf_{\sigma} \frac{1}{\sigma^2}\ex{S,W^*_{S}}{||W^*_{S}-\tilde{w}||^2+\frac{\eta d}{2b}} + d\log{\sigma^{2}}+ d\log\frac{2b}{\eta}-d\notag\\
      &=&\frac{1}{2} d\log\left(\frac{2b}{\eta d}\ex{S,W^*_{S}}{||W^*_{S}-\tilde{w}||^2}+1\right),
\end{eqnarray}
where Eq.~(\ref{ineq:kl-chain}) is by the chain rule of KL divergence, and the optimal $\sigma^*=\sqrt{\ex{S,W^*_{S}}{||W^*_{S}-\tilde{w}||^2/d+\frac{\eta }{2b}}}$. Let $\tilde{w}=W_0$ will conclude the proof.
\end{proof}

% \subsection{Recover Gradient Norm Based Bound from Theorem~\ref{thm:pacbayes-isotropic-prior}}
Additionally, Corollary~\ref{cor:pacbayes-isotropic-prior-init} can be used to recover a trajectory-based bound.
\begin{cor}
\label{cor:pacbayes-gradient}
Let $W_T=W_s^*$, $\tilde{w}=0$ and W.L.O.G, assume $W_0=0$, then
\[
\mathcal{E}_{\mu}(\mathcal{A})\leq\sqrt{\frac{dR^2}{n}\log\left(\frac{4bT\eta}{d}\sum_{t=1}^T\ex{}{||G_t||^2+tr\{C_t\}}+1\right)},
\]
\end{cor}
\begin{rem}
In Theorem \ref{thm:isotropic-prior-bound}, let $\tilde{g}=0$  and by applying Jensen's inequality, we could also let the summation and factor $T$ move inside the square root. Then the most different part in Corollary \ref{cor:pacbayes-gradient} is that $A_2(t)$ is now removed from the bound. 
% Since $-tr\{\log{C_t}\}$ is usually has very large magnitude in practice, this improvement is significant.
\end{rem}

\begin{proof}
When $W_0 = 0$, we notice that
\[
W_T = \sum_{t=1}^T -\eta G_t + \eta N_{C_t},
\]
where $N_{C_t} = C_t^{1/2} N_t$.

Thus,
\[
||W_T||^2=||\sum_{t=1}^T -\eta G_t + \eta N_{C_t}||^2\leq 2T\eta^2\sum_{t=1}^T||G_t||^2+||N_{C_t}||^2
\]

Let $\tilde{w}=0$, recall the bound in Corollary~\ref{cor:pacbayes-isotropic-prior-init} and plugging the inequality above, we have
\begin{eqnarray}
  \mathcal{E}_{\mu}(\mathcal{A})&\leq&\sqrt{\frac{R^2}{n}d\log\left(\frac{2b}{\eta d}\ex{S,W_T}{||W_T-\tilde{w}||^2}+1\right)}\notag\\
  &\leq&\sqrt{\frac{dR^2}{n}\log\left(4bT\eta/d\ex{S,W_{0:T-1},N_{C_{0:t-1}}}{\sum_{t=1}^T||G_t||^2+||N_{C_t}||^2}+1\right)}\notag\\
  &=&\sqrt{\frac{dR^2}{n}\log\left(\frac{4bT\eta}{d}\sum_{t=1}^T\ex{S,W_{t-1}}{||G_t||^2+tr\{C_t\}}+1\right)}\notag
\end{eqnarray}
This concludes the proof.
\end{proof}

\subsection{Proof of Theorem \ref{thm:pacbayes-data-dependent-prior}}

\begin{proof}
Let $P_{W_T|S_J=s_j} = \mathcal{N}(W^*_{s_j},\frac{\eta}{2b}\mathrm{I}_d)$, then
\begin{eqnarray}
  \mathrm{D_{KL}}(Q_{W_T|S=s}||P_{W_T|S_J=s_j})&=&\mathrm{D_{KL}}(Q_{W^*_{s}+\sqrt{\frac{\eta}{2b}}N|S=s}||P_{W^*_{s_j}+\sqrt{\frac{\eta}{2b}}N|S_J=s_j})\notag\\
  &\leq&\mathrm{D_{KL}}(Q_{W^*_{s}+\sqrt{\frac{\eta}{2b}}N,W^*_{s}|S=s}||P_{W^*_{s_j}+\sqrt{\frac{\eta}{2b}}N,W^*_{s_j}|S_J=s_j})\label{ineq:kl-chain-2}\\
  &=&\ex{W^*_{s},W^*_{s_j}}{\mathrm{D_{KL}}(Q_{W^*_{s}+\sqrt{\frac{\eta}{2b}}N|W^*_{s},S=s}||P_{W^*_{s_j}+\sqrt{\frac{\eta}{2b}}N|W^*_{s_j},S_J=s_j})}\notag\\
  &=&\ex{W^*_{s},W^*_{s_j}}{\frac{b}{\eta}||W^*_{s}-W^*_{s_j}||^2},\label{eq:bound-data-prior}
\end{eqnarray}
where Eq.~(\ref{ineq:kl-chain-2}) is by the chain rule of KL divergence.
Plugging the Eq. (\ref{eq:bound-data-prior}) into Lemma \ref{lem:data-dependent-prior} will obtain the final result.
\end{proof}

\subsection{Additional Empirical Results}

\begin{figure*}[!ht]
    \begin{subfigure}[b]{0.32\textwidth}
    \centering
\includegraphics[scale=0.33]{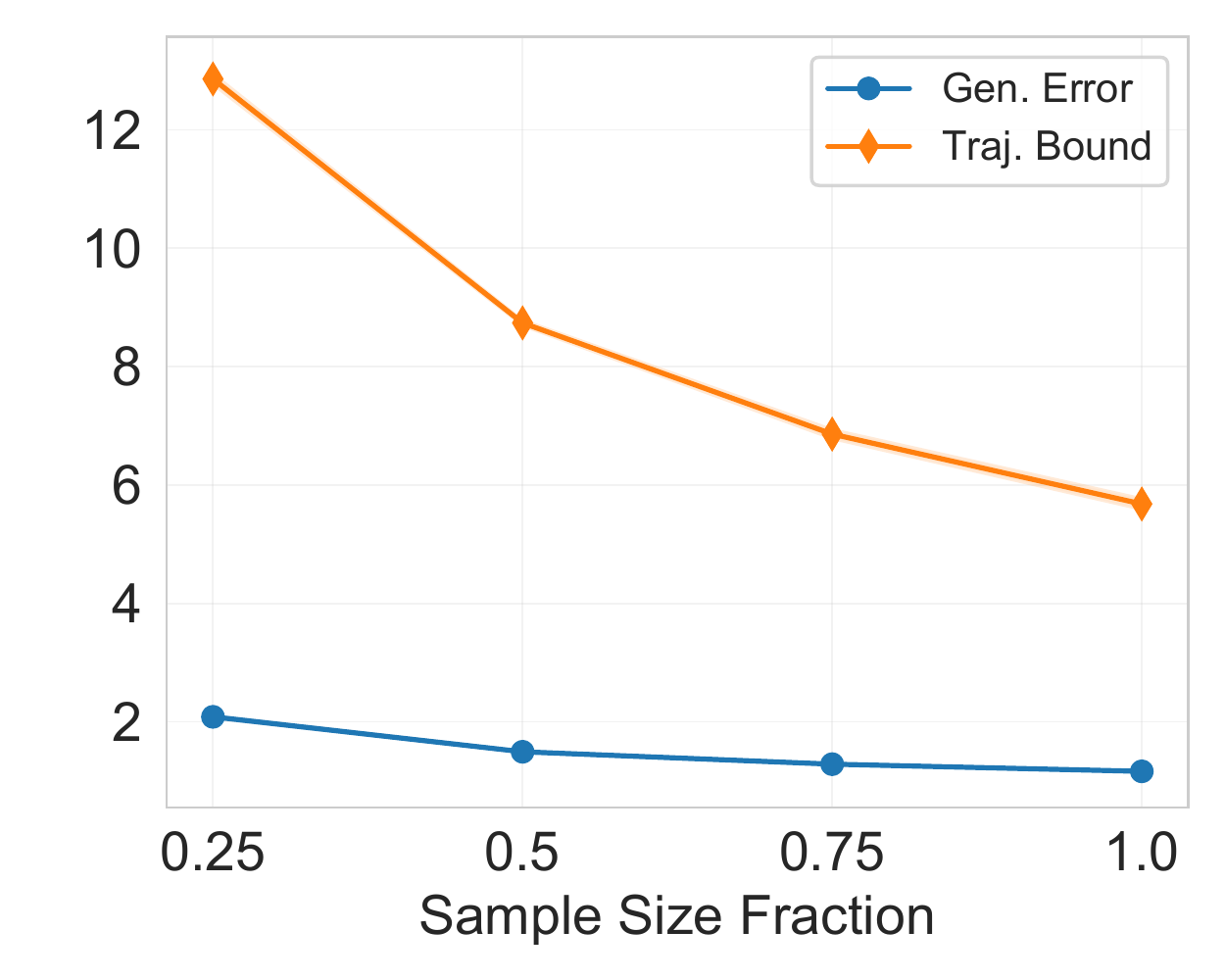}    
\caption{VGG on (small) SVHN}            \label{fig:vgg-svhn-err}
    \end{subfigure}
\begin{subfigure}[b]{0.32\textwidth}
\centering
\includegraphics[scale=0.33]{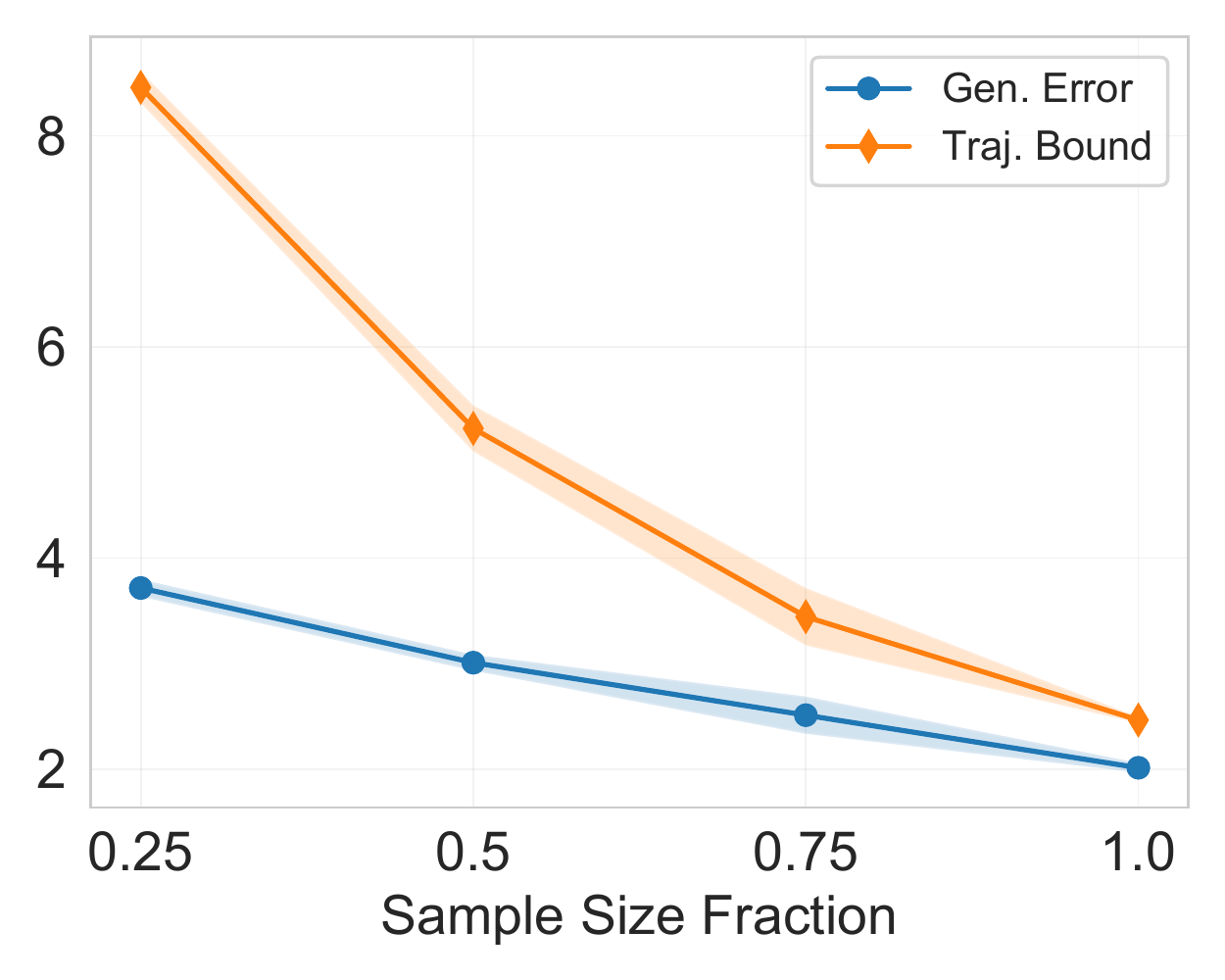}
\caption{VGG on CIFAR10}
    \label{fig:vgg-cifa10-err}
\end{subfigure}
 \begin{subfigure}[b]{0.32\textwidth}
 \centering
\includegraphics[scale=0.33]{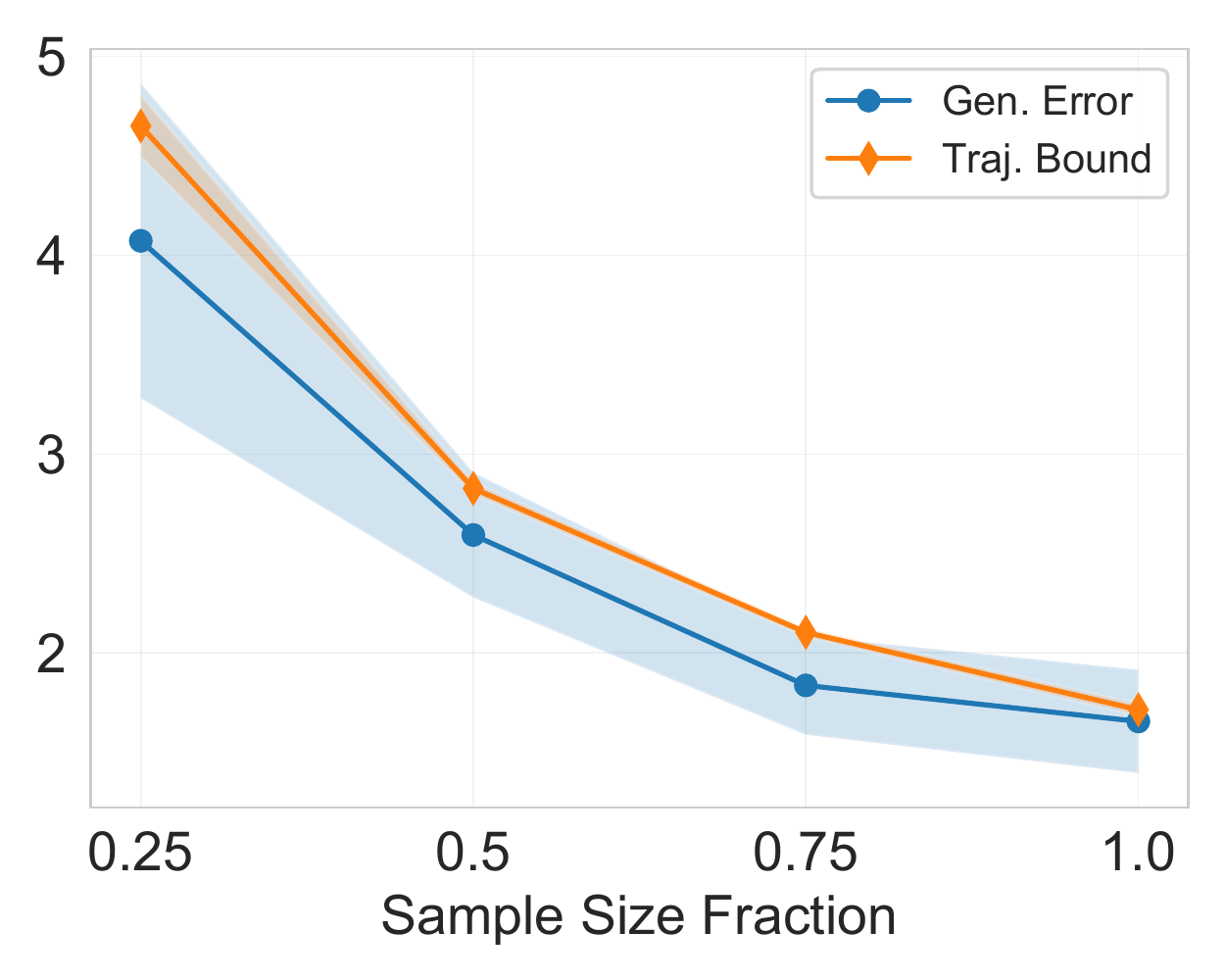}
\caption{ResNet on CIFAR10}
\label{fig:resnet-cifar-err}
    \end{subfigure}
% \begin{subfigure}[b]{0.245\textwidth}
% \includegraphics[scale=0.28]{figs/bound-plot-cifar10-resnet-TS.pdf}
% \caption{VGG on CIFAR10}
% \label{fig:vgg-cifar10-TM-bound}
% \end{subfigure}
\caption{Zoomed-in of generalization error.}\label{fig:errs}
\end{figure*}

\begin{figure*}[!ht]
    % \centering
    \begin{subfigure}[b]{0.48\textwidth}
    \centering
\includegraphics[scale=0.4]{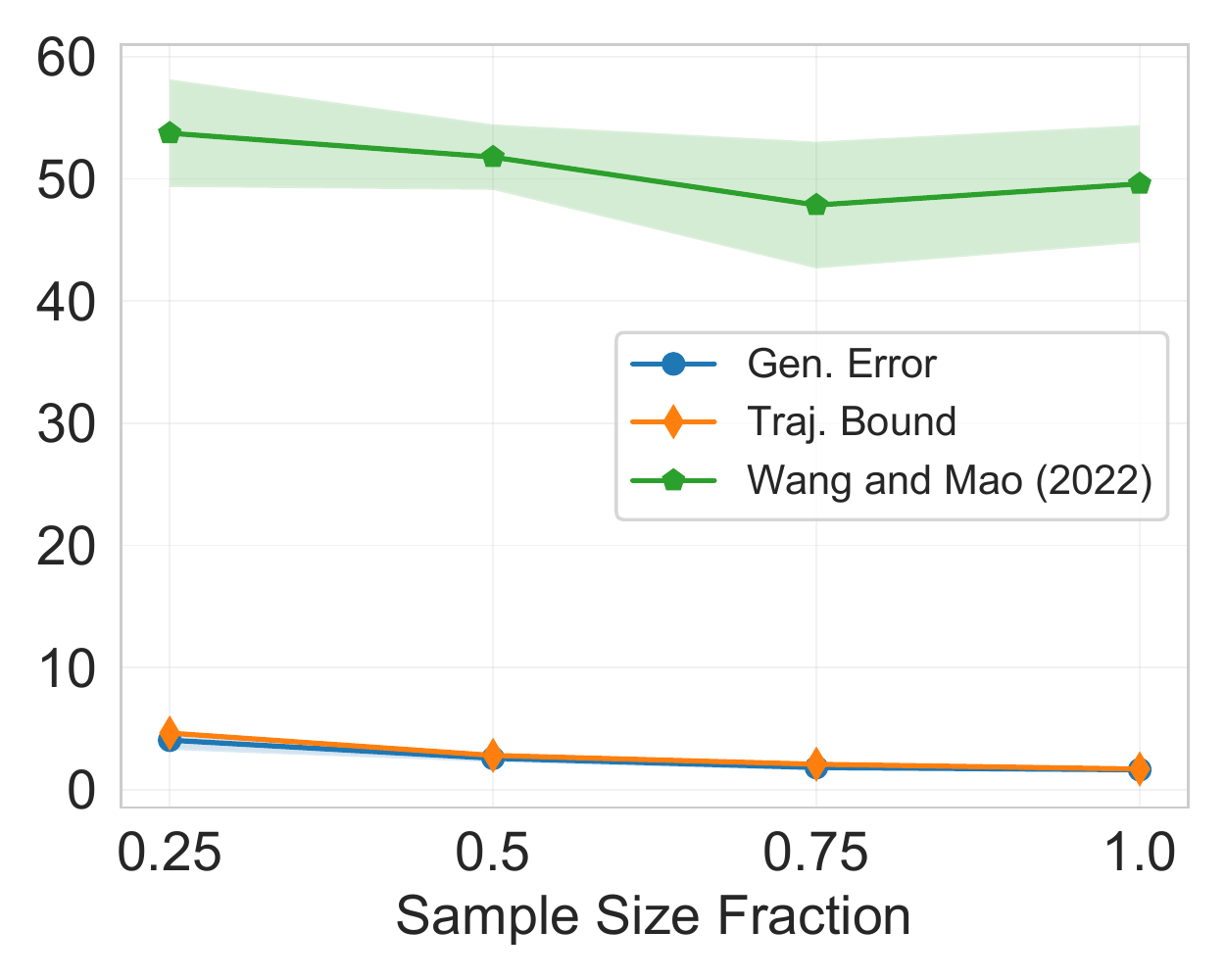}    
\caption{ResNet (Traj. Bound)}            \label{fig:resnet-cifar10-bound}
    \end{subfigure}
    \hfill
\begin{subfigure}[b]{0.48\textwidth}
\centering
\includegraphics[scale=0.4]{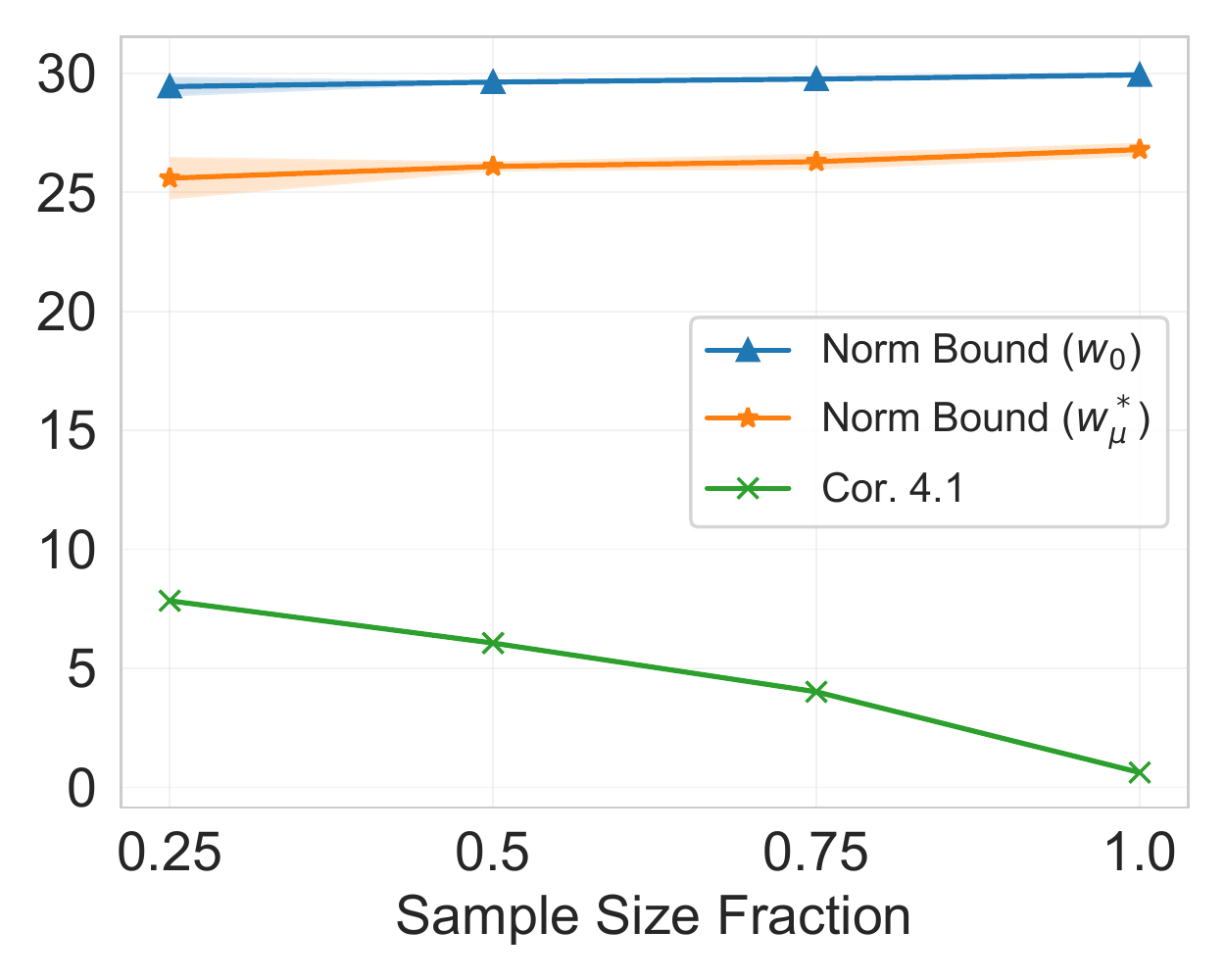}
\caption{ResNet (Term. Bound)}
    \label{fig:resnet-cifar10-bound-TS}
\end{subfigure}
\caption{Estimated trajectory-based bound and terminal-state based bound, with $R$ excluded. Models trained on CIFAR 10.}\label{fig:resnet-bounds}
\end{figure*}

\section{Additional Result: Inverse Population FIM as both Posterior and Prior Covariance}
% Another choice of the posterior covariance is the inverse population Fisher information matrix,
% \textcolor{red}{since the inverse of the FIM is an estimator of the asymptotic covariance matrix}, 
% which has already been treated as the posterior covariance 
Inspired by some previous works of \citep{achille2019information,harutyunyan2021estimating,wang2022pacbayes}, we can also select the inverse population Fisher information matrix $F^\mu_{w^*}=\ex{Z}{\nabla \ell(w^*,Z)\nabla \ell(w^*,Z)^{\bf T}}$ as the posterior covariance. Then, 
% recall Theorem~\ref{thm:pacbayes-data-dependent-prior},
the following theorem is obtained.
\begin{thm}
\label{thm:IF-pacbayes-FIM}
Under the same conditions in Theorem~\ref{thm:pacbayes-data-dependent-prior}, and assume the distribution $P_{W^*_{S_J}|S_J}$ is invariant of $J$, then
\[
\mathcal{E}_{\mu}(\mathcal{A})\leq\frac{M}{2n}\ex{S}{\sqrt{\cex{W^*_{S}}{S}{tr\{H^{-1}_{W^*_{S}}F^\mu_{W_S^*}\}}}}.
\]
% where $F^\mu_{w^*}=\ex{Z}{\nabla \ell(w^*,Z)\nabla \ell(w^*,Z)^T}$ is the population FIM.
\end{thm}
\begin{rem}
Notice that  $F^\mu_{W_S^*}\approx H^\mu_{W^*_{S}}\approx \Sigma^\mu(W_S^*)$ near minima \citep[Chapter~8]{pawitan2001all}, then $tr\{H^{-1}_{W^*_{S}}\Sigma^\mu(W_S^*)\}$ is very close to the Takeuchi Information Criterion \citep{takeuchi1976distribution}. In addition, our bound in Theorem~\ref{thm:IF-pacbayes-FIM} is similar to \citet[Theorem~3.]{singh2022phenomenology} with the same convergence rate, although strictly speaking, their result is not a generalization bound. Moreover, as also pointed out in \cite{singh2022phenomenology}, here $H^{-1}_{W^*_{S}}$ is evaluated on the training sample, unlike other works that evaluates the inverse Hessian on the testing sample (e.g., \citet{thomas2020interplay}). 
\end{rem}

The invariance assumption is also used in \citet{wang2021optimizing}. In practice, $n$ is usually very large, when $m=n-1$, this assumption indicates that replacing one instance in $s_{j}$ will not make $P_{W^*_{s_j}|s_j}$ be too different. 

% In practice, $n$ is usually very large, when $m=n-1$, the invariance assumption indicates that replacing one instance in $s_{j}$ will not make $P_{W^*_{s_j}|s_j}$ be too different. 

% Note that Theorem~\ref{thm:IF-pacbayes-FIM} is also based on the influence function (Eq.~(\ref{eq:influnce-function})). However, for the deep neural network training, the approximation made by influence function is often erroneous \citep{basu2021influence}. This, unfortunately, limits the practical application of Theorem~\ref{thm:IF-pacbayes-FIM}. 

% \section{Proof of Theorem~\ref{thm:IF-pacbayes-FIM}}
\begin{proof}[Proof of Theorem~\ref{thm:IF-pacbayes-FIM}]
We now use $(F^\mu_{W_S^*})^{-1}$ as both the posterior and prior covariance (again, we assume $F^\mu_{W_S^*}\approx F^\mu_{W_{S_j}^*}$ for any $j$), then
\begin{align*}
    \mathcal{E}_{\mu}(\mathcal{A})\leq&\ex{S}{\sqrt{\frac{M^2}{4}\cex{J,W^*_{S},W^*_{S_J}}{S}{{\left(W^*_{S}-W^*_{S_J}\right)F^\mu_{W_S^*}\left(W^*_{S}-W^*_{S_J}\right)^{\bf T}}}}}\\
    =&\frac{M}{2n}\ex{S}{\sqrt{\cex{W^*_{S},W^*_{S_j}}{S}{tr\left\{F^\mu_{W_S^*}H^{-1}_{W^*_{S}}H^{-1}_{W^*_{S}}\ex{J}{\nabla \ell(W^*_{S},Z_i)\nabla \ell(W^*_{S},Z_i)^{\bf T}}\right\}}}}\\
    =&\frac{M}{2n}\ex{S}{\sqrt{\cex{W^*_{S}}{S}{tr\left\{F^\mu_{W_S^*}H^{-1}_{W^*_{S}}\right\}}}},
\end{align*}
which completes the proof.
\end{proof}

\end{document}